\documentclass[journal, twoside]{IEEEtran}
\usepackage{amsmath,amsfonts}
\usepackage{algorithmic}
\usepackage{algorithm}
\usepackage{array}
\usepackage{textcomp}
\usepackage{stfloats}
\usepackage{hyperref}
\usepackage{url}
\usepackage{verbatim}
\usepackage{graphicx}
\usepackage{cite}
\usepackage{amsthm}
\usepackage{todonotes}
\usepackage{multirow}
\usepackage[caption=false,font=normalsize,
labelfont=sf,textfont=sf]{subfig}
\newtheorem{theorem}{Theorem}
\newtheorem{lemma}[theorem]{Lemma}

\theoremstyle{definition}
\newtheorem{remark}{Remark}

\begin{document}

\title{A Novel Model for 3D Motion Planning for a Generalized Dubins Vehicle with Pitch and Yaw Rate Constraints\\
\thanks{
Received 24 September 2025; revised 2 March 2026; accepted
13 April 2026. Recommended by Editor Rafael Murrieta-Cid. (\textit{Corresponding author: Deepak Prakash Kumar}).

Deepak Prakash Kumar is with the Department of Electrical Engineering and Computer Science, University of California, Irvine, CA 92697, USA (e-mail: {\tt\small deepakprakash1997@gmail.com}).

Swaroop Darbha is with the Department of Mechanical Engineering, Texas A\&M University, College Station, TX 77843, USA (e-mail: {\tt\small dswaroop@tamu.edu}).

Satyanarayana Gupta Manyam is with the DCS Corporation, 4027 Col Glenn Hwy, Dayton, OH 45431, USA (e-mail: {\tt\small msngupta@gmail.com}).

David Casbeer is with the Control Science Center, Air Force Research Laboratory, Wright-Patterson Air Force Base, OH 45433 USA (e-mail:
{\tt\small david.casbeer@us.af.mil}).

DISTRIBUTION STATEMENT A. Approved for public release. Distribution is unlimited. AFRL-2025-3035; Cleared 06/17/2025.}
}

\author{Deepak Prakash Kumar, \IEEEmembership{Member, IEEE}, Swaroop Darbha, \IEEEmembership{Fellow, IEEE}, Satyanarayana Gupta Manyam, \IEEEmembership{Senior Member, IEEE}, David W. Casbeer, \IEEEmembership{Senior Member, IEEE}}

\markboth{IEEE Transactions on Robotics}%
{Kumar \MakeLowercase{\textit{et al.}}: Model for 3D Motion Planning for a Dubins Vehicle with Pitch and Yaw Rate Constraints}

\maketitle

\bstctlcite{IEEEexample:BSTcontrol}

\begin{abstract}

In this paper, we propose a new modeling approach and a fast algorithm for 3D motion planning, applicable for fixed-wing unmanned aerial vehicles. The goal is to construct the shortest path connecting given initial and final configurations subject to motion constraints. Our work differs from existing literature in two ways. First, we consider full vehicle orientation using a body-attached frame, which includes roll, pitch, and yaw angles. However, existing work uses only pitch and/or heading angle, which is insufficient to uniquely determine orientation. 
Second, we use two control inputs to represent bounded pitch and yaw rates, reflecting control by two separate actuators. In contrast, most previous methods rely on a single input, such as path curvature, which is insufficient for accurately modeling the vehicle’s kinematics in 3D.
We use a rotation minimizing frame to describe the vehicle's configuration and its evolution, and construct paths by concatenating optimal Dubins paths on spherical, cylindrical, or planar surfaces. Numerical simulations show our approach generates feasible paths within 10 seconds on average and yields shorter paths than existing methods in most cases.
\end{abstract}


\section{Introduction}

\IEEEPARstart{T}{he} use of Unmanned Aerial Vehicles (UAVs) 
is rapidly growing in civilian and military applications, including search and rescue and surveillance.
Fixed-wing UAVs are of particular interest due to longer flight times, larger payload capacity, and the ability to fly at higher altitudes\cite{fixed_wing_UAV_1, fixed_wing_UAV_2}. However, they are persistently in motion, i.e., cannot stop or hover mid-air, and cannot change their heading angle instantaneously. Hence, they have a bound on the rate of change of their heading/orientation, which manifests itself as curvature constraints on the path. Motion planning is important for these vehicles, in which the goal is to plan the optimal path to travel from one configuration (i.e., position and orientation together) to another. The objective of interest in this paper is to obtain the minimum-time (or distance) path(s). We seek a finite set of candidate paths that includes the optimal path for any boundary condition. 
These candidate paths are suitable for constructing paths for fixed-wing aircraft or yaw rate-constrained vehicles. 

Motion planning for yaw rate-constrained vehicles is typically addressed by considering a simplified kinematic model, called the Dubins model. This models a vehicle traveling at a constant speed and has a minimum turning radius constraint, which is suitable for UAVs traveling at constant altitude (in 2D). Dubins \cite{Dubins} solved the problem of the shortest path between a pair of configurations on a plane. It was shown that the optimal path is of type $CSC, CCC,$ or a degenerate path of the same, where $C = L, R$ denotes a left or a right turn of minimum turning radius, and $S$ denotes a straight line segment. Although Dubins showed this result using geometric techniques, the same result was later derived using Pontryagin's Minimum Principle (\cite{PMP}) in \cite{Shortest_path_synthesis_Boissonat} using simpler proofs. Various variants of the planar path planning problem have been explored with variations in the model and/or the objective, such as in \cite{sinistral/dextral}, where different left and right turning radius was considered.

Motion planning for such vehicles in 3D has also been an area of interest, where the shortest path to travel from one configuration to another, considering their motion constraints, is sought. The 3D problem applies not only to fixed-wing UAVs but also to underwater gliders and robots \cite{3D_underwater, motion_planning_two_3D_Dubins_vehicles}. 
Although specifying the heading angle alone uniquely defines the orientation of the vehicle in the 2D problem, the 3D problem requires both the heading angle and the plane in which the UAV lies.
The UAV’s plane can be uniquely described using two additional angles (pitch and roll) or by a vector along its lateral or normal direction.

Simple kinematic models have also been used to tackle the 3D problem, where, similar to the 2D problem, the generated path can be tracked using a lower-level controller \cite{path_generation_tracking_3D}. 
Because these paths can be computed very quickly, they can be combined with algorithms such as Rapidly Exploring Random Tree (RRT) to find feasible, obstacle-avoiding routes \cite{rick_lind}.
This method has also been experimentally demonstrated using an ARDrone in \cite{path_planning_3D_Dubins_saripalli}.

To our knowledge, the first exploration of the 3D problem was by Sussman \cite{sussman_3D}. The author showed that the optimal path is of type $CSC,$ $CCC,$ or a degenerate\footnotemark\footnotetext{Degenerate paths of $CSC$ and $CCC$ paths are $CS,$ $SC,$ $CC,$ $C,$ and $S$.} version of these paths, or a helicoidal arc to connect a given location and heading direction\footnotemark. \footnotetext{We note that in 3D, specifying only the location and heading direction does not fully determine the vehicle's configuration, since the orientation of the plane containing the vehicle is not uniquely defined. In contrast, for 2D motion planning, location and heading are sufficient to uniquely specify the configuration.}
Unlike the 2D problem, there are infinitely many $C$ segments since the plane containing the segment can be arbitrarily picked; due to the tangential $S$ segment, many (finite) solutions may exist. Conditions exist for which infinitely many $CSC$ paths exist, such as those shown in \cite{analytic_solution_3D}. Hence, efficient construction of $CSC$ paths has been explored in the literature. In \cite{optimal_geometrical_path_in_3D}, a $CSC$ path was constructed for instances where the initial and final locations are spaced sufficiently far apart using geometric and numerical approaches. In a later work  \cite{optimal_path_planning_aerial_vehicle_3D}, the authors adopted this path as an initial guess for a nonlinear optimization problem that was solved using a multiple-shooting method to improve the solution. The $CSC$ path construction was addressed in \cite{analytic_solution_3D} as an inverse kinematics problem for a five degrees-of-freedom robotic manipulator. Analytical solutions were obtained for the path parameters to improve computational efficiency. $CSC$ path construction was also recently addressed in~\cite{reparametrization_3D_Dubins}. In that work, the authors parametrized the path in terms of two variables and performed a numerical search, utilizing off-the-shelf solvers assisted by derived gradients to construct the path.

In \cite{time_optimal_paths_Dubins_airplane}, the 3D problem was addressed with a model that has two controllable inputs - one for yaw rate, and another for the rate of change of the altitude. In this work, a ``Dubins airplane model" was proposed, and it was shown that the optimal trajectory comprises segments with arcs of minimum turning radius, straight lines, or a Dubins path of a certain length. Depending on the altitude difference between the initial and final locations, feasible solutions were generated by introducing additional segments to attain the desired altitude. This model was modified in \cite{Dubins_airplane_fixed_wing_UAVs}, wherein a pitch angle constraint was introduced. Contrary to \cite{time_optimal_paths_Dubins_airplane}, this paper generates trajectories by incorporating helicoidal segments to attain the desired altitude when necessary. Additionally, the constructed paths were validated by simulations using a six-degree-of-freedom model (given in \cite{small_unmanned_aircraft}) and a vector-field-based guidance law, inspired by \cite{goncalves}, for tracking.

The 3D motion planning problem has also been addressed using the 2D Dubins result in \cite{minimal_3D_Dubins_path_bounded_curvature_pitch}. The authors consider a \textit{total} curvature constraint and a bounded pitch angle for the vehicle. They decouple the $CSC$ path to connect the given initial and final locations and heading vector into horizontal and vertical components. In this regard, they use the horizontal projection of the configurations to connect the locations and heading angles. The obtained path length is utilized as the $x$ coordinate in the vertical plane, with the desired altitude difference to be attained serving as the $z$ coordinate. Using iterative optimization of the horizontal and vertical turning radii, a feasible solution is constructed such that the total curvature bound and pitch angle bounds are satisfied. The authors used this solution to provide an initial guess to a non-linear optimization problem to further refine the path in \cite{finding_3D_dubins_paths_pitch_angle_nonlinear_optimization}.

In the existing literature, we observe that the {\it complete orientation of the vehicle} has not been considered for 3D motion planning. This is because the description of the pitch and heading angles alone does not uniquely describe its orientation. For example, Fig.~\ref{fig: ambiguous_roll_angle} shows two orientations for the same vehicle moving along a straight line trajectory with a pitch angle of $0^\circ$ and a heading angle of $45^\circ:$ one where the roll angle is $0^\circ$ and another where the roll angle is $45^\circ.$ Infinitely many orientations exist for the same trajectory. However, from these figures, we can observe that prescribing the longitudinal direction and the lateral or normal direction of the vehicle would uniquely describe the orientation.\footnotemark\, While the study in \cite{towards_finding_shortest_paths_3D_rigid_bodies} models the complete configuration of a general robot, it is unclear how their model relates to the kinematics of a fixed-wing UAV.
\footnotetext{We remark here that in Fig.~\ref{fig: ambiguous_roll_angle}, unit vectors along the longitudinal, lateral, and normal directions are referred to as tangent, tangent normal, and surface normal vectors, respectively. The latter notation would be used later in the paper to describe the rotation minimizing frame model.}
\begin{figure}[htb!]
    \centering
    \subfloat[\footnotesize{Roll angle $= 0^\circ$ (animation provided on our GitHub page)}]{\includegraphics[width = 0.49\linewidth]{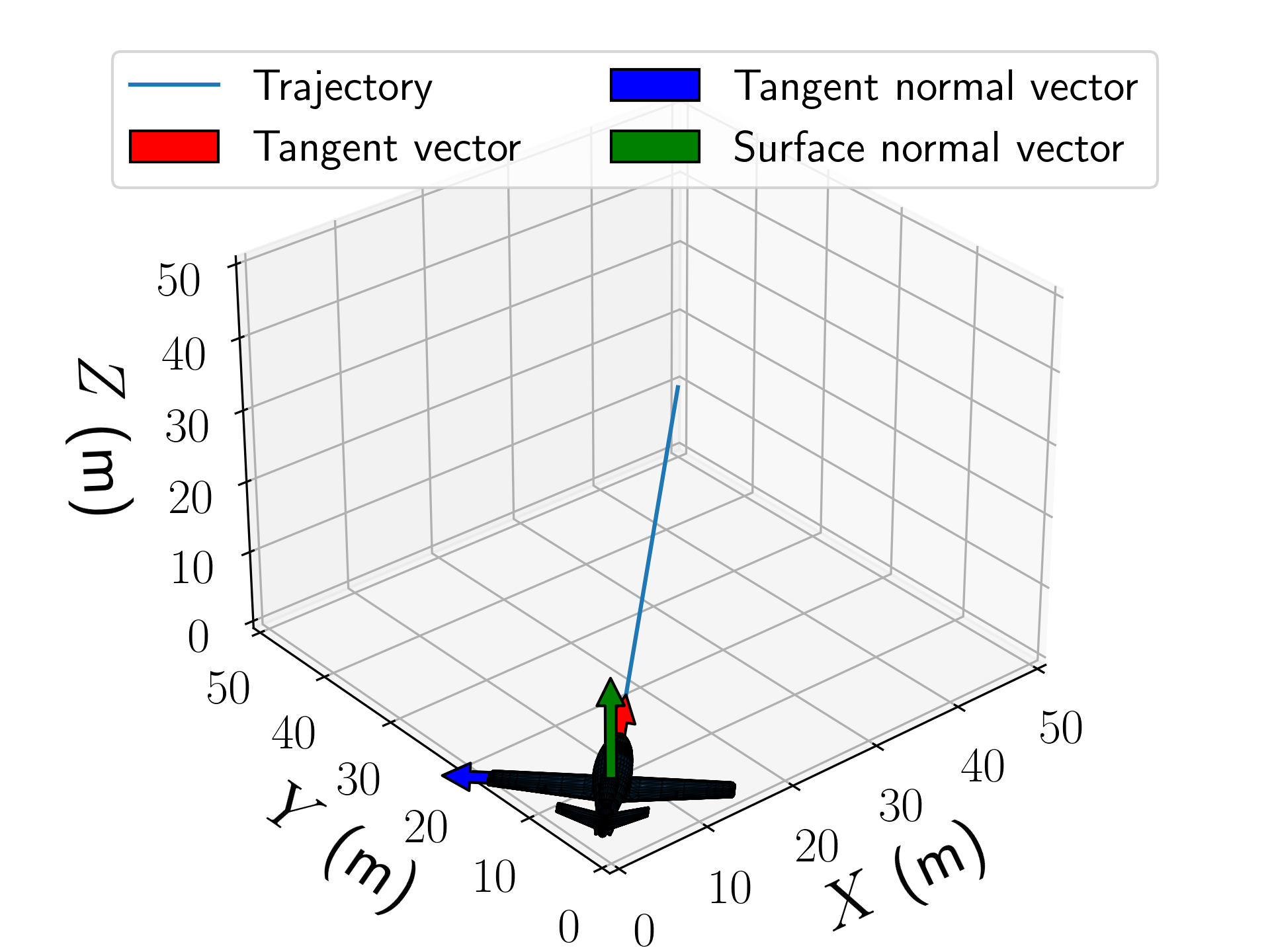}\label{subfig: roll_zero_deg}}
    \hfil
    \subfloat[\footnotesize{Roll angle $= 45^\circ$ (animation provided on our GitHub page)}]{\includegraphics[width = 0.49\linewidth]{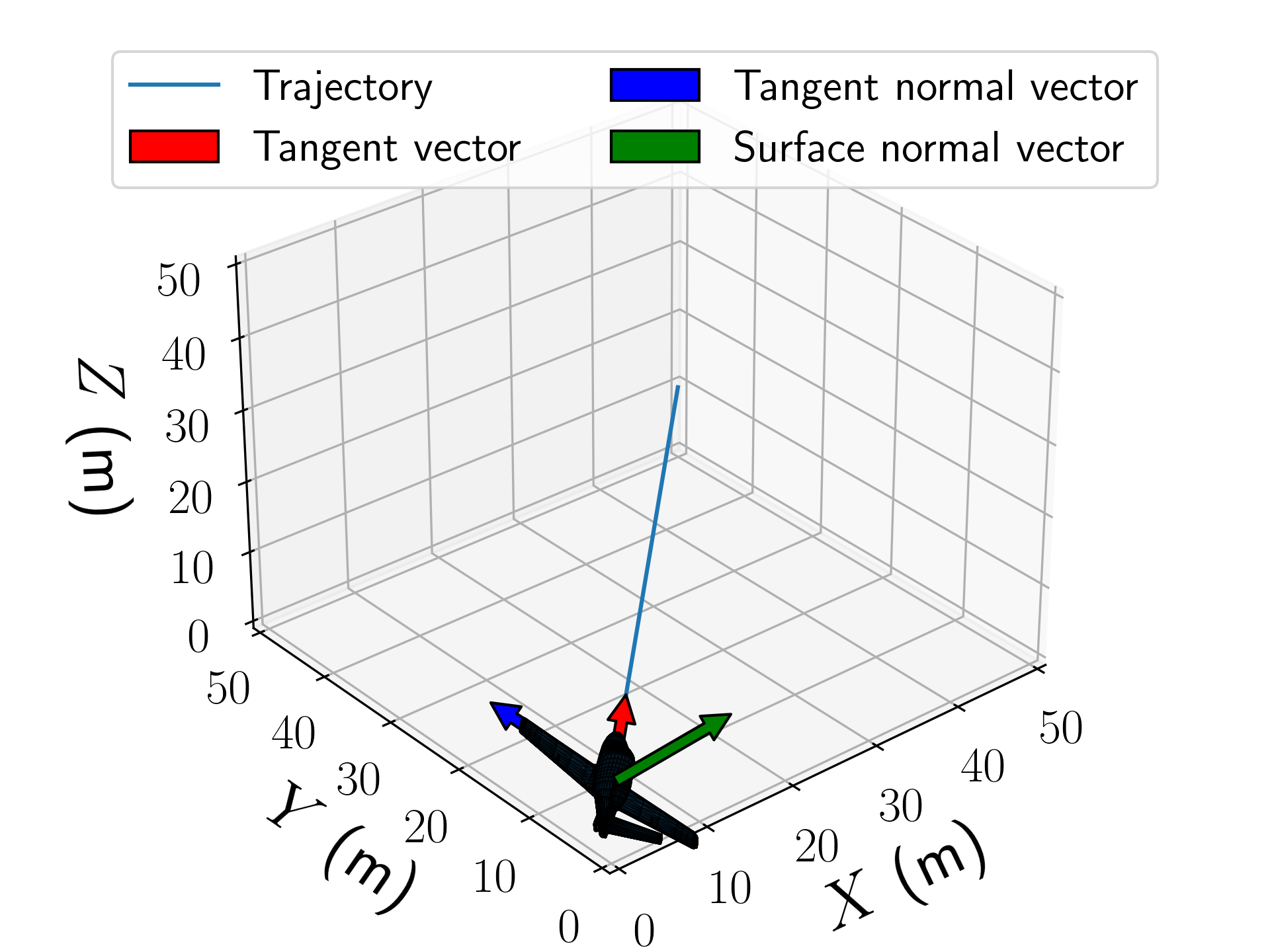}\label{subfig: roll_45_deg}}
    \caption{Depiction of two orientations corresponding to the same heading and pitch angle}
    \label{fig: ambiguous_roll_angle}
\end{figure}

The literature on path planning for aerial robots has primarily focused on models with a single control input, such as yaw rate, or a single constraint on the path's curvature.
In the 3D generalization of path planning, a single control input alone cannot capture the range of motions. This is because there are two elementary motions of interest for the motion planning of aerial vehicles: pitch and yaw. For fixed-wing UAVs, pitch motion is achieved using elevators, and yaw motion is achieved with rudder and/or aileron\footnotemark\footnotetext{Ailerons control the roll of the UAV, and this roll motion leads to a corresponding yaw motion for the vehicle.} \cite{small_unmanned_aircraft}. Since yaw and pitch motion are controlled by separate actuators, considering a single control is not sufficient. Hence, it is crucial to consider two control inputs: a bounded pitch rate and a bounded yaw rate. These constraints correspond to the minimum turning radii, $R_{pitch}$ and $R_{yaw}$, of the path's curvature.  The bounds on the pitch and yaw rates manifest as locally inaccessible spherical regions of radii $R_{pitch}$ and $R_{yaw}$, respectively.\footnotemark \footnotetext{
Imagine a vehicle moving in 3D space with a maximum allowable yaw rate (i.e., how quickly it can turn left or right). Due to this constraint, the vehicle cannot immediately change direction; it needs a certain minimum turning radius to execute a yaw maneuver.
A yaw motion sphere is a spherical region around the vehicle that it cannot enter directly unless it has traveled a sufficient distance to allow for the required turning maneuver. This is analogous to the turning circles on the left and right in the 2D Dubins path problem. Similar to the yaw motion spheres that define inaccessible regions in the horizontal plane, pitch constraints define vertical maneuverability limits.}
Though two control inputs are considered in \cite{time_optimal_paths_Dubins_airplane}, the second control input is the rate of change of the altitude of the vehicle; furthermore, the pitch angle is not considered in this model, which makes it more appropriate for a quadcopter. 
A demonstration of the limitations of state-of-the-art approaches that rely on a single control input is presented in Fig.~\ref{fig: single_control_issues} for two instances. In both of these cases, paths were generated using the method from \cite{minimal_3D_Dubins_path_bounded_curvature_pitch} - for which the code is publicly available. The minimum turning radius for this model was set to 40 meters to enforce the curvature constraint. In the following, we analyze these examples to explain why this path planning method is either infeasible or inefficient for the proposed (3D) model.
\begin{enumerate} 
\item Our proposed model incorporates two independent control inputs for the yaw and pitch rates. We set the minimum pitch turning radius, $R_{pitch}$, to 40 meters\footnotemark\footnotetext{In the paper, we will interchangeably use ``m" to denote ``meters".} for the first example. This choice is consistent with the parameter used in \cite{minimal_3D_Dubins_path_bounded_curvature_pitch}. Since the yaw turning radius, $R_{yaw}$, can be chosen independently, we set it to 50 meters. The solution generated by \cite{minimal_3D_Dubins_path_bounded_curvature_pitch}, shown in Fig.~\ref{subfig: yaw_rate_violation}, violates the yaw rate constraint by entering the yaw motion sphere. Such a maneuver is infeasible as the vehicle must travel a sufficient distance outside the sphere before entering it. This behavior is analogous to the infeasibility of entering the left or right turning circles in the 2D Dubins problem \cite{Dubins}. 

    \item In the second example, we alternately pick $R_{yaw}$ in our model to be the same as the parameter in \cite{minimal_3D_Dubins_path_bounded_curvature_pitch}, which is $40$ meters. Since $R_{pitch}$ is free to choose, we set $R_{pitch}$ to be equal to $60$ meters.
    The path obtained using the algorithm from \cite{minimal_3D_Dubins_path_bounded_curvature_pitch} is shown in Fig.~\ref{subfig: pitch_rate_violation}. We observe that the path enters one of the pitch motion spheres (which lies at the top of the vehicle), and hence violates the pitch rate constraint. 
\end{enumerate}
Alternatively, one could argue that the maximum of $R_{yaw}$ and $R_{pitch}$ can be chosen as the minimum turning radius in \cite{minimal_3D_Dubins_path_bounded_curvature_pitch}. However, this would lead to inefficient motion planning, since the vehicle would take larger-than-necessary turns in some instances.

\begin{figure}[htb!]
    \centering
    \subfloat[Yaw rate violation for $R_{pitch} = 40$ m. $R_{yaw} = 50$ m.]{\includegraphics[width = 0.46\linewidth]{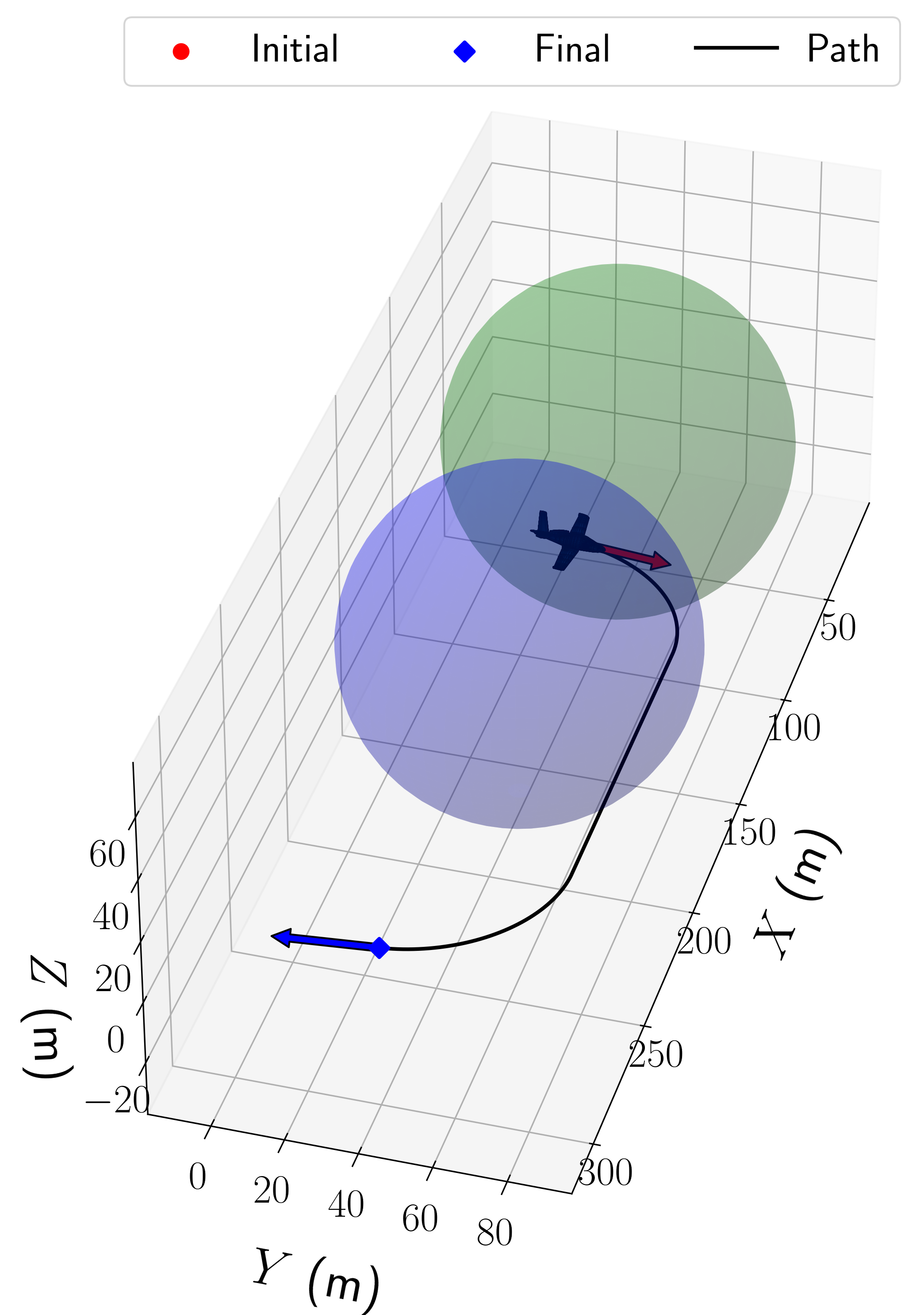}\label{subfig: yaw_rate_violation}}
    \hfil
    \subfloat[Pitch rate violation for $R_{yaw} = 40$ m. $R_{pitch} = 60$ m.]{\includegraphics[width = 0.49\linewidth]{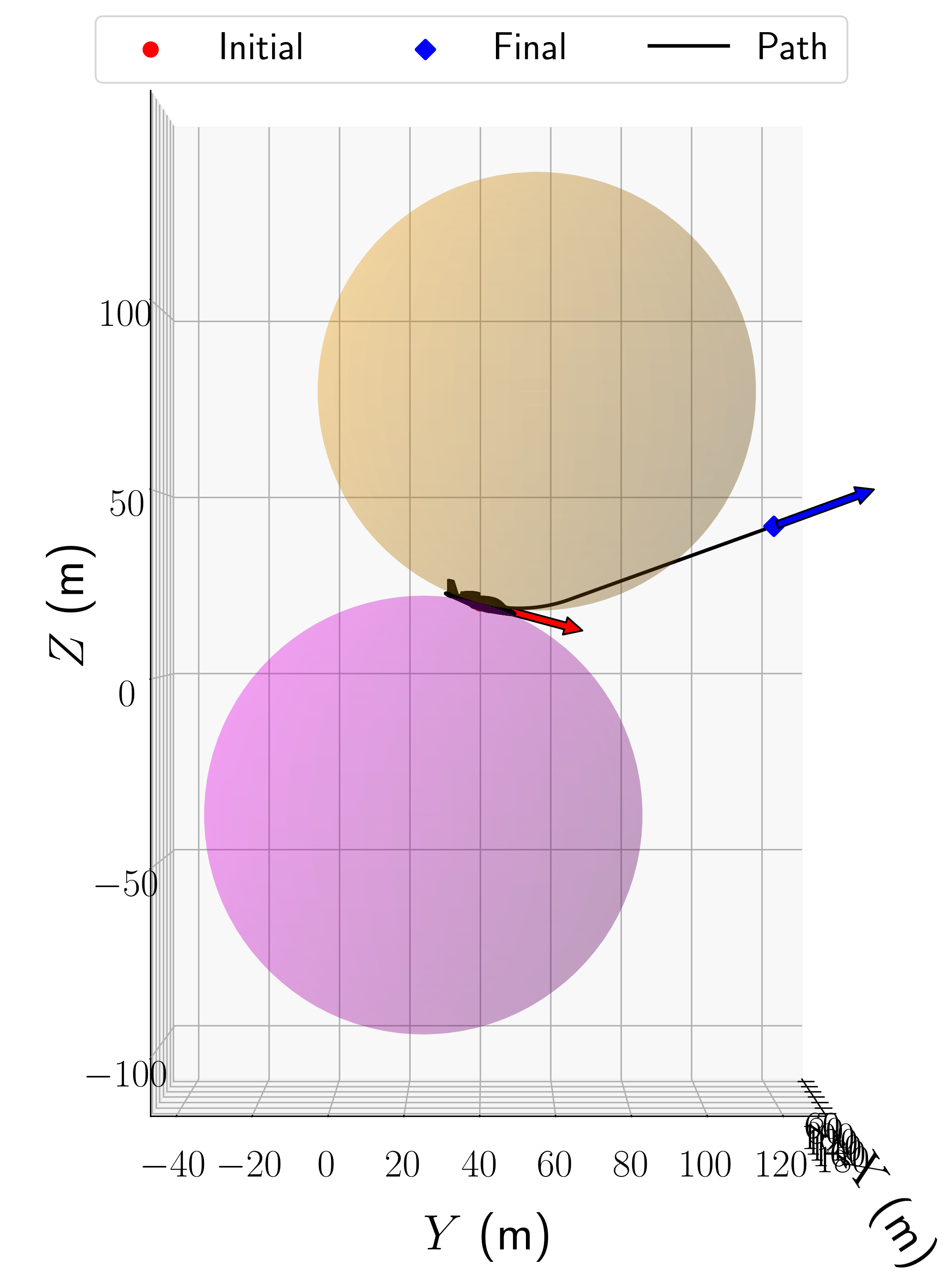}\label{subfig: pitch_rate_violation}}
    \caption{Issue with single control input}
    \label{fig: single_control_issues}
\end{figure}

The presented issues were addressed in our previous work in \cite{kumar2025newapproachmotionplanning}, where a special case of motion planning on the surface of a sphere was studied. This paper provided insights for the 3D motion planning by considering a vehicle model with bounded yaw rate and pitch rate. Furthermore, the spherical motion planning problem was shown to arise as an intermediary problem to be solved for the 3D problem.

Having identified two major issues in the literature, the contributions of this paper are as follows:
\begin{enumerate}
    \item We present a novel model using a rotation minimizing frame, also called the Bishop frame, to obtain the shortest path for a vehicle subject to pitch rate and yaw rate constraints. 
    We build on the insights provided in \cite{kumar2025newapproachmotionplanning} for the 3D problem. 
    \item We prove that the pitch rate and yaw rate constraints manifest as four spheres around the vehicle that represent temporarily inaccessible regions, thereby appropriately generalizing the 2D Dubins model to 3D.
    \item We propose a path construction algorithm that consists of three classes of paths. The main idea is to build path segments on spherical surfaces that are tangent to the initial and final configurations, and these segments are connected by an intermediary surface. This surface could be a cylindrical envelope, a cross-tangent plane, or another spherical surface. 
    \item We pose and solve a Dubins-type path planning problem subject to curvature constraints on a cylindrical surface. To the best of our knowledge, the cylindrical motion planning problem has not been addressed in the literature. The proposed solution method involves unwrapping the surface to a two-dimensional plane, computing the optimal Dubins path, and then mapping back onto the cylindrical surface. 
    \item We present extensive numerical results on several instances. We show the effect of $(i)$ model that defines the complete configuration of the vehicle, (\emph{i.e., heading and lateral orientation}) and $(ii)$ the impact of minimum turning radii on the best feasible path. We also observe that our algorithm can produce a high-quality feasible solution within 10 seconds. Additionally, we provide the code in a publicly available repository.\footnotemark
\end{enumerate}
\footnotetext{The code for our algorithm is available at \url{https://github.com/DeepakPrakashKumar/3D-Motion-Planning-for-Generalized-Dubins-with-Pitch-Yaw-constraints.git}.}



\section{Modeling and Geometric Preliminaries} \label{sect: modeling}

Let $t$ and $s$ denote the time and arc length, respectively, and $\mathbf{X} (s)$ denote the instantaneous location of the vehicle. 
We consider a Rotation-Minimizing frame, also called a Bishop frame \cite{Bishop}, attached to the center of mass of the UAV. Let $\mathbf{T}, \mathbf{Y}, \mathbf{U}$ denote the unit vectors of the Bishop frame with $\mathbf{T}, \mathbf{Y}$ directed along the longitudinal and lateral directions of the vehicle, respectively. The vector $\mathbf{U} := \mathbf{T} \times \mathbf{Y}$ is along the normal direction of the vehicle. Fig.~\ref{fig: configuration_vehicle} shows the vehicle configuration with the vectors $\mathbf{T}, \mathbf{Y}$ and $\mathbf{U}$.\footnotemark\footnotetext{We also denote $\mathbf{T},$ $\mathbf{Y},$ and $\mathbf{U}$ as tangent, tangent normal, and surface normal vectors, respectively.} 
\begin{figure}[htb!]
    \centering
    \includegraphics[width=0.7\linewidth]{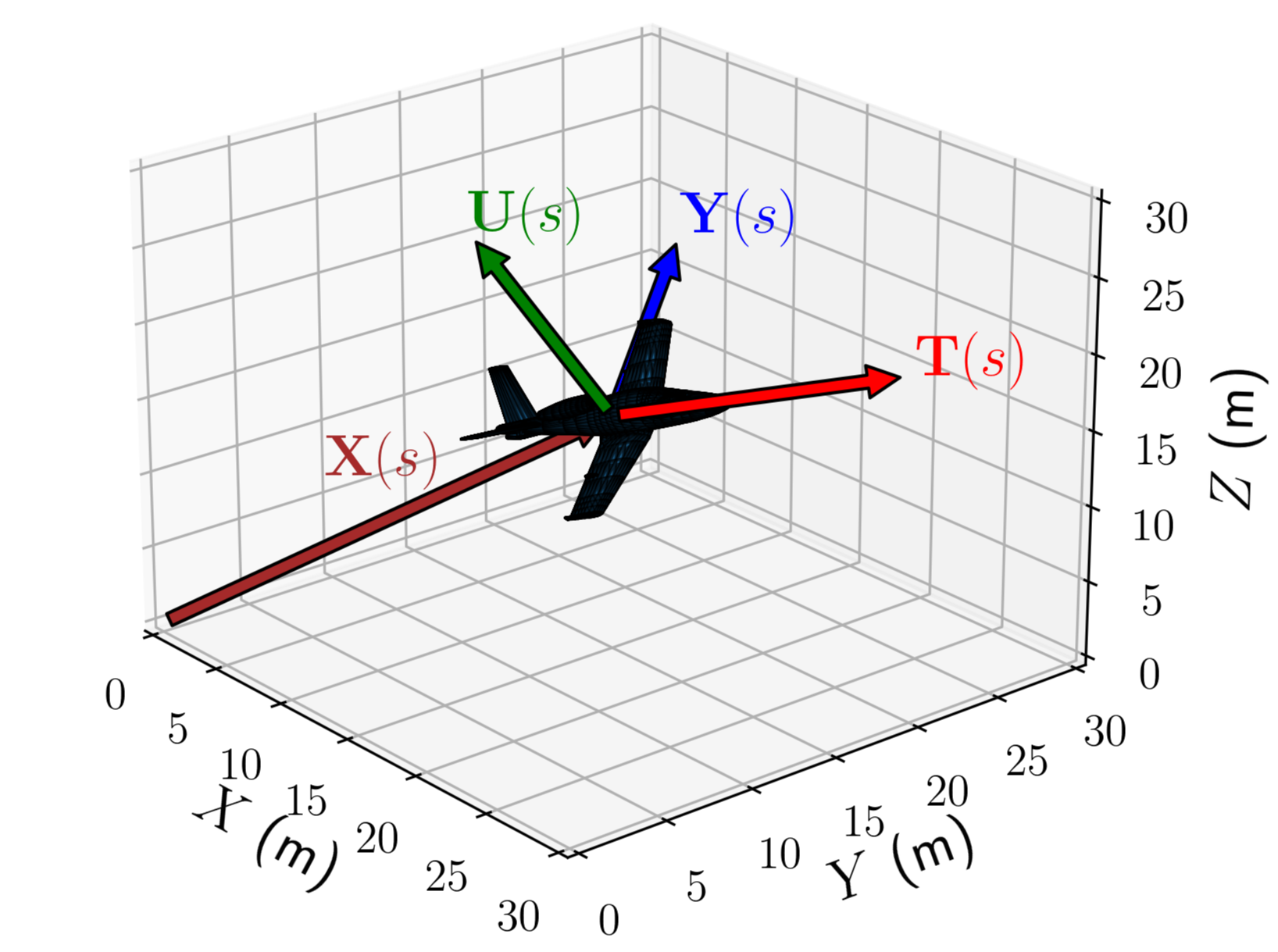}
    \caption{The configuration of the vehicle defined by the three vectors $\mathbf{T}, \mathbf{Y}$ and $\mathbf{U}$}
    \label{fig: configuration_vehicle}
\end{figure}

The instantaneous angular velocity of the frame can be written as
\begin{align*}
    \omega(t) = \omega_x(t) {\bf T}(t) + \omega_y(t) {\bf Y}(t) + \omega_z(t) {\bf U}(t),
\end{align*}
with $\omega_x, \omega_y, \omega_z$ denoting the components in the ${\bf T}, {\bf Y}$, and ${\bf U}$ directions, respectively. One may think of them as the roll, pitch, and yaw rates of the body, respectively. Notably, the Rotation Minimizing Frame (RMF) is constructed so that $\omega_x(t) = 0$: there is no rotation about the tangent, minimizing frame twisting (essentially, the roll rate is set to zero).\footnotemark\footnotetext{Since $\omega_x = 0,$ the generated paths do not allow for unbounded roll. The paths constructed remain feasible for a model with non-zero roll rate assumption, however, the paths may be suboptimal when additional constraints, such as bounded roll angles or rates, are imposed. Extending our model and method to enforce bounds on the roll angle would be valuable, and can be considered an important direction for future work.} The kinematics of the frame then satisfy:
\begin{align*}
    \frac{d{\bf T}}{dt} &= \omega(t) \times {\bf T}(t) = \omega_z(t) {\bf Y}(t) - \omega_y(t) {\bf U}(t), \\
    \frac{d{\bf Y}}{dt} &= \omega(t) \times {\bf Y}(t) = -\omega_z(t) {\bf T}(t), \\
    \frac{d{\bf U}}{dt} &= \omega(t) \times {\bf U}(t) = \omega_y(t) {\bf T}(t).
\end{align*}
A key property of an RMF is that ${\bf Y}(t)$ and ${\bf U}(t)$ change only in the direction of ${\bf T}(t)$.

Assume that the vehicle moves at a constant, nonzero longitudinal speed $V_0$. Defining
\begin{align} \label{eq: definitions}
    \kappa_g (s) := \frac{\omega_z(s)}{V_0}, \qquad \kappa_n(s) := -\frac{\omega_y(s)}{V_0},
\end{align}
the kinematic equations parameterized in terms of $s$ become
\begin{align}
    \frac{d\mathbf{X}(s)}{ds} &= \frac{1}{V_0}\frac{d{\bf X}}{dt} = \mathbf{T}(s), \label{eq:motion} \\
    \frac{d\mathbf{T}(s)}{ds} &= \kappa_g(s) \mathbf{Y}(s) + \kappa_n(s) \mathbf{U}(s), \label{eq:kinematics2}\\
    \frac{d\mathbf{Y}(s)}{ds} &= -\kappa_g(s) \mathbf{T}(s), \label{eq:kinematics3}\\
    \frac{d\mathbf{U}(s)}{ds} &= - \kappa_n(s) \mathbf{T}(s). \label{eq:kinematics4}
\end{align}
The bounds for control inputs $\kappa_g$ and $\kappa_n$, which we refer to as geodesic curvature and normal curvature, respectively, are stated as shown below,\footnotemark\footnotetext{We use the notation $\kappa_g$ and $\kappa_n$, since they have the same form as that of the geodesic and normal curvatures in the Darboux frame model, a differential geometric model.} 
\begin{align} \label{eq: curvature_constraints}
    |\kappa_n| \le \frac{1}{R_{pitch}}, \quad |\kappa_g| \leq \frac{1}{R_{yaw}}.
\end{align}
We shall later show that $R_{pitch}$ is the minimum turning radius corresponding to pitch motion and $R_{yaw}$ is the minimum turning radius corresponding to the yaw motion.  
The objective is to compute the minimum distance trajectory from the initial to final configuration, defined by $\mathbf{X},$ $\mathbf{T},$ $\mathbf{Y},$ and $\mathbf{U}$. Hence, the cost to minimize is $J = \int 1 ds.$

\begin{remark}[\textbf{Model}]
Three angles (roll, pitch, and yaw), or equivalently, $\mathbf{T},$ $\mathbf{Y},$ and $\mathbf{U},$ are required to specify the UAV's orientation in 3D, but only two (pitch and yaw) are directly controlled by the primary actuators on a fixed-wing UAV (aileron, rudder, and elevator). The roll angle evolves naturally as a consequence of coordinated turning \cite{small_unmanned_aircraft}, wherein ailerons, which cause roll, allows the vehicle to turn. From a planning perspective, the two control inputs are sufficient to reach any orientation; this is analogous to rigid-body kinematics, wherein a z-y-z rotation can be used to reach any orientation. The Bishop frame is used to separate vehicle orientation from path geometry, enabling continuous and physically feasible orientation profiles while capturing the aircraft's kinematics.
\end{remark}

Geometrically, the path planning problem can be depicted as shown in Fig.~\ref{fig: 3D_config}. A detailed description of this figure follows.

\begin{figure}[htb!]
    \centering
    \includegraphics[width = \linewidth]{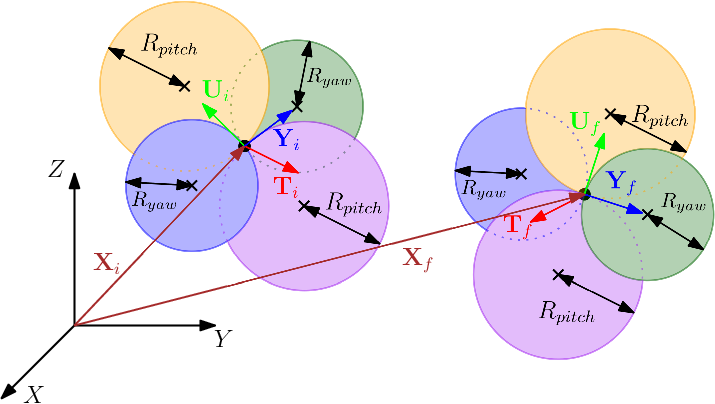}
    \caption{Depiction of spheres corresponding to pitch motion (in orange and magenta) and yaw motion (in blue and green) at the initial and final configuration}
    \label{fig: 3D_config}
\end{figure}

In Fig.~\ref{fig: 3D_config}, it can be observed that four spheres are constructed, which are tangential to the vehicle configuration. To understand how they appear, we need to understand the geometric impact of the curvatures, $\kappa_n$ and $\kappa_g.$ To this end, we derive the closed-form expression for $\mathbf{X},$ $\mathbf{T},$ $\mathbf{Y},$ and $\mathbf{U}$ over an interval wherein $\kappa_g$ and $\kappa_n$ are constants. The obtained expressions are shown in Appendix~\ref{appsubsect: construction_segments}.

We remark here that since the control inputs appear linearly in the differential equations \eqref{eq:kinematics2}-\eqref{eq:kinematics4}, and a minimum time problem is considered, the optimal control actions are expected to be bang-bang from Pontryagin's minimum principle \cite{PMP}.\footnotemark\footnotetext{We remark here that in general, minimizing path length and minimizing time are not equivalent. However, we assume the speed of the vehicle to be constant, which can be set to be a unit speed without loss of generality; if not, the optimal time and optimal path length will differ by a scalar value (which is the speed).}
Therefore, it is sufficient to consider intervals in which $\kappa_g$ and $\kappa_n$ are constant. Furthermore, it suffices to consider $\kappa_g \in \{-\frac{1}{R_{yaw}}, 0, \frac{1}{R_{yaw}} \}$ and $\kappa_n \in \{-\frac{1}{R_{pitch}}, 0, \frac{1}{R_{pitch}} \}$.

Using the closed-form expressions derived in Appendix~\ref{appsubsect: construction_segments}, we state and prove the following two lemmas.
\begin{lemma} \label{lemma: kappan}
    When $\kappa_n = \frac{1}{R_{pitch}}$ or $-\frac{1}{R_{pitch}}$ the corresponding segment lies on spheres with radius $R_{pitch}$ whose center lies along $\mathbf{U}$ or $-\mathbf{U}$, respectively. Furthermore, such segments correspond to a maximum ascent or descent motion of the vehicle with a turning radius of $\frac{1}{\sqrt{\kappa_g^2 + \frac{1}{R_{pitch}^2}}}.$
\end{lemma}
\begin{proof}
    The proof is provided in Appendix~\ref{appsect: Proof_Lemma_1}.
\end{proof}

The following lemma states a similar result in a different axis, and the proof follows the same reasoning.
\begin{lemma} \label{lemma: kappag}
    When $\kappa_g = \pm \frac{1}{R_{yaw}},$ the corresponding segment lies on spheres with radius $R_{yaw}$ whose center lies along $\mathbf{Y}$ or $-\mathbf{Y}$. Furthermore, such segments correspond to maximum turn (left or right) motion of the vehicle with a turning radius of $\frac{1}{\sqrt{\frac{1}{R_{yaw}^2} + \kappa_n^2}}.$
\end{lemma}

From these two lemmas, we see that the normal curvature $\kappa_n$ governs the pitch motion, while the geodesic curvature $\kappa_g$ governs the yaw motion. In fact, these curvatures directly correspond to the vehicle’s pitch rate and yaw rate, respectively (which is expected based on the Bishop frame setup and the definitions in \eqref{eq: definitions}). 
When the vehicle moves with its maximum pitch rate and zero yaw rate, it follows a great circle of radius $R_{pitch}$ on the orange or purple sphere shown in Fig.~\ref{fig: 3D_config}. This result comes from Lemma~\ref{lemma: kappan}. Since the vehicle travels at unit speed, the time to complete the circle is $t_{pitch} = 2 \pi R_{pitch}.$ Over this time, the pitch angle changes by $2 \pi$, so the pitch rate is $\frac{2 \pi}{t_{pitch}} = \frac{1}{R_{pitch}}.$ A similar argument holds for the yaw rate, giving a maximum value of $\frac{1}{R_{yaw}}$. Therefore, $\kappa_n$ and $\kappa_g$ represent the vehicle's pitch and yaw rates, respectively.

By varying $\kappa_n \in \{-\frac{1}{R_{pitch}}, 0, \frac{1}{R_{pitch}} \}$ and $\kappa_g \in \{-\frac{1}{R_{yaw}}, 0, \frac{1}{R_{yaw}} \}$, we obtain nine distinct motion primitives, shown in Fig.~\ref{fig: optimal_segments}. 
These were generated using the closed-form expressions, presented in Appendix~\ref{appsubsect: construction_segments}. Using Lemma~\ref{lemma: kappan}, we find that the segments $L_{si},$ $R_{si},$ $L_{so},$ and $R_{so}$ have radius $\frac{1}{\sqrt{\frac{1}{R_{yaw}^2} + \frac{1}{R_{pitch}^2}}}$, corresponding to motion with maximum absolute pitch and yaw rates. Here, $L$ and $R$ denote a left turn and right turn, respectively, which correspond to $\kappa_g = \frac{1}{R_{yaw}}$ and $\kappa_g = -\frac{1}{R_{yaw}},$ respectively. Additionally, subscripts ``$si$" and ``$so$" are used to refer to the segments that lie on the ``inner" sphere and ``outer" sphere, respectively; the inner sphere corresponds to $\kappa_n = \frac{1}{R_{pitch}},$ and the outer sphere corresponds to $\kappa_n = -\frac{1}{R_{pitch}}.$
The segments $G_{si}$ and $G_{so}$ result from pure pitch motion ($\kappa_g = 0$), while $L_p$ and $R_p$ result from pure yaw motion. When both curvatures are zero ($\kappa_n = \kappa_g = 0$), the vehicle moves in a straight line segment $S$.


Using the obtained motion primitives and the observation that $\kappa_n$ and $\kappa_g$ attaining values of $\pm \frac{1}{R_{pitch}}$ and $\pm \frac{1}{R_{yaw}}$ yields two spheres each (a pair along $\mathbf{U}$ and a pair along $\mathbf{Y}$), we can observe that at both the initial and final configuration, four spheres exist around the vehicle.
Additionally, portions of the optimal path will lie on one of the four spheres at the initial configuration and one of the four spheres at the final configuration\footnotemark. Hence, we propose three classes of paths to construct a feasible path connecting one of the initial spheres with one of the final spheres. 
We construct the path using three types of intermediary surfaces (or classes): a cylindrical envelope, a planar surface, or a spherical surface. These three classes of paths are a generalization of the classical $CSC$ and $CCC$ paths for the planar Dubins problems. In our algorithm, we consider a sphere at the initial or final configuration to serve as a generalization of the turn segment ($C$) in a plane; furthermore, we consider the cylindrical envelope and planar surface to generalize the $S$ segment. In the following section, we describe the three classes of paths in more detail.

\footnotetext{The only motion primitive that does not lie on a sphere is a straight line segment $S.$ However, even in this case, a portion of the optimal path can be modeled to lie on one of the four spheres; however, the path will be trivial, i.e., of zero length.}

\begin{figure}[htb!]
    \centering
    \subfloat[Segments on spheres corresponding to max. pitch rate]{\includegraphics[width = 0.4\linewidth]{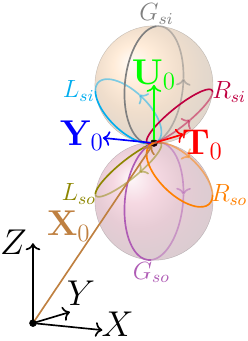}\label{subfig: sphere_segments}}
    \hfil
    \subfloat[Segments on spheres corresponding to max. yaw rate (and straight line segment)]{\includegraphics[width = 0.48\linewidth]{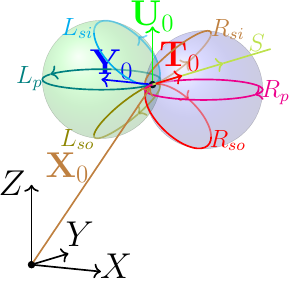}\label{subfig: planar_segments}}
    \caption{Visualization of segments \cite{kumar2025newapproachmotionplanning}. We note that $L_{si},$ $R_{si},$ $L_{so},$ and $R_{so}$ are shown in both subfigures, since each of these segments lies on two spheres.}
    \label{fig: optimal_segments}
\end{figure}

\begin{remark}
In general, the yaw and pitch rates for different UAVs may vary and can be coupled. The problem posed here is still of significant interest in obtaining lower and upper bounds for the shortest path length. One such case is illustrated by the region within the boundaries, shown in brown, in Fig.~\ref{fig: general_control_input}. However, by replacing the boundary with a rectangular region inscribed within this area, one can derive an upper bound that is a feasible solution. Similarly, by outer approximating the allowable region with a larger rectangular region, shown in green in Fig.~\ref{fig: general_control_input}, a lower bound for the optimal path length can be obtained.
\end{remark}

\begin{figure}[htb!]
    \centering
    \includegraphics[width=0.6\linewidth]{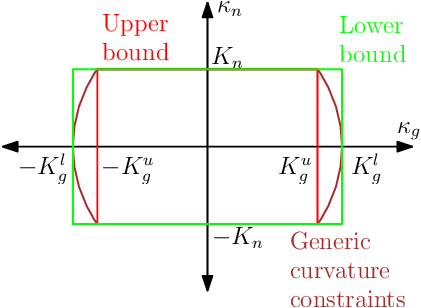}
    \caption{Generic control inputs region and obtaining bounds for rectangular control input region considered in this paper}
    \label{fig: general_control_input}
\end{figure}

\section{Shortest Path Construction} \label{sect: feasible_solution_construction}

The constructed paths start and end on a spherical surface at the initial and final configurations. Three distinct classes of paths are presented, each using a different intermediary surface for the sub-path between the spheres, which can be cylindrical, planar, or spherical. We refer to the spheres centered along the $\mathbf{U}$-axis as the inner (orange, along $\mathbf{U}$) and outer (purple, along $-\mathbf{U}$) spheres, as shown in Fig.~\ref{fig: 3D_config}. Similarly, the spheres centered along the $\mathbf{Y}$-axis are called the left (green, along $\mathbf{Y}$) and right (blue, along $-\mathbf{Y}$) spheres.

The intermediary sub-paths considered are as follows:
\begin{enumerate}
    \item In the first class, the sub-path is constructed using a cylindrical envelope. In this case, we connect the pair of spheres of the same type at the initial and final configurations. There are four such pairs: inner-to-inner, outer-to-outer, left-to-left, and right-to-right. Fig.~\ref{subfig: cylindrical_envelope} illustrates the cylindrical envelope between inner and outer spheres. We will later show that these paths satisfy the curvature constraints in \eqref{eq: curvature_constraints}. The full construction is detailed in Section~\ref{sect: cylindrical_envelope}.
    \item In the second class of paths, a sub-path between a pair of spheres is constructed on a cross-tangent plane.  For spheres of opposite type, such as inner-to-outer, outer-to-inner, left-to-right, or right-to-left, a path through a cylindrical envelope is not feasible. This is because the normal vector of the cylindrical surface, $\mathbf{U}$ for pitch spheres and $\mathbf{Y}$ for yaw spheres, remains constant along the envelope and does not support a continuous feasible orientation between opposing directions. Hence, such spheres are connected using a cross-tangent plane. An example for the inner-to-outer case is shown in Fig.~\ref{subfig: cross-tangent_planes}. Details of this construction are provided in Section~\ref{sect: cross_tangent_envelope}.
    \item In the third class of paths, we construct sub-paths between pairs of spheres of the same type using an intermediary sphere. There are four such configurations: inner–outer–inner, outer–inner–outer, left–right–left, and right–left–right. These paths are designed for initial and final locations that are close to each other. An example of this type is illustrated in Fig.~\ref{subfig: int_sphere_envelope}, and the full construction is described in Section~\ref{sect: sphere_envelope}.
\end{enumerate}

\begin{figure*}[htb!]
    \centering
    \subfloat[Cylindrical envelope]{\includegraphics[width = 0.34\linewidth]{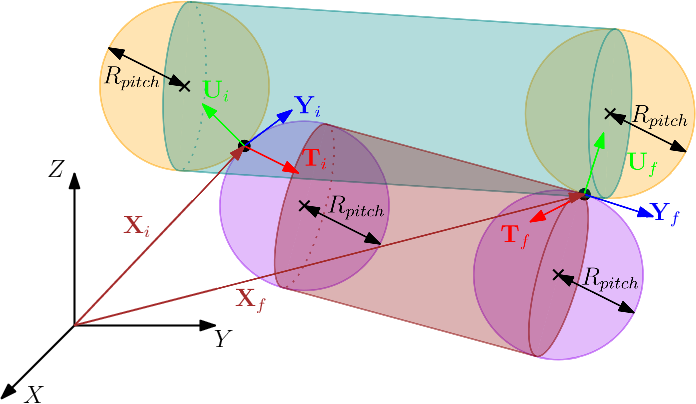}
    \label{subfig: cylindrical_envelope}}
    \hfil
    \subfloat[Sample cross-tangent plane]{\includegraphics[width = 0.34\linewidth]{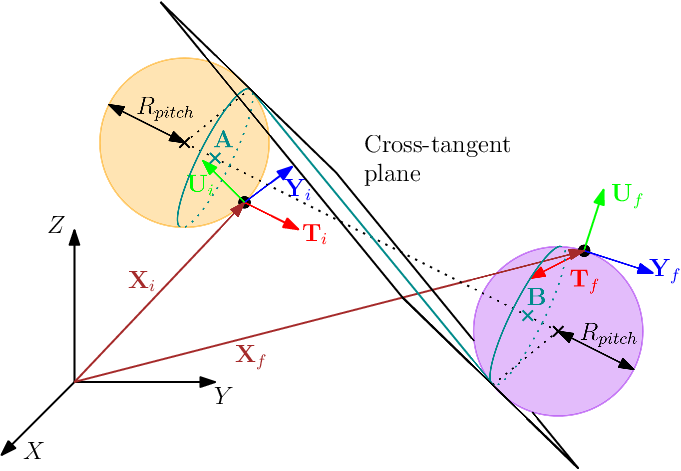}
    \label{subfig: cross-tangent_planes}}
    \hfil
    \subfloat[Connection through an intermediary sphere]{\includegraphics[width = 0.28\linewidth]{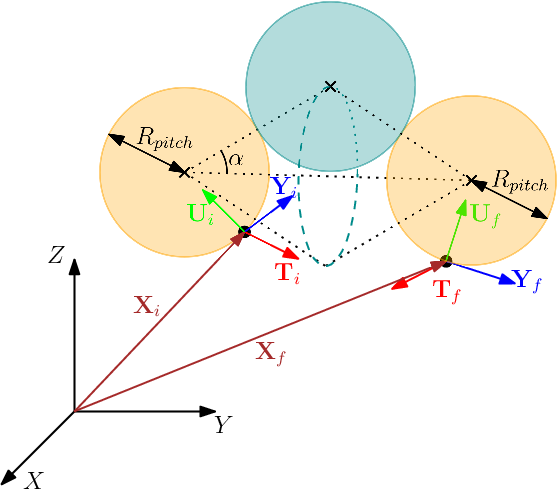}
    \label{subfig: int_sphere_envelope}}
    \caption{Depiction of surfaces used to connect spheres at the initial and final configurations}
    \label{fig: cylindrical_conical_envelopes}
\end{figure*}

\begin{remark}
    Note that only the listed classes are possible using a single intermediary surface. For example, connecting an inner sphere at the initial configuration with a left sphere at the final configuration is not possible. A cylindrical surface cannot be used because the outward normal directions differ: $-\mathbf{U}$ for the inner sphere and $-\mathbf{Y}$ for the left sphere. Since the normal vector remains constant across cylindrical, spherical, and frustum surfaces, none of these can bridge the two spheres. A planar surface also doesn’t work, as there is no common tangent plane between the two. For instance, the $\mathbf{T}-\mathbf{Y}$ plane is tangent to the inner sphere but not to the left sphere (see Fig.~\ref{fig: 3D_config}). Therefore, using a single intermediary surface, we can only connect either a pair of pitch spheres or a pair of yaw spheres, not a mix of both. The underlying reason is that $\mathbf{X},$ $\mathbf{T},$ $\mathbf{Y},$ and $\mathbf{U}$ must be continuous. The same argument applies for using an intermediary sphere as the tangential sphere to the initial and final spheres, i.e., only a right sphere can be used to connect two left spheres.
\end{remark}

\begin{remark}
    Although it is possible to connect a pair of spheres of the same type (e.g., inner–inner) using a plane, we do not consider such paths in this paper. This is because the planar connection is a special case of the cylindrical connection. This is because a plane can be wrapped into a cylinder without changing the path length or violating the curvature constraints. We discuss this preservation property in more detail when introducing the cylindrical path construction in the next section.
\end{remark}

\begin{remark}
    In Sections~\ref{sect: cylindrical_envelope}, \ref{sect: cross_tangent_envelope}, and \ref{sect: sphere_envelope}, we present the methodology for constructing three types of paths by introducing parameters that describe each path, which are subsequently discretized. Whenever we refer to the ``shortest'' path, we mean the least-length path obtained by our construction methodology under the chosen discretization; global optimality is not claimed. However, we note that the subpath constructed on each individual surface (cylinders, spheres, and planes) is optimal for that surface. Our methodology will always yield a feasible path, as at least the first class of path (through the cylindrical envelope) always exists.
\end{remark}

\section{Path Synthesis on Cylindrical Envelope} \label{sect: cylindrical_envelope}

To generate a feasible path connecting two spheres of the same radius and type via a cylindrical envelope (as shown in Fig.~\ref{subfig: cylindrical_envelope}), the vehicle follows this sequence:

\begin{itemize}
\item Step 1: The path starts on the initial sphere and transitions to the cylindrical surface. The transition point, $\mathbf{X}_{ic}$,  lies on the boundary circle formed by the intersection of the sphere and the cylindrical envelope. At this point, its longitudinal direction aligns with $\mathbf{T}_{ic},$ as illustrated in Fig.~\ref{fig: notation_discretization}.
\item Step 2: The path exits the cylindrical envelope at $\mathbf{X}_{oc}$, with longitudinal direction $\mathbf{T}_{oc}$.
\item Step 3: Finally, the path continues on the final sphere to reach the desired final configuration.
\end{itemize}

\begin{figure}[htb!]
    \centering
    \includegraphics[width = 0.75\linewidth]{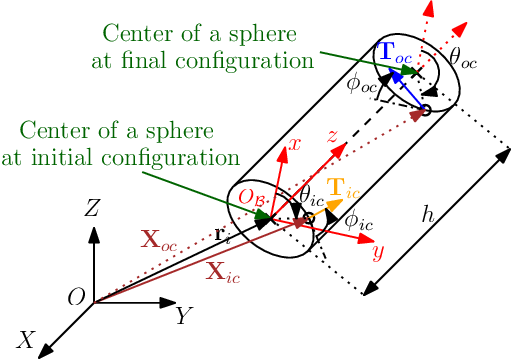}
    \caption{Notation for discretization of position and headings on a cylinder connecting two spheres}
    \label{fig: notation_discretization}
\end{figure}

Note that the entry and exit points on the cylinder, along with their corresponding tangent directions, are not fixed and can be freely chosen. We parameterize these directions using four angles: $\theta_{ic},$ $\phi_{ic},$ $\theta_{oc},$ and $\phi_{oc}.$ The following section describes how the path on the spheres and the cylinder is constructed based on these four parameters.

\begin{remark}
    Since cylinders connect spheres of the same type, the cylinder always has the same radius as the spheres and thus does not expand or shrink.
\end{remark}

\subsection{Origins of the Spheres and Axes of the Cylinders} \label{subsect: origin_sphere_cylinder_axes}

We begin by deriving expressions for the centers of the spheres at the initial and final configurations, denoted by $\mathbf{r}_i$ and $\mathbf{r}_{f},$ respectively. These vectors are given by
\begin{align}
    \mathbf{r}_i &= \mathbf{X}_i + \delta^{initial}_{i, o} R_{pitch} \mathbf{U}_i + \delta^{initial}_{l, r} R_{yaw} \mathbf{Y}_i, \label{eq: initial_sphere_center_expression} \\
    \mathbf{r}_f &= \mathbf{X}_f + \delta^{final}_{i, o} R_{pitch} \mathbf{U}_f + \delta^{final}_{l, r} R_{yaw} \mathbf{Y}_f. \label{eq: final_sphere_center_expression}
\end{align}
Here, $\delta^{initial}_{i, o} = 1, -1$ or $0$ depending on whether the inner sphere, outer sphere, or one of the left/right spheres is selected at the initial configuration, respectively. Similarly, $\delta^{initial}_{l, r} = 1, -1, 0$ if the left sphere, right sphere, or one of the inner/outer spheres is chosen. The same interpretation applies for $\delta^{final}_{i, o}$ and $\delta^{final}_{l, r}.$
For cylindrical envelope constructions, we require that $\delta^{initial}_{i, o} = \delta^{final}_{i, o}$ and $\delta^{initial}_{l, r} = \delta^{final}_{l, r}$.

We can obtain the axis of the cylinder that connects the selected pair of spheres as (refer to Fig.~\ref{subfig: cylindrical_envelope})
\begin{align}
    \mathbf{k} &= \frac{\mathbf{r}_{f} - \mathbf{r}_{i}}{\|\mathbf{r}_{f} - \mathbf{r}_{i}\|_2}. \label{eq: k_definition}
\end{align}
Furthermore, the length of the cylinder is given by $h = \|\mathbf{r}_{f} - \mathbf{r}_{i}\|_2.$ Since the radius of the cylinder is the same as the radius of the selected pair of spheres, the cylinder's radius is $R_{pitch}$ if $\delta^{inner}_{i, o} \neq 0$ and $R_{yaw}$ if $\delta^{inner}_{l, r} \neq 0.$

We now derive expressions for the entry and exit points on the cylinder ($\mathbf{X}_{ic}$ and $\mathbf{X}_{oc}$) as well as their corresponding tangent directions ($\mathbf{T}_{ic}$ and $\mathbf{T}_{oc}$).

\subsection{Parameters for the location and tangent vector on the cylinder}

On the cylindrical envelope, we parameterize the entry point $\mathbf{X}_{ic}$ and the tangent $\mathbf{T}_{ic}$ by two angles, $\theta_{ic}$ and $\phi_{ic}$ (see Fig.~\ref{fig: notation_discretization}). Likewise, $\theta_{oc}$ and $\phi_{oc}$ parameterize the exit point $\mathbf{X}_{oc}$ and the corresponding tangent $\mathbf{T}_{oc}$. To derive these expressions, we introduce a body frame $\mathcal{B} (O_B, x, y, z)$ centered at the cylinder’s base point $\mathbf{r}_i$, with its $z$‑axis aligned along the cylinder axis (also shown in Fig.~\ref{fig: notation_discretization}).\footnotemark

\footnotetext{
Since the cylinder’s $z$-axis is known, but the $x$- and $y$-axes of the body frame are not, we begin by aligning the $x$-axis with the global $X$-axis. We then apply Gram–Schmidt orthogonalization to compute a unit vector perpendicular to the $z$-axis. If the dot product between the global $X$-axis and the cylinder’s $z$-axis is close to one (i.e., they are nearly aligned), we instead use the global $Y$-axis to initialize the process.
}

The expressions for $\mathbf{X}_{ic}$ and $\mathbf{X}_{oc}$ can be derived in the body frame $\mathcal{B}$ to be
\begin{align} \label{eq: expression_Xic_Xoc_B}
    \mathbf{X}_{ic}^\mathcal{B} &= \begin{pmatrix}
        \overline{R} \cos{\theta_{ic}} \\
        \overline{R} \sin{\theta_{ic}} \\ 0
    \end{pmatrix}, \quad
    \mathbf{X}_{oc}^\mathcal{B} = \begin{pmatrix}
        \overline{R} \cos{\theta_{oc}} \\
        \overline{R} \sin{\theta_{oc}} \\
        h
    \end{pmatrix},
\end{align}
where $\overline{R}$ is the radius of the cylinder, and is given by
\begin{align} \label{eq: Rbar_definition}
    \overline{R} = R_{pitch} |\delta^{initial}_{i, o}| + R_{yaw} |\delta^{initial}_{l, r}|.
\end{align}
We can derive the direction cosines of the tangent vector when it enters and exits the cylinder as (refer to Fig.~\ref{fig: notation_discretization})
\begin{align*}
    \mathbf{T}_{ic}^\mathcal{B} &= \mathbf{R}_z (\theta_{ic}) \mathbf{R}_x (\phi_{ic}) \begin{pmatrix} 0 \\ 1 \\ 0 \end{pmatrix} = \begin{pmatrix}
        -\sin{\theta_{ic}} \cos{\phi_{ic}} \\
        \cos{\theta_{ic}} \cos{\phi_{ic}} \\
        \sin{\phi_{ic}}
    \end{pmatrix}, \\
    \mathbf{T}_{oc}^\mathcal{B} &= \mathbf{R}_z (\theta_{oc}) \mathbf{R}_x (\phi_{oc}) \begin{pmatrix} 0 \\ 1 \\ 0 \end{pmatrix} = 
    \begin{pmatrix}
        -\sin{\theta_{oc}} \cos{\phi_{oc}} \\
        \cos{\theta_{oc}} \cos{\phi_{oc}} \\
        \sin{\phi_{oc}}
    \end{pmatrix}.
\end{align*}
Here, $\mathbf{R}_z$ and $\mathbf{R}_x$ are standard {\it elementary rotation matrices} for rotation about the $z$ and $x$ axis, respectively.

The entry position $\mathbf{X}_{ic}$ and tangent direction $\mathbf{T}_{ic}$ in the global frame $\mathcal{G} (O, X, Y, Z)$ can be expressed as
\begin{align}
    \mathbf{X}_{ic}^\mathcal{G} &= \begin{pmatrix}
        \mathbf{x} & \mathbf{y} & \mathbf{z}
    \end{pmatrix} \mathbf{X}_{ic}^\mathcal{B} + \begin{pmatrix}
        X_{O_\mathcal{B}} \\ Y_{O_\mathcal{B}} \\ Z_{O_\mathcal{B}}
    \end{pmatrix}, \label{eq: Eulerian_viewpoint_position} \\
    \mathbf{T}_{ic}^\mathcal{G} &= \begin{pmatrix}
        \mathbf{x} & \mathbf{y} & \mathbf{z}
    \end{pmatrix} \mathbf{T}_{ic}^\mathcal{B}, \label{eq: Eulerian_viewpoint_tangent}
\end{align}
where $\mathbf{x}, \mathbf{y},$ and $\mathbf{z}$ are unit vectors along the $x, y,$ $z$ axes of the body frame $\mathcal{B},$ and $X_{O_\mathcal{B}}, Y_{O_\mathcal{B}}, Z_{O_\mathcal{B}}$ is the location of the body frame's origin. Analogous expressions hold for $\mathbf{X}_{oc}$ and $\mathbf{T}_{oc}$.

With the expressions for the entry and exit locations and their corresponding tangent vectors on the cylindrical envelope now established, we proceed to construct the optimal path on the initial and final spheres, as well as on the cylindrical envelope.

\subsection{Generation of paths on initial and final spheres} \label{subsect: initial_final_spheres_path_generation}

Consider the chosen sphere at the initial configuration. We need to obtain the optimal path connecting the initial configuration to the location $\mathbf{X}_{ic}^\mathcal{G}$ with heading direction given by $\mathbf{T}_{ic}^\mathcal{G}$ to enter the cylindrical envelope (as shown in Fig.~\ref{fig: notation_discretization}). 
We simplify this problem by translating the sphere's center to the origin. This allows us to analyze the motion using a Sabban frame \cite{3D_Dubins_sphere, kumar2025newapproachmotionplanning}. The task becomes finding the optimal path on the sphere's surface that connects an initial location $\mathbf{X}_{sp, 0}$ and tangent $\mathbf{T}_{sp, 0}$ to a final location and tangent.


In \cite{3D_Dubins_sphere, kumar2025newapproachmotionplanning}, the Sabban frame model was used to study motion planning on a \textit{unit} sphere. The configuration of the vehicle was specified by a location $\hat{\mathbf{X}}_{sp}$ (which is a unit vector pointing radially outwards), a tangent vector $\mathbf{T}_{sp}$ along the longitudinal direction of the vehicle, and a normal vector $\mathbf{N}_{sp}$ along the lateral direction. Additionally, the path was parametrized in terms of arc length $\hat{s}.$ The evolution equations for these vectors are given by
\begin{align}
    &\frac{d\hat{\mathbf{X}}_{sp}}{d\hat{s}} (\hat{s}) = \mathbf{T}_{sp} (\hat{s}), \quad
    \frac{d\mathbf{T}_{sp}}{d\hat{s}} (\hat{s}) = - \hat{\mathbf{X}}_{sp} (\hat{s}) + \hat{u}_g \mathbf{N}_{sp} (\hat{s}), \nonumber \\
    &\frac{d\mathbf{N}_{sp}}{d \hat{s}} (\hat{s}) = - \hat{u}_g \mathbf{T}_{sp} (\hat{s}),
\end{align}
where $\hat{u}_g \in [-\hat{U}_{max}, \hat{U}_{max}]$ is the geodesic curvature on the unit sphere and serves as the control input. It relates to the minimum turning radius $\hat{r}$ on the unit sphere by $\hat{r} = \frac{1}{\sqrt{1 + \hat{U}_{max}^2}}.$

We can adapt the previous results for motion planning on a sphere of any radius ($\overline{R}$) by scaling the problem to a unit sphere problem.\footnotemark\, First, we compute the normal vector $\mathbf{N}_{sp} := \frac{1}{\overline{R}} \mathbf{X}_{sp} \times \mathbf{T}_{sp}.$ While scaling does not affect the tangent vector $\mathbf{T}_{sp}$ or the normal vector $\mathbf{N}_{sp}$, other parameters change. The location on the unit sphere becomes $\hat{\mathbf{X}}_{sp} := \frac{1}{\overline{R}} \mathbf{X}_{sp}$, and the corresponding minimum turning radius becomes $\hat{r} = \frac{1}{\overline{R}} r.$ A detailed derivation of this scaling is available in Appendix~\ref{appsubsect: Sabban_frame_generalization}.
\footnotetext{We note here that we perform this scaling since $\mathbf{X}_{sp}, \mathbf{T}_{sp},$ and $\mathbf{N}_{sp}$ do not form a rotation matrix as $\mathbf{X}_{sp}$ is not a unit vector. However, when the problem is scaled to motion planning on a unit sphere, these vectors form a rotation matrix; therefore, the path can be constructed easily.}

\begin{figure}[htb!]
    \centering
    \includegraphics[width=0.7\linewidth]{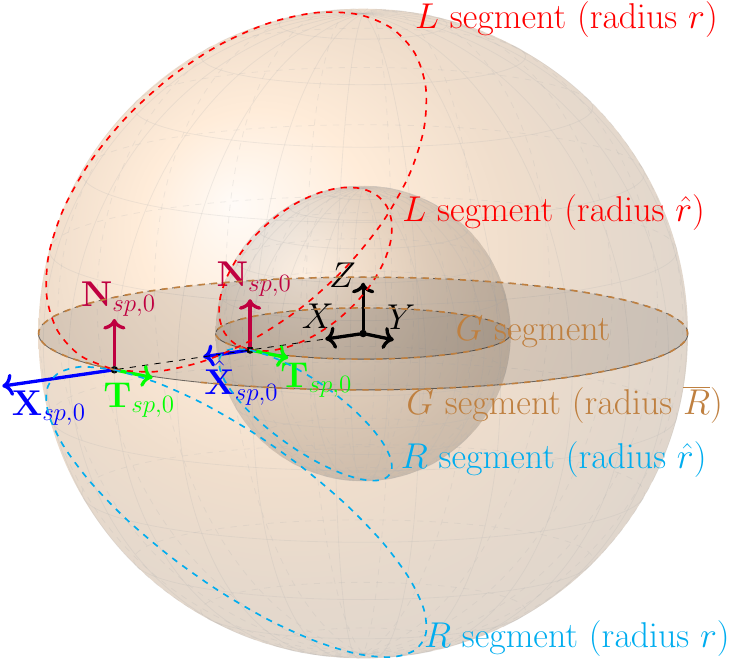}
    \caption{Motion planning on sphere with radius $\overline{R}$ and on a unit sphere}
    \label{fig: sphere_motion_planning}
\end{figure}

From \cite{kumar2025newapproachmotionplanning}, the candidate optimal paths on a unit sphere are of type $CGC, CCC$ for $\hat{r} \leq \frac{1}{2},$ $CGC, CCCC,$ or a degenerate path for $\frac{1}{2} < \hat{r} \leq \frac{1}{\sqrt{2}},$ and $CGC,$ $CCCCC,$ or $CC_\pi C$\footnotemark\footnotetext{$C_\pi$ refers to a $C$ segment (i.e., a left turn or right turn of minimum turning radius) with an arc angle of exactly $\pi$ radians.} for $\frac{1}{\sqrt{2}} < \hat{r} \leq \frac{\sqrt{3}}{2}.$\footnotemark\footnotetext{The optimal path candidates for spherical motion planning with $\hat{r} > \frac{\sqrt{3}}{2}$ remain an open problem. It has been hypothesized that, as $\hat{r} \rightarrow 1$, an increasing number of path concatenations is required due to the progressively limited maneuverability of the vehicle, with the number of concatenations tending to infinity.} The analytical computation of the arc angles for each path is provided in \cite{kumar2025generationpathsmotionplanning}. We note here that the arc angles of the segments of a path on the unit sphere and the corresponding path on the sphere with radius $\overline{R}$ would be the same. Hence, we can obtain the arc angles of the segments for each candidate path on the sphere with radius $\overline{R}$ using \cite{kumar2025generationpathsmotionplanning}.

\begin{remark}
    The arc angle $\phi$ is related to the segment length $l$ by $l = \hat{r} \phi$ for $L$ and $R$ segments, and $l = \phi$ for a $G$ segment on a unit sphere.
\end{remark}


Finally, we want to obtain the expressions for $\mathbf{X}_{sp},$ $\mathbf{T}_{sp},$ and $\mathbf{N}_{sp}$ along the path to describe the instantaneous configuration of the vehicle along the sphere. We can obtain these expressions by solving the Sabban frame equations, derived in Appendix~\ref{appsubsect: Sabban_frame_generalization}, 
using the Euler-Rodriguez formula. Therefore, the configuration of the vehicle on the sphere along the path can be obtained. 

The last step to be performed is to obtain the configuration of the vehicle in 3D (which utilizes the rotation minimizing frame). Since we had shifted the origin of the sphere to coincide with the origin of the global frame $\mathcal{G},$ 
the location ($\mathbf{X}$) and longitudinal direction ($\mathbf{T}$) 
can be easily obtained as
\begin{align*}
    \mathbf{X} (s) = \mathbf{X}_{sp} (s) + \mathbf{r}_i, \quad \mathbf{T} (s) = \mathbf{T}_{sp} (s).
\end{align*}
Furthermore, depending on the type of sphere chosen at the initial configuration, $\mathbf{Y}$ and $\mathbf{U}$ can be computed (refer to Fig.~\ref{fig: sphere_motion_planning}). If $\delta_{i, o}^{initial} \neq 0$ or $\delta_{l, r}^{initial} \neq 0,$ the expressions for $\mathbf{U}$ (the surface normal) and $\mathbf{Y}$ (the tangent normal) are obtained, respectively, as (refer to Figs.~\ref{fig: 3D_config} and \ref{fig: sphere_motion_planning})
\begin{align*}
    \mathbf{U} (s) &= -\delta_{i, o}^{initial} \frac{1}{\overline{R}} \mathbf{X}_{sp} (s), \quad \delta_{i, o}^{initial} \neq 0, \\
    \mathbf{Y} (s) &= -\delta_{l, r}^{initial} \frac{1}{\overline{R}} \mathbf{X}_{sp} (s), \quad \delta_{l, r}^{initial} \neq 0.
\end{align*}
When $\delta_{i, o}^{initial} \neq 0,$ $\mathbf{Y} = \mathbf{U} \times \mathbf{T};$ when $\delta_{l, r}^{initial} \neq 0,$ $\mathbf{U} = \mathbf{T} \times \mathbf{Y}.$ The path on the sphere at the final configuration is constructed similarly.

\subsection{Generation of path on cylinder}

In this section, we describe the construction of the optimal Dubins path on a cylindrical surface. 
This path connects an initial and final configuration on the cylinder, which we had previously parameterized in terms of $\theta_{ic},$ $\phi_{ic},$ $\theta_{oc},$ and $\phi_{oc}.$ 
The path we construct on the cylinder must obey the geodesic curvature (yaw rate) and normal curvature (pitch rate) constraints for the 3D model. We note that the radius of the cylinder is $\overline{R},$ whose definition is given in \eqref{eq: Rbar_definition}. We choose the bound on the geodesic curvature for motion over the cylinder to be
\begin{align} \label{eq: geodesic_curvature_cylinder}
    |\kappa_{g, cyc}| \leq \begin{cases}
        \frac{1}{R_{pitch}}, & \delta_{i, o}^{initial} = 0, \\
        \frac{1}{R_{yaw}}, & \delta_{l, r}^{initial} = 0.
    \end{cases}
\end{align}
We claim that the considered radius for the cylinder and the geodesic curvature bounds satisfy the geodesic curvature and normal curvature constraints for the 3D problem.

\begin{lemma} \label{lemma: cylinder}
The optimal path on a cylinder of radius $\overline{R},$ defined in \eqref{eq: Rbar_definition}, with geodesic curvature bounds given by \eqref{eq: geodesic_curvature_cylinder} satisfies the geodesic curvature and normal curvature bounds for the proposed rotation minimizing frame model.
\end{lemma}
\begin{proof}
    The proof is provided in Appendix~\ref{appsubsect: proof_lemma_cylinder}.
\end{proof}

We present the construction of the optimal path on the cylinder. Note that geodesic curvature is bending invariant \cite{struik}. Hence, we can unwrap the cylinder onto a plane, as shown in Fig.~\ref{fig: cylinder_unwrapping}, and construct the optimal path on the plane. Finally, the constructed path on the plane can be wrapped back onto the cylinder. 

\subsubsection{Unwrapping frame for cylinder}

For unwrapping the cylinder, we consider a frame $\mathcal{U},$ referred to as the unwrapping frame, with axes $x_\mathcal{U},$ $y_\mathcal{U},$ and $z_\mathcal{U}$. The origin for $\mathcal{U}$ is at the point of entry of the cylinder ($\mathbf{X}_{ic}$). Furthermore, $z_\mathcal{U}$ is parallel to $z$ and $y_\mathcal{U}$ points radially inwards to the cylinder. 
We will use the unwrapping frame for constructing the path on the plane. Once we construct such a path, we will represent it in the body frame $\mathcal{B},$ and finally obtain the vehicle's configuration in the global frame $\mathcal{G}$ for the 3D problem.

Consider a point $Q$ on the cylinder.
The relationship between its location in the unwrapping frame ($\mathbf{X}_Q^\mathcal{U}$) and the body frame ($\mathbf{X}_Q^\mathcal{B}$) is given by\footnotemark\footnotetext{For motion planning on a cylinder where the initial location (equivalent to the origin of $\mathcal{U}$) is not in the $xy$ plane, the last term in \eqref{eq: coordinate_transformation_cylinder_body} can be replaced with $(\overline{R} \cos{\theta_{ic}}, \overline{R} \sin{\theta_{ic}}, d_{ic})^T.$ Here, $d_{ic}$ is the distance from the $xy$ plane.}
\begin{align} \label{eq: coordinate_transformation_cylinder_body}
    \mathbf{X}_Q^\mathcal{B} &= \begin{pmatrix}
        -\sin{\theta_{ic}} & -\cos{\theta_{ic}} & 0 \\
        \cos{\theta_{ic}} & -\sin{\theta_{ic}} & 0 \\
        0 & 0 & 1
    \end{pmatrix} \mathbf{X}_Q^\mathcal{U} + \begin{pmatrix}
        \overline{R} \cos{\theta_{ic}} \\
        \overline{R} \sin{\theta_{ic}} \\
        0
    \end{pmatrix}.
\end{align}

\begin{figure}[htb!]
    \centering
    \subfloat[Unwrapping frame chosen]{\includegraphics[width = 0.35\linewidth]{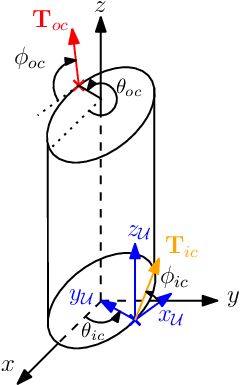}
    \label{subfig: body_frame_cylinder}} \hfill
    \subfloat[Unwrapping plane chosen]{\includegraphics[width = 0.5\linewidth]{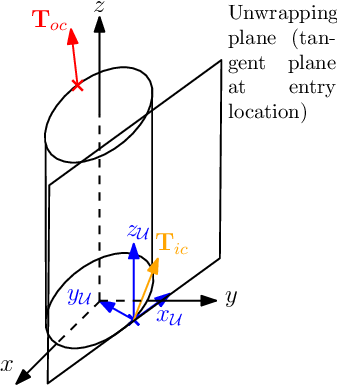}
    \label{subfig: cylinder_unwrapping_plane}}
    \caption{Frames on the cylinder and the unwrapping plane chosen for the cylinder}
    \label{fig: cylinder_unwrapping}
\end{figure}

Using \eqref{eq: coordinate_transformation_cylinder_body}, we can represent the entry location $\mathbf{X}_{ic}$ and the exit location $\mathbf{X}_{oc}$ in the unwrapping frame $\mathcal{U}$ as
$\mathbf{X}_{ic}^\mathcal{U} = (0, 0, 0)^T$ and $\mathbf{X}_{oc}^\mathcal{U} = (\overline{R} \sin{(\delta \theta)}, \overline{R} (1 - \cos{(\delta \theta)}), h)^T.$
Here, the expression for $\mathbf{X}_{oc}^\mathcal{B}$ from \eqref{eq: expression_Xic_Xoc_B} was used, and $\delta \theta := \theta_{oc} - \theta_{ic}.$

We now aim to unwrap the cylindrical surface onto a plane, selecting the tangent plane at $\mathbf{X}_{ic}$ as reference. Therefore, the unwrapping plane is defined by $x_\mathcal{U}$ and $z_\mathcal{U}$ axes, as shown in Fig.~\ref{subfig: cylinder_unwrapping_plane}. We will now describe the mapping of the initial and final configurations of the cylinder to the unwrapping plane. 

\subsubsection{Configurations after unwrapping cylinder}
Consider unwrapping a point on the cylinder as shown in Fig.~\ref{fig: unwrapping_point_right_circular_cylinder}. A point $P,$ whose coordinates are $(\overline{R} \sin{(\delta \theta)}, \overline{R} (1 - \cos{(\delta \theta)}), \delta d)$ in $\mathcal{U},$ gets mapped to two points on the plane due to periodicity of the angle $\delta \theta \in (-\pi, \pi]$\footnotemark. Hence, the two images of $P$ obtained on the plane, shown in Fig.~\ref{fig: unwrapping_point_right_circular_cylinder}, are given by $P_1 (\overline{R} \theta_1, \delta d)$ and $P_2 (\overline{R} \theta_2, \delta d),$ where
\begin{align}
    \theta_1 &= \begin{cases}
        \delta \theta, & \delta \theta \geq 0 \\
        \delta \theta + 2 \pi, & \delta \theta < 0
    \end{cases}, \quad
    \theta_2 = \begin{cases}
        \delta \theta - 2 \pi, & \delta \theta > 0 \\
        \delta \theta, & \delta \theta \leq 0
    \end{cases}.
\end{align}
\footnotetext{In principle, there are infinitely many images due to the periodicity of the angle $\delta \theta.$ However, we consider only two images for simplicity.}

\begin{figure}[htb!]
    \centering
    \includegraphics[width = 0.9\linewidth]{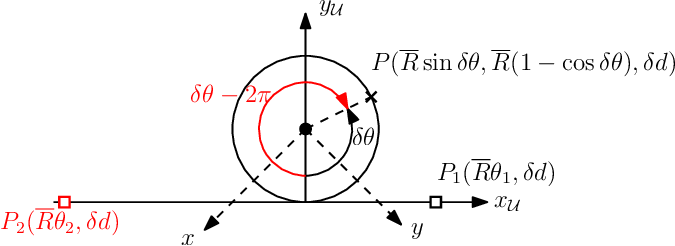}
    \caption{Unwrapping point lying on the cylinder}
    \label{fig: unwrapping_point_right_circular_cylinder}
\end{figure}


The two images corresponding to the final configuration, which is the exit location of the cylinder, obtained on the unwrapping plane, are shown in Fig.~\ref{fig: initial_final_configs_unwrapping_plane}. It can be observed that the heading angles for the entry and exit locations are $\phi_{ic}$ and $\phi_{oc}$ on the plane, respectively (compare with Fig.~\ref{fig: cylinder_unwrapping}).
\begin{figure}[htb!]
    \centering
    \includegraphics[width = 0.4\linewidth]{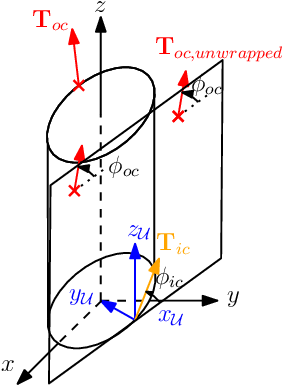}
    \caption{Initial configuration and two images of the final configuration obtained on the unwrapping plane}
    \label{fig: initial_final_configs_unwrapping_plane}
\end{figure}

We can now plan the optimal path to each image of the final configuration on the plane. To this end, we generate the six 2D Dubins candidate paths ($CSC$ and $CCC$) using the analytical expressions provided in \cite{dubins_classification} to each image, and pick the shortest path. Let this shortest path be
\begin{align} \label{eq: Pplane}
    \mathbf{X}_{plane} (s) = (u (s), v (s))^T,    
\end{align}
where $s$ is the arc length. Additionally, let the instantaneous heading angle on the plane be $\psi (s),$ which is the angle made with respect to $x_\mathcal{U}.$

\subsubsection{Wrapping path onto cylinder}
After the path is generated on the unwrapped plane, the corresponding path on the cylinder is retrieved by inverse mapping. To this end, consider a point given by $(\overline{R} \theta, d)$ on the unwrapping plane. The corresponding image of this point on the cylinder will be $(\overline{R} \sin{\theta}, \overline{R} (1 - \cos{\theta}), d)$ in $\mathcal{U}$ using the previously established procedure (refer to Fig.~\ref{fig: unwrapping_point_right_circular_cylinder}). 

Now, consider the curve on the plane given by \eqref{eq: Pplane}. Using the previous argument, the corresponding curve obtained on the cylinder in the unwrapping frame $\mathcal{U}$ is given by
\begin{align} \label{eq: Pcyc}
    \mathbf{X}_{cyl}^\mathcal{U} (s) = \left(\overline{R} \sin{\left(\frac{u (s)}{\overline{R}} \right)}, \overline{R} \left(1 - \cos{\left(\frac{u (s)}{\overline{R}} \right)} \right), v (s) \right)^T.
\end{align}
We now prove that the proposed mapping preserves the length of the curve.

\begin{lemma} \label{lemma: mapping_preservation}
    The proposed mapping between a planar curve and the wrapped curve on the cylindrical surface preserves the length of the curve.
\end{lemma}
\begin{proof}
    The proof is provided in Appendix~\ref{appsubsect: Lemma_mapping_proof}.
\end{proof}
Using \eqref{eq: coordinate_transformation_cylinder_body}, the equation of the considered curve in \eqref{eq: Pcyc} can be obtained in the body frame $\mathcal{B}$ as
\begin{align}
\begin{split}
    \mathbf{X}^\mathcal{B} (s) 
    &= \begin{pmatrix}
        \overline{R} \cos{\left(\theta_{ic} + \frac{u (s)}{\overline{R}} \right)} \\
        \overline{R} \sin{\left(\theta_{ic} + \frac{u (s)}{\overline{R}} \right)} \\
        v (s)
    \end{pmatrix}.
\end{split}
\end{align}

We compute the tangent vector along the path using the instantaneous heading angle $\psi (s)$ obtained from the 2D Dubins path on the $x_\mathcal{U} z_{\mathcal{U}}$ plane.
Hence, the angle made by $\mathbf{T}$ in the unwrapped plane with respect to $x_\mathcal{U}$, which is the heading angle, is known (refer to Fig.~\ref{fig: initial_final_configs_unwrapping_plane}). From Fig.~\ref{subfig: body_frame_cylinder}, the direction cosines of $\mathbf{T}$ expressed in the body frame can be obtained as
\begin{align*}
    \mathbf{T}^\mathcal{B} (s) = \begin{pmatrix}
        -\sin{\left(\theta_{ic} + \frac{u (s)}{\overline{R}} \right)} \cos{\left(\psi (s) \right)} \\
        \cos{\left(\theta_{ic} + \frac{u (s)}{\overline{R}} \right)} \cos{\left(\psi (s) \right)} \\
        \sin{\left(\psi (s) \right)}
    \end{pmatrix}.
\end{align*}
If the inner or outer sphere was chosen at the initial configuration, the expression for $\mathbf{U}$ in $\mathcal{B}$ can be obtained by noting that it is radially outwards or inwards to the cylinder, as (refer to Fig.~\ref{subfig: cylindrical_envelope})
\begin{align*}
    \mathbf{U}^\mathcal{B} = -\delta_{i, o}^{initial} \left(
        \cos{\left(\theta_{ic} + \frac{u (s)}{\overline{R}} \right)}, \sin{\left(\theta_{ic} + \frac{u (s)}{\overline{R}} \right)}, 0 \right)^T.
\end{align*}
Alternatively, if the left or right spheres were chosen, the same expression for $\mathbf{Y}^\mathcal{B}$ is obtained with $\delta_{i, o}^{initial}$ replaced with $\delta_{l, r}^{initial}.$ The expression for $\mathbf{Y}$ when $\delta_{i, o}^{initial} \neq 0$ and $\mathbf{U}$ when $\delta_{l, r}^{initial} \neq 0$ can be obtained as $\mathbf{Y}^\mathcal{B} = \mathbf{U}^\mathcal{B} \times \mathbf{T}^\mathcal{B}$ and $\mathbf{U}^\mathcal{B} = \mathbf{T}^\mathcal{B} \times \mathbf{Y}^\mathcal{B},$ respectively.

Finally, the expressions for $\mathbf{X},$ $\mathbf{T},$ $\mathbf{Y},$ and $\mathbf{U}$ can be obtained in the global frame $\mathcal{G}$ using \eqref{eq: Eulerian_viewpoint_position} and \eqref{eq: Eulerian_viewpoint_tangent} (refer to Fig.~\ref{fig: notation_discretization})\footnotemark.\, Hence, we have obtained the configuration of the vehicle along the shortest path on the cylinder for chosen $\theta_{ic},$ $\phi_{ic},$ $\theta_{oc},$ and $\phi_{oc}$ values; this path satisfies the pitch and yaw rate constraints of the vehicle.

\footnotetext{Similar to the transformation for the tangent vector in \eqref{eq: Eulerian_viewpoint_tangent} between the body and global frames, equations for the tangent-normal and surface-normal vectors can be obtained.}

\begin{remark}
    Though four parameters were introduced for path construction using a cylindrical envelope in the beginning of Section~\ref{sect: cylindrical_envelope}, the initial and final sphere computations depend only on two parameters each ($\theta_{ic}$ and $\phi_{ic},$ or $\theta_{oc}$ and $\phi_{oc}$). Furthermore, the motion planning on the cylinder depends on $\phi_{ic},$ $\phi_{oc},$ and the difference between $\theta_{ic}$ and $\theta_{oc}.$
\end{remark}

To compute the best feasible path for a selected pair of spheres at the initial and final configuration, we discretize $\theta_{ic}$ and $\theta_{oc}$ in $[0, 2 \pi).$ We discretize $\phi_{ic}$ and $\phi_{oc}$ over the interval $[0, \pi]$, representing the feasible interval of heading angles that allow the path to enter the cylinder at $\mathbf{X}_{ic}$ and exit at $\mathbf{X}_{oc}$ (refer to Fig.~\ref{fig: notation_discretization}). The number of discretizations of $\theta_{ic}$ and $\theta_{oc}$ are determined by a parameter $\theta_{disc},$ whereas $\phi_{disc}$ dictates the number of discretizations for $\phi_{ic}$ and $\phi_{oc}$.
We choose the combination that yields the shortest feasible path.


\section{Constructing Feasible Solution using Cross-Tangent Plane} \label{sect: cross_tangent_envelope}

In this section, we describe the second class of paths where the sub-path between the initial and final spheres is constructed on a cross-tangent plane.
There exist infinitely many cross-tangent planes between these two spheres; the locus of the point of intersection of these cross-tangent planes with the initial/final sphere will be a circle, as shown in Fig.~\ref{fig: cross-tangent_planes_parametrization}. 
To uniquely define a cross-tangent plane, we use angle $\theta$ as a parameter (see Fig.~\ref{fig: cross-tangent_planes_parametrization}).

\begin{figure}[htb!]
    \centering
    \includegraphics[width = 0.9\linewidth]{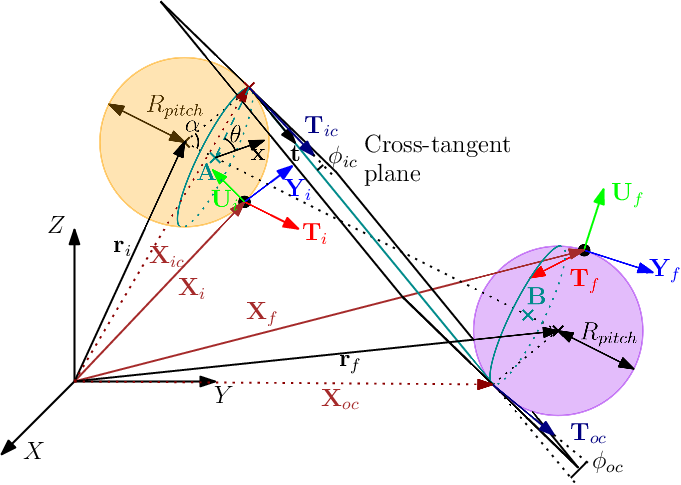}
    \caption{Parameterization of family of planes and configurations at entry and exit from cross-tangent plane}
    \label{fig: cross-tangent_planes_parametrization}
\end{figure}

\begin{remark}
    From Fig.~\ref{fig: cross-tangent_planes_parametrization}, we can observe that the considered cross-tangent plane exists when $\|\mathbf{r}_f - \mathbf{r}_i \|_2 \geq 2 \overline{R},$ i.e., when the spheres at the initial and final configurations do not intersect.
\end{remark}

We denote the center of the locus at the initial and final spheres by $\mathbf{A}$ and $\mathbf{B}$, respectively, as shown in Fig.~\ref{fig: cross-tangent_planes_parametrization}. The location of $\mathbf{A}$ and $\mathbf{B}$ can be obtained as
\begin{align}
    \mathbf{A} &= \mathbf{r}_i + \overline{R} \cos{\alpha} \left(\frac{\mathbf{r}_f - \mathbf{r}_i}{\|\mathbf{r}_f - \mathbf{r}_i \|_2} \right), \\
    \mathbf{B} &= \mathbf{r}_f - \overline{R} \cos{\alpha} \left(\frac{\mathbf{r}_f - \mathbf{r}_i}{\|\mathbf{r}_f - \mathbf{r}_i \|_2} \right),
\end{align}
where $\alpha := \cos^{-1} \left(\frac{2 \overline{R}}{\|\mathbf{r}_f - \mathbf{r}_i \|_2} \right)$. 
Here, the expressions for $\mathbf{r}_i$ and $\mathbf{r}_f$ 
are given in \eqref{eq: initial_sphere_center_expression} and \eqref{eq: final_sphere_center_expression}, respectively, and $\overline{R}$ is given in \eqref{eq: Rbar_definition}. The cross-tangent plane between the initial and final spheres is needed for inner-outer, outer-inner, left-right, and right-left pairings. It follows that $\delta^{initial}_{i, o} = -\delta^{final}_{i, o}$ and $\delta^{initial}_{l, r} = -\delta^{final}_{l, r}.$

To define the parameter $\theta$, we first designate a unit vector $\mathbf{x}$ perpendicular to $\mathbf{r}_f - \mathbf{r}_i$, as shown in Fig.~\ref{fig: cross-tangent_planes_parametrization}.\footnotemark\footnotetext{The generation of $\mathbf{x}$ is similar to the procedure described for the cylindrical envelope.} The angle $\theta$ specifies the point of tangency on the initial sphere, with respect to $\mathbf{x}$. In other words, $\theta$ describes the point of intersection of the cross-tangent plane with the circular locus of cross-tangent planes (shown in green) on the initial sphere; since the point of tangency is on the circular locus and $\mathbf{x}$ lies on the circular locus, $\theta,$ which is measured as the angle from $\mathbf{x},$ uniquely describes the point. We then define another unit vector  $\mathbf{y} := \left(\frac{\mathbf{r}_f - \mathbf{r}_i}{\|\mathbf{r}_f - \mathbf{r}_i \|_2} \right) \times \mathbf{x}$, which is orthogonal to both $\mathbf{x}$ and the axis ($\mathbf{k}$, defined in \eqref{eq: k_definition}) between the spheres. Using $\mathbf{x}$ and $\mathbf{y}$, we now express the entry and exit points $\mathbf{X}_{ic}$ and $\mathbf{X}_{oc}$, which are the points of tangency (see Fig.~\ref{fig: cross-tangent_planes_parametrization}).
\begin{align*}
    \mathbf{X}_{ic} (\theta) &= \mathbf{A} + \overline{R} \sin{\alpha} \cos{\theta} \mathbf{x} + \overline{R} \sin{\alpha} \sin{\theta} \mathbf{y}, \\
    \mathbf{X}_{oc} (\theta) &= \mathbf{B} + \overline{R} \sin{\alpha} \cos{(\theta + \pi)} \mathbf{x} + \overline{R} \sin{\alpha} \sin{(\theta + \pi)}  \mathbf{y}.
\end{align*}

Next, we parameterize the tangent vectors at the entry and exit points using angles $\phi_{ic}$ and $\phi_{oc}$, defined relative to the axis $\mathbf{t}$, where $\mathbf{t}$ is a unit vector pointing from $\mathbf{X}_{ic} (\theta)$ to $\mathbf{X}_{oc} (\theta)$ (see Fig.~\ref{fig: cross-tangent_planes_parametrization}). The tangent vectors $\mathbf{T}_{ic}$ and $\mathbf{T}_{oc}$ at the entry and exit points are derived as below:
\begin{align*}
    \mathbf{T}_{ic} (\phi_{ic}) &= \cos{\phi_{ic}} \mathbf{t} (\theta) + \sin{\phi_{ic}} \left(\frac{\mathbf{X}_{ic} (\theta) - \mathbf{r}_{i}}{\overline{R}} \times \mathbf{t} (\theta) \right), \\
    \mathbf{T}_{oc} (\phi_{oc}) &= \cos{\phi_{oc}} \mathbf{t} (\theta) + \sin{\phi_{oc}} \left(\frac{\mathbf{X}_{ic} (\theta) - \mathbf{r}_{i}}{\overline{R}} \times \mathbf{t} (\theta) \right).
\end{align*}
Given the location and tangent vectors at the exit point from the initial sphere and the entry point of the final sphere, we can construct the optimal path over each sphere using the approach in Section~\ref{subsect: initial_final_spheres_path_generation}. We can also construct the path on the cross-tangent plane using the 2D Dubins result \cite{Dubins, dubins_classification}, illustrated in Fig.~\ref{fig: configurations_cross_tangent_plane}. In this figure, the minimum turning radius $R_{plane}$ is dictated by the type of spheres considered at the initial and final configurations. If we are considering inner-outer or outer-inner connections, $R_{plane} = R_{yaw}$ since the vehicle moves in the $\mathbf{T}-\mathbf{Y}$ plane (as observed from Fig.~\ref{fig: cross-tangent_planes_parametrization}). Alternatively, $R_{plane} = R_{pitch}$ if left-right or right-left sphere connections are considered.

\begin{figure}[htb!]
    \centering
    \includegraphics[width=0.8\linewidth]{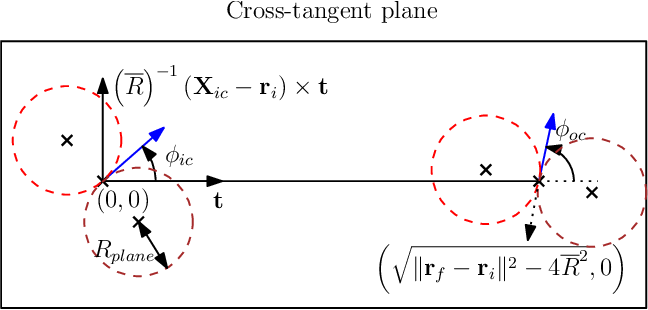}
    \caption{Configurations on cross-tangent plane}
    \label{fig: configurations_cross_tangent_plane}
\end{figure}

After the path on the plane is constructed, the vehicle's coordinates ($u$ and $v$) and the heading angle ($\psi$) are defined. We can reconstruct the configuration of the vehicle along the path in 3D. First, we compute the location and the tangent vector along the plane as
\begin{align*}
    \mathbf{X} (s) &= \mathbf{X}_{ic} + u (s) \mathbf{t} (\theta) + v (s) \left(\frac{\mathbf{X}_{ic} (\theta) - \mathbf{r}_{i}}{\overline{R}} \times \mathbf{t} (\theta) \right), \\
    \mathbf{T} (s) &= \cos{\left(\psi (s) \right)} \mathbf{t} (\theta) \\
    & \quad\, + \sin{\left(\psi (s) \right)} \left(\frac{\mathbf{X}_{ic} (\theta) - \mathbf{r}_{i}}{\overline{R}} \times \mathbf{t} (\theta) \right).
\end{align*}
The vectors $\mathbf{U}$ and $\mathbf{Y}$ are computed depending on $\delta_{i, o}^{initial} \neq 0$ or $\delta_{l, r}^{initial} \neq 0,$ as shown below:
\begin{align*}
    \mathbf{U} (s) = -\delta_{i, o}^{initial} \frac{\mathbf{X}_{ic} (\theta) - \mathbf{r}_{i}}{\overline{R}}, \quad \delta_{i, o}^{initial} \neq 0, \\
    \mathbf{Y} (s) = - \delta_{l, r}^{initial} \frac{\mathbf{X}_{ic} (\theta) - \mathbf{r}_{i}}{\overline{R}}, \quad \delta_{l, r}^{initial} \neq 0.
\end{align*}
Further, we can compute these vectors as $\mathbf{Y} = \mathbf{U} \times \mathbf{T}$ when $\delta_{i, o}^{initial} \neq 0,$ and $\mathbf{U} = \mathbf{T} \times \mathbf{Y}$ when $\delta_{l, r}^{initial} \neq 0$. Thus, the vehicle's configuration in 3D is completely described on the initial sphere, cross-tangent plane, and final sphere.

Note that the path construction for this class is a function of the three parameters: $\theta,$ $\phi_{ic},$ and $\phi_{oc}$. However, motion planning on each of the surfaces depends only on two parameters. We optimize on these parameters by discretizing $\theta \in [0, 2 \pi),$ and $\phi_{ic}$ and $\phi_{oc}$ in $\left[-\frac{\pi}{2}, \frac{\pi}{2} \right]$ (refer to Fig.~\ref{fig: cross-tangent_planes_parametrization}). The number of discretizations of $\theta$ is dictated by $\theta_{disc},$ whereas $\phi_{disc}$ represents the number of discretizations for $\phi_{ic}$ and $\phi_{oc}$.
We choose the parameter set that yields the shortest path length for each pairing of the initial and final spheres. 

\section{Constructing Feasible Solution using Intermediary Sphere} \label{sect: sphere_envelope}

In this section, we present the construction of the third class of paths.
When the initial and the final positions are sufficiently close, we construct a path that goes through an intermediary spherical surface. We consider the four possible combinations in this regard, as outlined in Section~\ref{sect: feasible_solution_construction}. This class of paths exists only when the Euclidean distance between the initial and final position satisfies $\|\mathbf{r}_{f} - \mathbf{r}_{i} \|_2 \leq 4 \overline{R}$, as illustrated in Fig.~\ref{subfig: int_sphere_envelope}.

We parameterize the center of the intermediary sphere using a parameter $\alpha$. The locus of the center of the intermediary sphere is a circle, as shown in Fig.~\ref{fig: computation_xic_intermediary_sphere}. The value of $\alpha$ is related to the radius of this circular locus, and can be derived to be
\begin{align*}
	\alpha = \cos^{-1} \left(\frac{\|\mathbf{r}_{f} - \mathbf{r}_{i} \|_2}{4 \overline{R}} \right),
\end{align*}
and the radius of the circular locus is $2 \overline{R} \sin{\alpha}.$

\begin{figure}[htb!]
    \centering
    \includegraphics[width=0.8\linewidth]{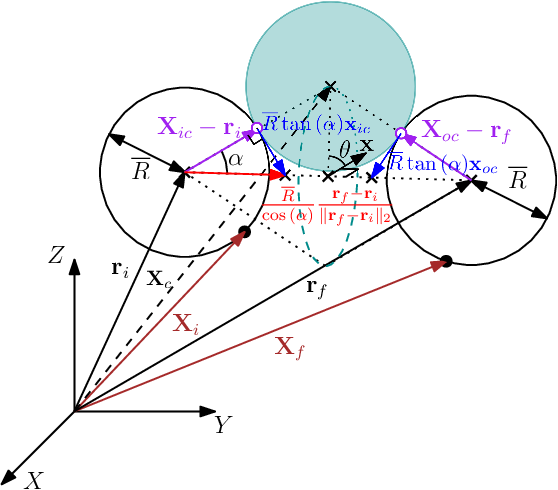}
    \caption{Parameterization of the locus of intermediary spheres and the configurations at entry and exit from the intermediary sphere. (Note that, instead of the vectors $\mathbf{x}_{ic}$ and $\mathbf{x}_{oc}$, we show the same vectors scaled by $\overline{R} \tan(\alpha)$.)}
    \label{fig: computation_xic_intermediary_sphere}
\end{figure}

To parameterize the center of the intermediary sphere, we generate a unit vector $\mathbf{x}$ perpendicular to $\mathbf{r}_f - \mathbf{r}_i$ (similar to the one in Section~\ref{sect: cylindrical_envelope}). We define $\theta \in [0, 2 \pi)$ as the angle made by the center of the intermediary sphere on the circular locus with respect to $\mathbf{x},$ similar to the definition for the cross-tangent plane case. This parameter specifies the location of the intermediary sphere. Hence, we can define the center of the intermediary sphere as (refer to Fig.~\ref{fig: computation_xic_intermediary_sphere})
\begin{align*}
	\mathbf{X}_{c} (\theta) &= \mathbf{r}_{i} + \frac{1}{2} (\mathbf{r}_{f} - \mathbf{r}_{i}) \\
	& \quad + 2 \overline{R} \sin{\alpha} \left(\cos{\theta} \mathbf{x} + \sin{\theta} \left(\frac{\mathbf{r}_f - \mathbf{r}_i}{\|\mathbf{r}_f - \mathbf{r}_i \|_2} \times \mathbf{x} \right) \right).
\end{align*}
Given $\mathbf{X}_{c}$, entry point ($\mathbf{X}_{ic}$) and exit point ($\mathbf{X}_{oc}$) for the intermediary sphere are derived as shown below,
\begin{align*}
    \mathbf{X}_{ic} (\theta) &= \frac{1}{2} \left(\mathbf{r}_{i} + \mathbf{X}_c (\theta) \right), \quad
    \mathbf{X}_{oc} (\theta) = \frac{1}{2} \left(\mathbf{r}_{f} +  \mathbf{X}_c (\theta) \right).
\end{align*}

Finally, to parameterize the tangent vectors at $\mathbf{X}_{ic}$ and $\mathbf{X}_{oc}$, we define the unit vectors $\mathbf{x}_{ic}$ and $\mathbf{x}_{fc}$. These vectors are perpendicular to $\mathbf{X}_{ic} - \mathbf{r}_{i}$ and $\mathbf{X}_{oc} - \mathbf{r}_{f},$ as illustrated in Fig.~\ref{fig: computation_xic_intermediary_sphere}, and are derived as follows:
\begin{align*}
	\mathbf{x}_{ic} &= \frac{\frac{\overline{R}}{\cos{\alpha}} \frac{\mathbf{r}_f - \mathbf{r}_i}{\|\mathbf{r}_f - \mathbf{r}_i \|_2} - \left(\mathbf{X}_{ic} - \mathbf{r}_{i} \right)}{\overline{R} \tan{\alpha}}, \\
	\mathbf{x}_{oc} &= \frac{-\frac{\overline{R}}{\cos{\alpha}} \frac{\mathbf{r}_f - \mathbf{r}_i}{\|\mathbf{r}_f - \mathbf{r}_i \|_2} - \left(\mathbf{X}_{oc} - \mathbf{r}_{f} \right)}{\overline{R} \tan{\alpha}}.
\end{align*}
We parameterize the tangent vectors at $\mathbf{X}_{ic}$ and $\mathbf{X}_{oc}$ denoted by $\mathbf{T}_{ic}$ and $\mathbf{T}_{oc}$ respectively, using the parameters $\phi_{ic}$ and $\phi_{oc},$ as shown below:
\begin{align*}
    \mathbf{T}_{ic} (\phi_{ic}) &= \cos{\left(\phi_{ic} \right)} \mathbf{x}_{ic} + \sin{\left(\phi_{ic} \right)} \frac{\left(\left(\mathbf{X}_{ic} - \mathbf{r}_{i} \right) \times \mathbf{x}_{ic} \right)}{\overline{R}}, \\
    \mathbf{T}_{oc} (\phi_{oc}) &= \cos{\left(\phi_{oc} \right)} \mathbf{x}_{oc} + \sin{\left(\phi_{oc} \right)} \frac{\left(\left(\mathbf{X}_{oc} - \mathbf{r}_{f} \right) \times \mathbf{x}_{oc} \right)}{\overline{R}}.
\end{align*}

We optimize the path length over the parameters $\theta \in [0, 2 \pi),$ $\phi_{ic} \in [0, 2 \pi),$ and $\phi_{oc} \in [0, 2 \pi)$. For a given set of parameters, we can compute $\mathbf{X}_{ic},$ $\mathbf{X}_{oc},$ $\mathbf{T}_{ic},$ and $\mathbf{T}_{oc}$. Similar to the methodology in Section~\ref{subsect: initial_final_spheres_path_generation}, we generate the optimal path on the initial sphere, intermediary sphere, and the final sphere. Note that motion planning on each of the surfaces depends only on two parameters. For instance, in the case of the intermediary sphere, the path length depends only on $\phi_{ic}$ and $\phi_{oc}$. Similar to the path construction using an intermediary plane, $\theta_{disc}$ dictates the number of discretizations of $\theta$ and $\phi_{disc}$ represents the number of discretizations for $\phi_{ic}$ and $\phi_{oc}$. Finally, among all the combinations of the discretized parameter values, we select the path with the minimal total length.

\section{Summary of Path Construction Algorithm}

In this section, we summarize the path construction comprising the three classes of paths, presented in Sections~\ref{sect: cylindrical_envelope}, \ref{sect: cross_tangent_envelope}, and \ref{sect: sphere_envelope} as a pseudocode in Algorithm~\ref{alg: algorithm}.

The initial configuration and final configuration, each of which is defined by the four vectors $\mathbf{X},$ $\mathbf{T},$ $\mathbf{Y},$ and $\mathbf{U},$ are the inputs to the algorithm. We compactly represent them by the homogeneous transformation matrices $\mathbf{H}_0$ and $\mathbf{H}_f,$ respectively.\footnotemark\footnotetext{The homogeneous transformation matrix $\mathbf{H}$ in terms of $\mathbf{X},$ $\mathbf{T},$ $\mathbf{Y},$ and $\mathbf{U}$ is given by $\mathbf{H} = \begin{bmatrix}
    \mathbf{T} & \mathbf{Y} & \mathbf{U} & \mathbf{X} \\
    0 & 0 & 0 & 1
\end{bmatrix}$.}\, The other inputs to the algorithm include the pitch rate (captured by $R_{pitch}$), yaw rate (captured by $R_{yaw}$), and the discretization parameters, ($\theta_{disc}$ and $\phi_{disc}$). Recall that $\theta_{disc}$ and $\phi_{disc}$ are the discretization parameters corresponding to the position and the heading, respectively (which were discussed in Sections~\ref{sect: cylindrical_envelope}, \ref{sect: cross_tangent_envelope}, and \ref{sect: sphere_envelope}).

In line 1 of the algorithm, the minimum cylinder path is computed, which constructs the shortest path through an intermediary cylindrical envelope (described in Section~\ref{sect: cylindrical_envelope}). The function returns the length of the shortest path and the configuration of the vehicle along the path. Similarly, paths through a cross-tangent plane (described in Section~\ref{sect: cross_tangent_envelope}) and through an intermediary sphere (described in Section~\ref{sect: sphere_envelope}), are constructed in lines $2$ and $3$, respectively. Finally, the shortest of all the three classes of paths is determined in line $4$, and the configuration of the vehicle along the shortest path is returned by the algorithm.

\renewcommand{\algorithmicrequire}{\textbf{Input:}}
\newcommand{\NoNumberComment}[1]{\item[] \textit{/* #1 */}}
\begin{algorithm}[htb!]
\caption{Path construction algorithm for 3D Dubins} \label{alg: algorithm}
\begin{algorithmic}[1]
\REQUIRE $\mathbf{H}_{0},$ $\mathbf{H}_{f},$ $R_{pitch},$ $R_{yaw},$ $\theta_{disc},$ $\phi_{disc}$ 
\NoNumberComment{Computing the shortest path through cylindrical envelope}
\STATE $l_{SCS}, \mathbf{H}_{SCS}$ $\gets$ \texttt{MinimumCylinderPath}($\mathbf{H}_{0},$ $\mathbf{H}_{f},$ $R_{pitch},$ $R_{yaw},$ $\theta_{disc},$ $\phi_{disc}$)
\NoNumberComment{Computing the shortest path through cross-tangent plane}
\STATE $l_{SPS}$, $\mathbf{H}_{SPS}$ $\gets$ \texttt{MinimumPlanePath}($\mathbf{H}_{0},$ $\mathbf{H}_{f},$ $R_{pitch},$ $R_{yaw},$ $\theta_{disc},$ $\phi_{disc}$)
\NoNumberComment{Computing the shortest path through intermediary spherical envelope}
\STATE $l_{SSS}$, $\mathbf{H}_{SSS}$ $\gets$ \texttt{MinimumSpherePath}($\mathbf{H}_{0},$ $\mathbf{H}_{f},$ $R_{pitch},$ $R_{yaw},$ $\theta_{disc},$ $\phi_{disc}$)
\NoNumberComment{Computing the shortest overall path}
\STATE $l^*,$ $\mathbf{H}^*\gets$ \texttt{MinimumLength} ($\mathbf{H}_{SCS}$,  $\mathbf{H}_{SPS}$,  $\mathbf{H}_{SSS})$
\RETURN $l^*$, $\mathbf{H}^*$
\end{algorithmic}
\end{algorithm}

\begin{remark}
    If an existence condition is violated for a candidate path (such as an inner sphere -- plane -- outer sphere connection), its path length is returned as \texttt{NaN}. The functions \texttt{MinimumPlanePath} and \texttt{MinimumSpherePath} select the shortest path among those with finite length for each path type. Notably, $SCS$ paths always exist. If no valid path exists for a particular connection type (e.g., $SPS$), the corresponding function (such as \texttt{MinimumPlanePath}) returns \texttt{NaN}. Since the \texttt{MinimumLength} function considers only paths with finite length, any non-existent paths are automatically ignored.
\end{remark}

\section{Results}

In this section, we present computational results to study the performance of Algorithm~\ref{alg: algorithm} and the effect of the vehicle's configuration and the motion constraints on the path. To this end, we consider the scenarios provided in \cite{minimal_3D_Dubins_path_bounded_curvature_pitch}, where five ``Long" and five ``Short" instances were considered depending on the distance between the initial and final configurations. The minimum turning radius was chosen to be $40$ m in \cite{minimal_3D_Dubins_path_bounded_curvature_pitch}, and correspondingly, we consider $R_{pitch} = 40$ m. In \cite{minimal_3D_Dubins_path_bounded_curvature_pitch}, the orientation is defined by only the pitch and heading angles. To describe the complete orientation, we additionally specify the roll angle at the initial and final positions to be one of the values from the set $\{-15^\circ, 0^\circ, 15^\circ\}$\footnotemark. To study the effect of the motion constraints, we run the experiments for different values of $R_{yaw} \in \{30 \,\text{m}, 40 \,\text{m}, 50 \,\text{m}\}$.  
\footnotetext{It should be noted that the orientation of the vehicle can be uniquely prescribed by the vectors $\mathbf{T},$ $\mathbf{Y},$ and $\mathbf{U},$ or by specifying the yaw (rotation about $z$), pitch (rotation about $y$), and roll (rotation about $x$) angles. Fixing a $ZYX$ rotation sequence, one can uniquely compute the expressions for $\mathbf{T},$ $\mathbf{Y},$ and $\mathbf{U}$ for given angles. It should be noted that positive pitch angle is considered to be about $-y$ axis since geometrically, when the vehicle has its nose pointing up, the pitch angle is desired to be considered to be positive.}

Furthermore, we consider two sets of scenarios referred to as  ``Additional 1" and ``Additional 2" described shortly, based on Figs.~\ref{subfig: yaw_rate_violation} and \ref{subfig: pitch_rate_violation}, respectively. In the first set of scenarios, the vehicle needs to perform a turn maneuver with a marginal altitude change. The initial location is $(120, 40, 20)$ with initial heading and pitch angles of $90^\circ$ and $-5^\circ,$ and the final location is $(300, 40, 15)$ with heading and pitch angles of $-90^\circ$ and $-5^\circ$. In ``Additional 2", the vehicle needs to perform an ascent maneuver from an initial location of $(120, 40, 20)$ with initial heading and pitch angles of $90^\circ$ and $-15^\circ$ to a final location of $(130, 120, 41)$ with heading and pitch angles of $85^\circ$ and $20^\circ$.\footnotemark
\footnotetext{All coordinates specified in this paragraph are in meters.}

In addition to considering the algorithm from~\cite{minimal_3D_Dubins_path_bounded_curvature_pitch}, we also use the path construction methodologies described in~\cite{analytic_solution_3D} and~\cite{mathworks_uavdubinsconnection}; the latter implementation utilizes the model proposed in~\cite{Dubins_airplane_fixed_wing_UAVs}. The implementation for the former algorithm is available on the GitHub page of the authors (link in~\cite{analytic_solution_3D}). For the algorithm from~\cite{analytic_solution_3D}, since only a single turning radius parameter is used, we set $R = R_{pitch} = 40$\,m.\footnotemark\footnotetext{In this implementation, the turning radius is assumed to be $1$, and the initial heading vector is aligned with the $z$-axis. Thus, for all of our instances, we scale down the problem and use coordinate transforms to align the initial location with the origin and the initial heading vector with the $+z$-axis. With this scaled instance, the code provided by the authors can be used to construct the $CSC$ paths. In our implementation, we supply the code for scaling down the problem, and for recovering the $CSC$ path for the original problem instance using the authors' parameters.}
For the implementation in~\cite{mathworks_uavdubinsconnection}, we select the maximum roll rate so that the turning radius matches $40$\,m (which is our $R_{pitch}$), and retain the default bounds for the maximum flight path angle of $\pm 0.5$ radians. We note that the model considered in~\cite{mathworks_uavdubinsconnection}, which implements~\cite{Dubins_airplane_fixed_wing_UAVs}, utilizes the initial and final heading angles, rather than the full heading vector.

For Algorithm~\ref{alg: algorithm}, to optimize the path length with respect to the parameters, we consider $15$ discretizations for all parameters that describe the positions and tangent vectors for entry and exit between the intermediary surfaces. We implemented the algorithms in Python~$3.8$ on a computer with AMD Ryzen $9$ $5900$HS CPU running at $3.30$ GHz with $16$ GB RAM. For all the classes of the paths constructed, we parallelized the functions that compute the sub-paths on each individual surface. The computational results of the algorithm are summarized in Table~\ref{tab: results}. \footnotemark
\footnotetext{The first instance of running each function takes a higher time than the reported times in Table~\ref{tab: results}, due to the overheads associated with spawning different processes to implement the functions in a parallel manner. To circumvent this issue, we ``warm start" our algorithm using a dummy instance.}

\begin{table*}[htb!]
\centering
\caption{Summary of best feasible path and computation time for different instances with varying initial roll angle, final roll angle, and $R_{yaw}.$ Path length obtained using algorithms from \cite{minimal_3D_Dubins_path_bounded_curvature_pitch}, \cite{analytic_solution_3D}, and \cite{mathworks_uavdubinsconnection} are shown under the instance name. In the table, $S$ denotes a sphere, $C$ denotes a cylinder, and $P$ denotes a plane. For spheres, subscripts $i,$ $o,$ $l,$ and $r$ represent the inner sphere, outer sphere, left sphere, and right sphere, respectively.}
\label{tab: results}
\begin{tabular}{|c|c|c|c|c|c|c@{\hspace{0.005cm}}|c|c|c|c|c|c|}
\cline{1-6} \cline{8-13}
\textbf{Inst.} & \textbf{\begin{tabular}[c]{@{}c@{}}Roll$^*$\\ ($^\circ$,$^\circ$)\end{tabular}} & \textbf{\begin{tabular}[c]{@{}c@{}}$R_{yaw}$\\ (m)\end{tabular}} & \textbf{\begin{tabular}[c]{@{}c@{}}Length\\ (m)\end{tabular}} & \textbf{\begin{tabular}[c]{@{}c@{}}Path\\ type\end{tabular}} & \textbf{\begin{tabular}[c]{@{}c@{}}Time\\ (s)\end{tabular}} & & \textbf{Inst.} & \textbf{\begin{tabular}[c]{@{}c@{}}Roll$^*$\\ ($^\circ$,$^\circ$)\end{tabular}} & \textbf{\begin{tabular}[c]{@{}c@{}}$R_{yaw}$\\ (m)\end{tabular}} & \textbf{\begin{tabular}[c]{@{}c@{}}Length\\ (m)\end{tabular}} & \textbf{\begin{tabular}[c]{@{}c@{}}Path\\ type\end{tabular}} & \textbf{\begin{tabular}[c]{@{}c@{}}Time\\ (s)\end{tabular}} \\ \cline{1-6} \cline{8-13}
\multirow{9}{*}{\begin{tabular}[c]{@{}c@{}}Long 1\\ 446.04 \cite{minimal_3D_Dubins_path_bounded_curvature_pitch}\\ 437.86 \cite{analytic_solution_3D}\\ 445.10 \cite{mathworks_uavdubinsconnection}\end{tabular}} & \multirow{3}{*}{$(-15,0)$} & 30 & 478.06 & $S_r P S_l$ & 9.91 & & \multirow{9}{*}{\begin{tabular}[c]{@{}c@{}}Short 2\\ 668.17 \cite{minimal_3D_Dubins_path_bounded_curvature_pitch}\\ 281.96 \cite{analytic_solution_3D}\\ 506.05 \cite{mathworks_uavdubinsconnection}\end{tabular}} & \multirow{3}{*}{$(-15,0)$} & 30 & 302.02 & \multirow{9}{*}{$S_o P S_i$} & 9.54 \\ \cline{3-6} \cline{10-11} \cline{13-13}
&  & 40 & 510.64 & \multirow{2}{*}{$S_o P S_i$} & 9.60 & &  &  & 40 & 323.47 &  & 9.34 \\ \cline{3-4} \cline{6-6} \cline{10-11} \cline{13-13}
&  & 50 & 533.24 &  & 9.77 & &  &  & 50 & 356.72 &  & 9.50 \\ \cline{2-6} \cline{9-11} \cline{13-13}
& \multirow{3}{*}{$(0,15)$} & 30 & 475.98 & $S_r P S_l$ & 9.87 & &  & \multirow{3}{*}{$(0,15)$} & 30 & 298.45 &  & 9.59 \\ \cline{3-6} \cline{10-11} \cline{13-13}
&  & 40 & 412.17 & $S_l C S_l$ & 10.17 & &  &  & 40 & 313.69 &  & 9.33 \\ \cline{3-6} \cline{10-11} \cline{13-13}
&  & 50 & 545.09 & $S_r P S_l$ & 10.46 & &  &  & 50 & 335.89 &  & 9.51 \\ \cline{2-6} \cline{9-11} \cline{13-13}
& \multirow{3}{*}{$(15,-15)$} & 30 & 447.36 & \multirow{3}{*}{$S_i C S_i$} & 9.68 & &  & \multirow{3}{*}{$(15,-15)$} & 30 & 313.70 &  & 9.44 \\ \cline{3-4} \cline{6-6} \cline{10-11} \cline{13-13}
&  & 40 & 470.08 &  & 9.53 & &  &  & 40 & 336.86 &  & 9.79 \\ \cline{3-4} \cline{6-6} \cline{10-11} \cline{13-13}
&  & 50 & 473.59 &  & 9.87 & &  &  & 50 & 349.95 &  & 10.14 \\ \cline{1-6} \cline{8-13}
\multirow{9}{*}{\begin{tabular}[c]{@{}c@{}}Long 2\\ 638.45 \cite{minimal_3D_Dubins_path_bounded_curvature_pitch}\\ 631.73 \cite{analytic_solution_3D}\\ 637.22 \cite{mathworks_uavdubinsconnection}\end{tabular}} & \multirow{3}{*}{$(-15,0)$} & 30 & 587.86 & \multirow{9}{*}{$S_r C S_r$} & 9.72 & & \multirow{9}{*}{\begin{tabular}[c]{@{}c@{}}Short 3\\ 976.79 \cite{minimal_3D_Dubins_path_bounded_curvature_pitch}\\ 342.52 \cite{analytic_solution_3D}\\ 521.46 \cite{mathworks_uavdubinsconnection}\end{tabular}} & \multirow{3}{*}{$(-15,0)$} & 30 & 364.92 & \multirow{7}{*}{$S_o P S_i$} & 9.83 \\ \cline{3-4} \cline{6-6} \cline{10-11} \cline{13-13}
&  & 40 & 606.64 &  & 9.66 & &  &  & 40 & 386.80 &  & 9.37 \\ \cline{3-4} \cline{6-6} \cline{10-11} \cline{13-13}
&  & 50 & 614.57 &  & 9.90 & &  &  & 50 & 400.55 &  & 9.62 \\ \cline{2-4} \cline{6-6} \cline{9-11} \cline{13-13}
& \multirow{3}{*}{$(0,15)$} & 30 & 590.12 &  & 12.03 & &  & \multirow{3}{*}{$(0,15)$} & 30 & 356.75 &  & 9.49 \\ \cline{3-4} \cline{6-6} \cline{10-11} \cline{13-13}
&  & 40 & 601.87 &  & 10.29 & &  &  & 40 & 368.52 &  & 9.37 \\ \cline{3-4} \cline{6-6} \cline{10-11} \cline{13-13}
&  & 50 & 620.67 &  & 10.87 & &  &  & 50 & 384.82 &  & 10.03 \\ \cline{2-4} \cline{6-6} \cline{9-11} \cline{13-13}
& \multirow{3}{*}{$(15,-15)$} & 30 & 583.47 &  & 10.18 & &  & \multirow{3}{*}{$(15,-15)$} & 30 & 349.90 &  & 10.84 \\ \cline{3-4} \cline{6-6} \cline{10-13}
&  & 40 & 603.90 &  & 10.29 & &  &  & 40 & 416.92 & \multirow{2}{*}{$S_r C S_r$} & 10.29 \\ \cline{3-4} \cline{6-6} \cline{10-11} \cline{13-13}
&  & 50 & 612.32 &  & 10.83 & &  &  & 50 & 408.28 &  & 9.79 \\ \cline{1-6} \cline{8-13}
\multirow{9}{*}{\begin{tabular}[c]{@{}c@{}}Long 3\\ 1068.34 \cite{minimal_3D_Dubins_path_bounded_curvature_pitch}\\ 1059.20 \cite{analytic_solution_3D}\\ 1054.19 \cite{mathworks_uavdubinsconnection}\end{tabular}} & \multirow{3}{*}{$(-15,0)$} & 30 & 1032.72 & \multirow{2}{*}{$S_i C S_i$} & 10.97 & & \multirow{9}{*}{\begin{tabular}[c]{@{}c@{}}Short 4\\ 1169.80 \cite{minimal_3D_Dubins_path_bounded_curvature_pitch}\\ 422.26 \cite{analytic_solution_3D}\\ 625.75 \cite{mathworks_uavdubinsconnection}\end{tabular}} & \multirow{3}{*}{$(-15,0)$} & 30 & 425.17 & \multirow{6}{*}{$S_l C S_l$} & 9.77 \\ \cline{3-4} \cline{6-6} \cline{10-11} \cline{13-13}
&  & 40 & 1141.58 &  & 10.90 & &  &  & 40 & 425.65 &  & 9.53 \\ \cline{3-6} \cline{10-11} \cline{13-13}
&  & 50 & 1161.82 & $S_l C S_l$ & 11.23 & &  &  & 50 & 421.25 &  & 9.75 \\ \cline{2-6} \cline{9-11} \cline{13-13}
& \multirow{3}{*}{$(0,15)$} & 30 & 1026.87 & \multirow{6}{*}{$S_i C S_i$} & 11.30 & &  & \multirow{3}{*}{$(0,15)$} & 30 & 420.89 &  & 9.94 \\ \cline{3-4} \cline{6-6} \cline{10-11} \cline{13-13}
&  & 40 & 1045.90 &  & 11.14 & &  &  & 40 & 422.72 &  & 9.47 \\ \cline{3-4} \cline{6-6} \cline{10-11} \cline{13-13}
&  & 50 & 1048.38 &  & 11.30 & &  &  & 50 & 447.19 &  & 9.68 \\ \cline{2-4} \cline{6-6} \cline{9-13}
& \multirow{3}{*}{$(15,-15)$} & 30 & 1037.06 &  & 11.06 & &  & \multirow{3}{*}{$(15,-15)$} & 30 & 445.42 & \multirow{2}{*}{$S_o P S_i$} & 9.62 \\ \cline{3-4} \cline{6-6} \cline{10-11} \cline{13-13}
&  & 40 & 1040.40 &  & 10.91 & &  &  & 40 & 493.81 &  & 9.55 \\ \cline{3-4} \cline{6-6} \cline{10-13}
&  & 50 & 1058.79 &  & 12.67 & &  &  & 50 & 512.31 & $S_l C S_l$ & 9.84 \\ \cline{1-6} \cline{8-13}
\multirow{9}{*}{\begin{tabular}[c]{@{}c@{}}Long 4\\ 1788.80 \cite{minimal_3D_Dubins_path_bounded_curvature_pitch}\\ 1784.85 \cite{analytic_solution_3D}\\ 1787.15 \cite{mathworks_uavdubinsconnection}\end{tabular}} & \multirow{3}{*}{$(-15,0)$} & 30 & 1744.87 & \multirow{9}{*}{$S_l C S_l$} & 11.52 & & \multirow{9}{*}{\begin{tabular}[c]{@{}c@{}}Short 5\\ 1362.91 \cite{minimal_3D_Dubins_path_bounded_curvature_pitch}\\ 437.46 \cite{analytic_solution_3D}\\ 730.04 \cite{mathworks_uavdubinsconnection}\end{tabular}} & \multirow{3}{*}{$(-15,0)$} & 30 & 444.24 & \multirow{5}{*}{$S_o P S_i$} & 9.62 \\ \cline{3-4} \cline{6-6} \cline{10-11} \cline{13-13}
&  & 40 & 1758.23 &  & 12.05 & &  &  & 40 & 463.42 &  & 9.53 \\ \cline{3-4} \cline{6-6} \cline{10-11} \cline{13-13}
&  & 50 & 1773.09 &  & 11.29 & &  &  & 50 & 477.80 &  & 9.66 \\ \cline{2-4} \cline{6-6} \cline{9-11} \cline{13-13}
& \multirow{3}{*}{$(0,15)$} & 30 & 1747.76 &  & 10.65 & &  & \multirow{3}{*}{$(0,15)$} & 30 & 440.33 &  & 9.60 \\ \cline{3-4} \cline{6-6} \cline{10-11} \cline{13-13}
&  & 40 & 1759.44 &  & 10.60 & &  &  & 40 & 446.09 &  & 9.33 \\ \cline{3-4} \cline{6-6} \cline{10-13}
&  & 50 & 1776.86 &  & 10.65 & &  &  & 50 & 520.75 & $S_l C S_l$ & 9.77 \\ \cline{2-4} \cline{6-6} \cline{9-13}
& \multirow{3}{*}{$(15,-15)$} & 30 & 1744.13 &  & 10.87 & &  & \multirow{3}{*}{$(15,-15)$} & 30 & 453.28 & \multirow{3}{*}{$S_r C S_r$} & 9.66 \\ \cline{3-4} \cline{6-6} \cline{10-11} \cline{13-13}
&  & 40 & 1763.51 &  & 10.86 & &  &  & 40 & 447.54 &  & 9.52 \\ \cline{3-4} \cline{6-6} \cline{10-11} \cline{13-13}
&  & 50 & 1768.64 &  & 10.61 & &  &  & 50 & 467.20 &  & 9.88 \\ \cline{1-6} \cline{8-13}
\multirow{9}{*}{\begin{tabular}[c]{@{}c@{}}Long 5\\ 2214.54 \cite{minimal_3D_Dubins_path_bounded_curvature_pitch}\\ 2213.70 \cite{analytic_solution_3D}\\ 2208.70 \cite{mathworks_uavdubinsconnection}\end{tabular}} & \multirow{3}{*}{$(-15,0)$} & 30 & 2187.58 & \multirow{9}{*}{$S_r C S_r$} & 10.91 & & \multirow{9}{*}{\begin{tabular}[c]{@{}c@{}}Add. 1\\ 225.89 \cite{minimal_3D_Dubins_path_bounded_curvature_pitch}\\ 225.67 \cite{analytic_solution_3D}\\ 225.72 \cite{mathworks_uavdubinsconnection}\end{tabular}} & \multirow{3}{*}{$(-15,0)$} & 30 & 212.63 & \multirow{9}{*}{$S_r C S_r$} & 9.43 \\ \cline{3-4} \cline{6-6} \cline{10-11} \cline{13-13}
&  & 40 & 2201.98 &  & 10.79 & &  &  & 40 & 223.18 &  & 9.83 \\ \cline{3-4} \cline{6-6} \cline{10-11} \cline{13-13}
&  & 50 & 2211.11 &  & 11.00 & &  &  & 50 & 233.88 &  & 10.26 \\ \cline{2-4} \cline{6-6} \cline{9-11} \cline{13-13}
& \multirow{3}{*}{$(0,15)$} & 30 & 2189.22 &  & 11.17 & &  & \multirow{3}{*}{$(0,15)$} & 30 & 213.34 &  & 9.58 \\ \cline{3-4} \cline{6-6} \cline{10-11} \cline{13-13}
&  & 40 & 2201.75 &  & 10.79 & &  &  & 40 & 223.83 &  & 10.00 \\ \cline{3-4} \cline{6-6} \cline{10-11} \cline{13-13}
&  & 50 & 2212.44 &  & 11.14 & &  &  & 50 & 234.02 &  & 10.11 \\ \cline{2-4} \cline{6-6} \cline{9-11} \cline{13-13}
& \multirow{3}{*}{$(15,-15)$} & 30 & 2192.96 &  & 10.86 & &  & \multirow{3}{*}{$(15,-15)$} & 30 & 211.93 &  & 9.48 \\ \cline{3-4} \cline{6-6} \cline{10-11} \cline{13-13}
&  & 40 & 2209.89 &  & 10.59 & &  &  & 40 & 222.00 &  & 9.85 \\ \cline{3-4} \cline{6-6} \cline{10-11} \cline{13-13}
&  & 50 & 2225.56 &  & 11.04 & &  &  & 50 & 232.53 &  & 10.10 \\ \cline{1-6} \cline{8-13}
\multirow{9}{*}{\begin{tabular}[c]{@{}c@{}}Short 1\\ 580.79 \cite{minimal_3D_Dubins_path_bounded_curvature_pitch}\\ 299.34 \cite{analytic_solution_3D}\\ 312.87 \cite{mathworks_uavdubinsconnection}\end{tabular}} & \multirow{3}{*}{$(-15,0)$} & 30 & 289.67 & $S_r C S_r$ & 9.75 & & \multirow{9}{*}{\begin{tabular}[c]{@{}c@{}}Add. 2\\ 84.55 \cite{minimal_3D_Dubins_path_bounded_curvature_pitch}\\ 84.54 \cite{analytic_solution_3D}\\ 83.33 \cite{mathworks_uavdubinsconnection}\end{tabular}} & \multirow{3}{*}{$(-15,0)$} & 30 & 107.59 & \multirow{2}{*}{$S_l S_r S_l$} & 11.86 \\ \cline{3-6} \cline{10-11} \cline{13-13}
&  & 40 & 376.88 & \multirow{2}{*}{$S_o C S_o$} & 9.47 & &  &  & 40 & 91.54 &  & 11.63 \\ \cline{3-4} \cline{6-6} \cline{10-13}
&  & 50 & 394.71 &  & 9.72 & &  &  & 50 & 274.51 & $S_i C S_i$ & 11.95 \\ \cline{2-6} \cline{9-13}
& \multirow{3}{*}{$(0,15)$} & 30 & 355.38 & $S_o P S_i$ & 9.58 & &  & \multirow{3}{*}{$(0,15)$} & 30 & 87.17 & $S_i S_o S_i$ & 12.00 \\ \cline{3-6} \cline{10-13}
&  & 40 & 375.31 & \multirow{2}{*}{$S_r C S_r$} & 9.54 & &  &  & 40 & 250.06 & $S_r S_l S_r$ & 11.52 \\ \cline{3-4} \cline{6-6} \cline{10-13}
&  & 50 & 389.47 &  & 9.54 & &  &  & 50 & 280.35 & $S_i C S_i$ & 12.07 \\ \cline{2-6} \cline{9-13}
& \multirow{3}{*}{$(15,-15)$} & 30 & 352.59 & $S_o C S_o$ & 9.58 & &  & \multirow{3}{*}{$(15,-15)$} & 30 & 95.39 & $S_r S_l S_r$ & 11.82 \\ \cline{3-6} \cline{10-13}
&  & 40 & 380.83 & $S_r C S_r$ & 9.44 & &  &  & 40 & 259.63 & $S_i C S_i$ & 11.61 \\ \cline{3-6} \cline{10-13}
&  & 50 & 360.12 & $S_r S_l S_r$ & 10.18 & &  &  & 50 & 280.38 & $S_i P S_o$ & 12.13 \\ \cline{1-6} \cline{8-13}
\multicolumn{13}{l}{$^*$ -- Initial and final roll angles are specified as an ordered pair.}
\end{tabular}
\end{table*}

In Table~\ref{tab: results}, 
we show the length of the path obtained using the three benchmarking algorithms under the instance name. For the implementation in \cite{minimal_3D_Dubins_path_bounded_curvature_pitch}, we consider the same parameter values used by the authors, which are a minimum turning radius of $40$ m and a bounded pitch angle in $[-15^\circ, 20^\circ]$.

From this table, we can observe that
\begin{itemize}
    \item The path lengths obtained from our model are comparable to those from \cite{minimal_3D_Dubins_path_bounded_curvature_pitch}, \cite{analytic_solution_3D}, and \cite{mathworks_uavdubinsconnection} for all ``Long" maneuvers. In particular, our path length is shorter for a majority of ``Long 2", ``Long 3", ``Long 4", and ``Long 5" instances.
    The generated paths satisfy both pitch and yaw rate constraints while connecting the initial and final configurations, which include the vehicle's roll angle.
    \item For ``Short" maneuvers, the length of the paths generated by the proposed algorithm is two to three times shorter than the path obtained from \cite{minimal_3D_Dubins_path_bounded_curvature_pitch}, while remaining comparable to the paths from \cite{analytic_solution_3D} and \cite{mathworks_uavdubinsconnection}. On a few of the maneuvers, \cite{mathworks_uavdubinsconnection} yields larger path lengths due to constraints on the flight path angle. 
    \item The computation time is around $10$ seconds for most of the instances.
\end{itemize}
These results show the impact of the model and the control inputs on the resulting trajectory. We further expand upon these results to showcase the effect of the motion constraints and the configurations in the following subsections, along with additional examples.

\subsection{Effect of motion constraints}

From Table~\ref{tab: results}, we can observe that increasing $R_{yaw}$ changes the path length and also the best feasible path type. For instance, consider the ``Short 4" instance with initial and final roll angles of $15^\circ$ and $-15^\circ,$ respectively. The path generated for $R_{yaw} = 30$ m, $40$ m, and $50$ m is illustrated in Fig.~\ref{fig: varying_paths_Ryaw}. We can observe that increasing $R_{yaw}$ from $40$ m to $50$ m changes the path obtained from our algorithm from a cross-tangent connection to a path through the left spheres at the initial and final configurations connected by a cylinder. Additionally, while increasing $R_{yaw}$ from $30$ m to $40$ m retains the same path type as the best path, the points of departure and arrival at the initial and final spheres have changed significantly. This is because $R_{yaw}$ particularly affects the turning capability of the vehicle on the cross-tangent plane, as can be observed from Figs.~\ref{subfig: Short 4 ini_roll 15 fin_roll -15 Ryaw 30} and \ref{subfig: Short 4 ini_roll 15 fin_roll -15 Ryaw 40}.

For an instance of the ``Short 4" case where the path length is around $400$ m, the change in path length is around $50$ to $70$ m when $R_{yaw}$ is varied.
The effect of $R_{yaw}$ is more pronounced for ``Additional 2", where the vehicle needs to perform an ascent motion. For this instance, the path length may vary from  $100$ m to as high as $280$ m depending on $R_{yaw}.$ An illustration of the paths for increasing $R_{yaw}$ from $30$ m to $50$ m is shown in Fig.~\ref{fig: varying_paths_Ryaw_additional_2}.

These effects may not be captured by the existing models, including \cite{minimal_3D_Dubins_path_bounded_curvature_pitch}, \cite{analytic_solution_3D}, and \cite{mathworks_uavdubinsconnection}, due to the single control input considered in these models. Illustrations of the paths obtained using the results from \cite{minimal_3D_Dubins_path_bounded_curvature_pitch}, \cite{analytic_solution_3D}, and \cite{mathworks_uavdubinsconnection} for the scenarios ``Short 4" and ``Additional 2" are shown in Figs.~\ref{fig: Short4_Vana} and \ref{fig: Additional_2_comparison}, respectively. 
 We note that the path obtained from~\cite{mathworks_uavdubinsconnection} utilizes a helical segment followed by a left turn, straight line segment, and another left turn to construct the path for ``Short~4'' in Fig.~\ref{fig: Short4_Vana}. Additionally, since the full heading vector is not considered in \cite{mathworks_uavdubinsconnection}, it yields a starkly different path, which enters the pitch sphere and hence violates the pitch rate constraint, unlike the other two algorithms for ``Additional 2".

\begin{figure*}[htb!]
    \centering
    \subfloat[$R_{yaw} = 30$ m (path through intermediary cross-tangent plane)]{\includegraphics[width = 0.32\linewidth]{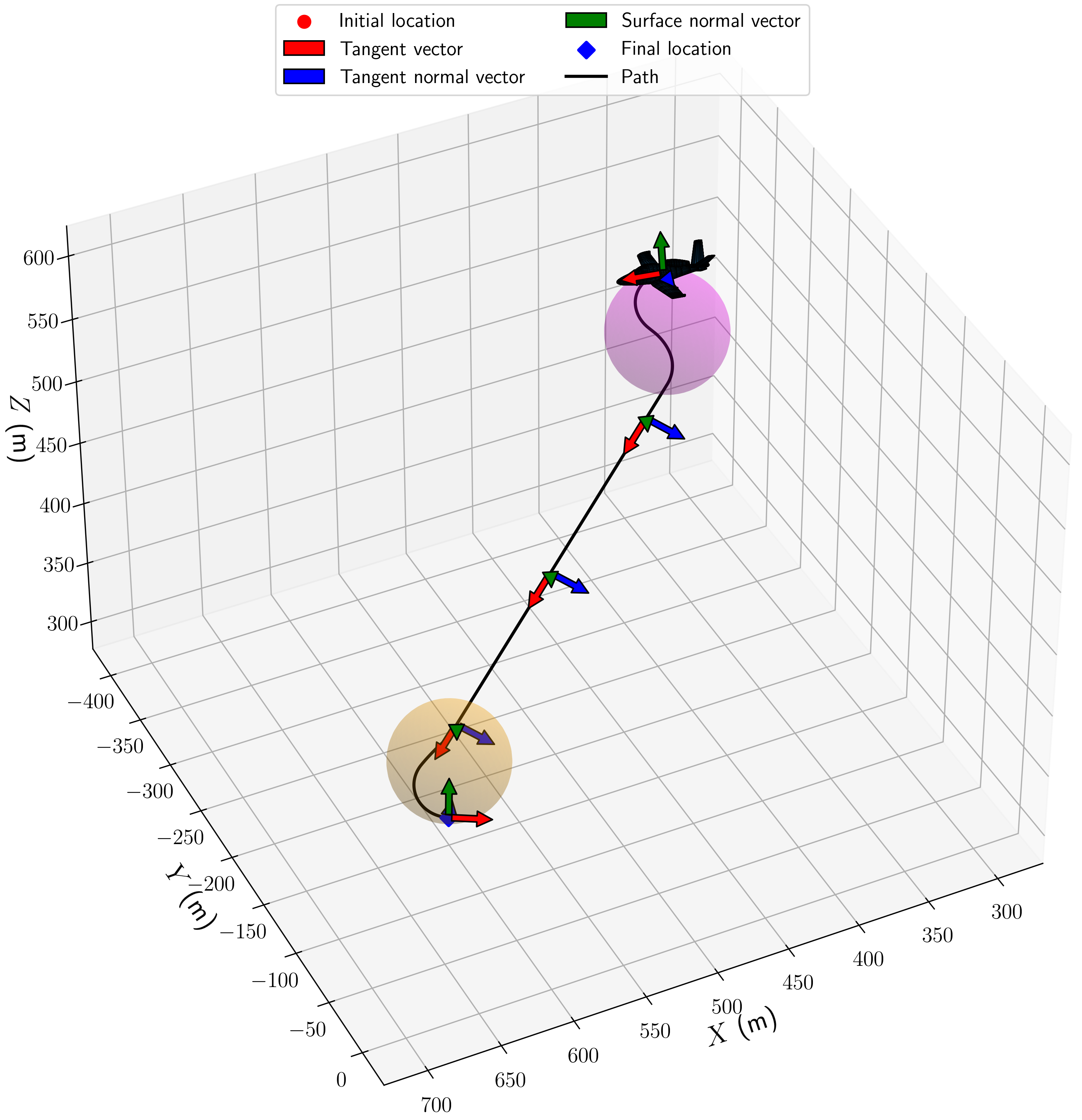}
    \label{subfig: Short 4 ini_roll 15 fin_roll -15 Ryaw 30}} \hfill
    \subfloat[$R_{yaw} = 40$ m (path through intermediary cross-tangent plane)]{\includegraphics[width = 0.32\linewidth]{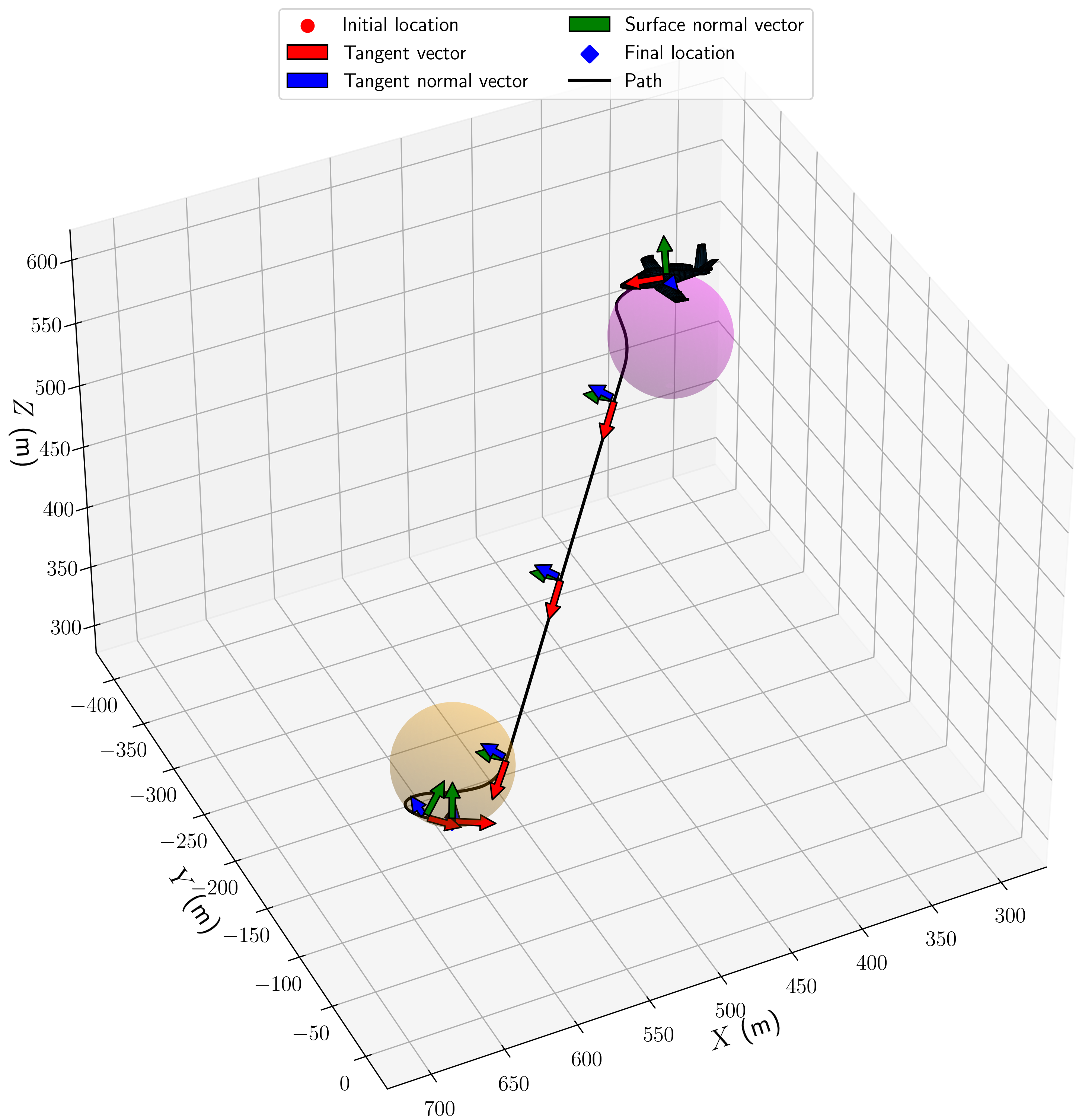}
    \label{subfig: Short 4 ini_roll 15 fin_roll -15 Ryaw 40}} \hfill
    \subfloat[$R_{yaw} = 50$ m (path through intermediary cylindrical envelope)]{\includegraphics[width = 0.32\linewidth]{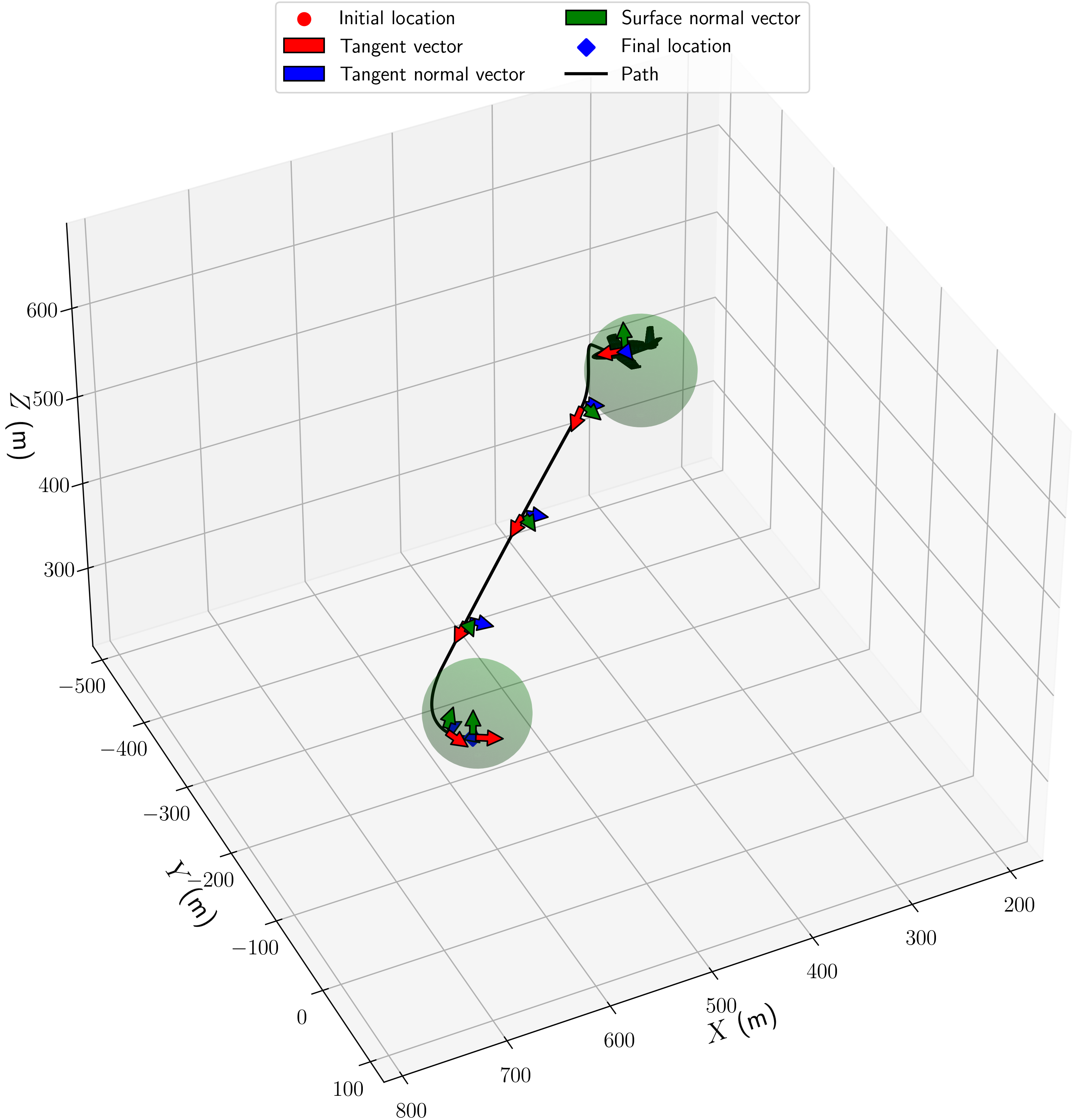}
    \label{subfig: Short 4 ini_roll 15 fin_roll -15 Ryaw 50}}
    \caption{Depiction of varying paths with $R_{yaw}$ for ``Short 4" with initial roll angle of $15^\circ$ and final roll angle of $-15^\circ$ and instantaneous configuration of the vehicle. Animations of a vehicle moving along these paths are available on our GitHub page.}
    \label{fig: varying_paths_Ryaw}
\end{figure*}

\begin{figure*}[htb!]
    \centering
    \subfloat[$R_{yaw} = 30$ m (path through intermediary sphere envelope)]{\includegraphics[width = 0.33\linewidth]{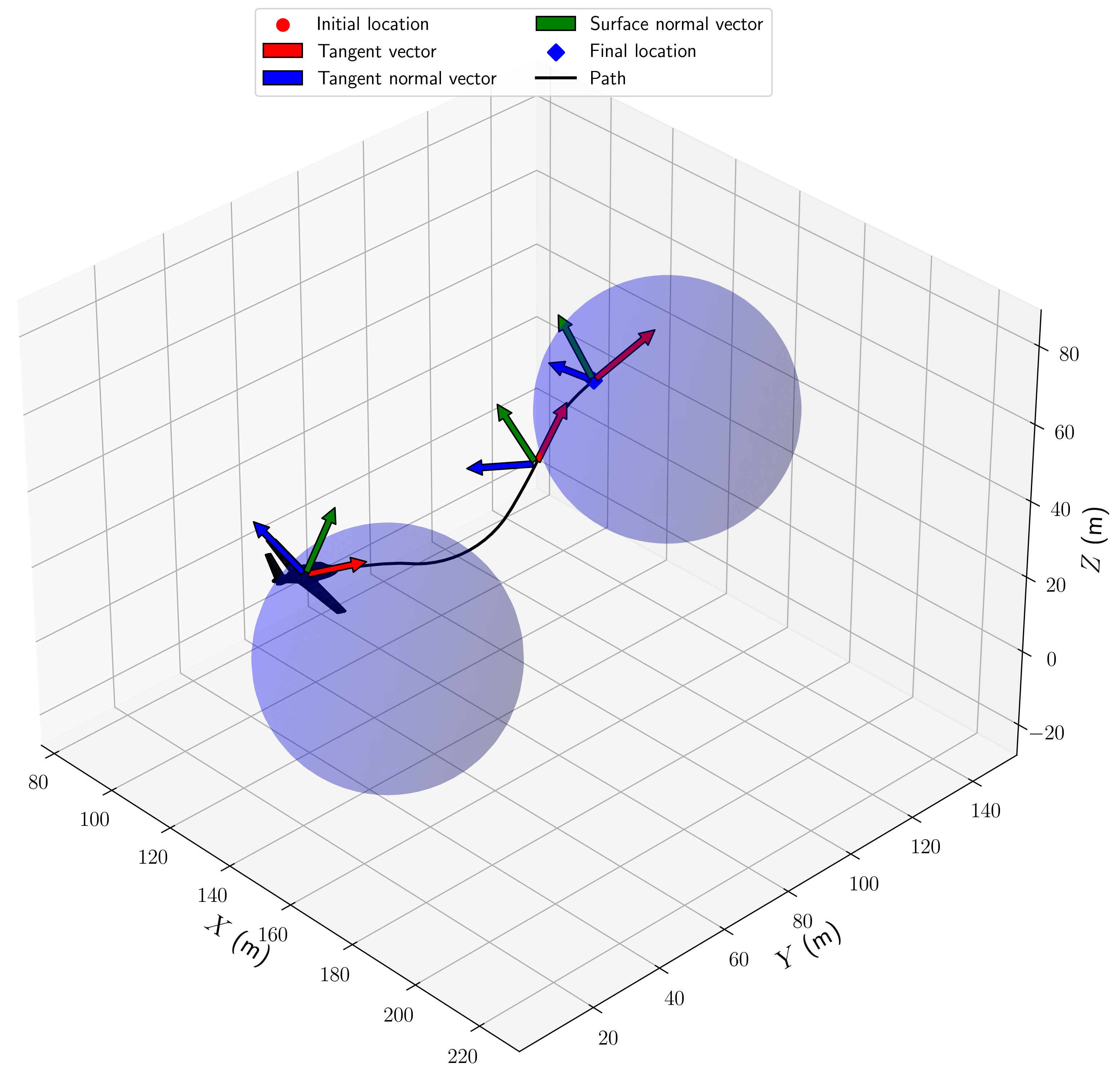}
    \label{subfig: Additional 2 ini_roll 15 fin_roll -15 Ryaw 30}} \hfill
    \subfloat[$R_{yaw} = 40$ m (path through intermediary cylindrical envelope)]{\includegraphics[width = 0.3\linewidth]{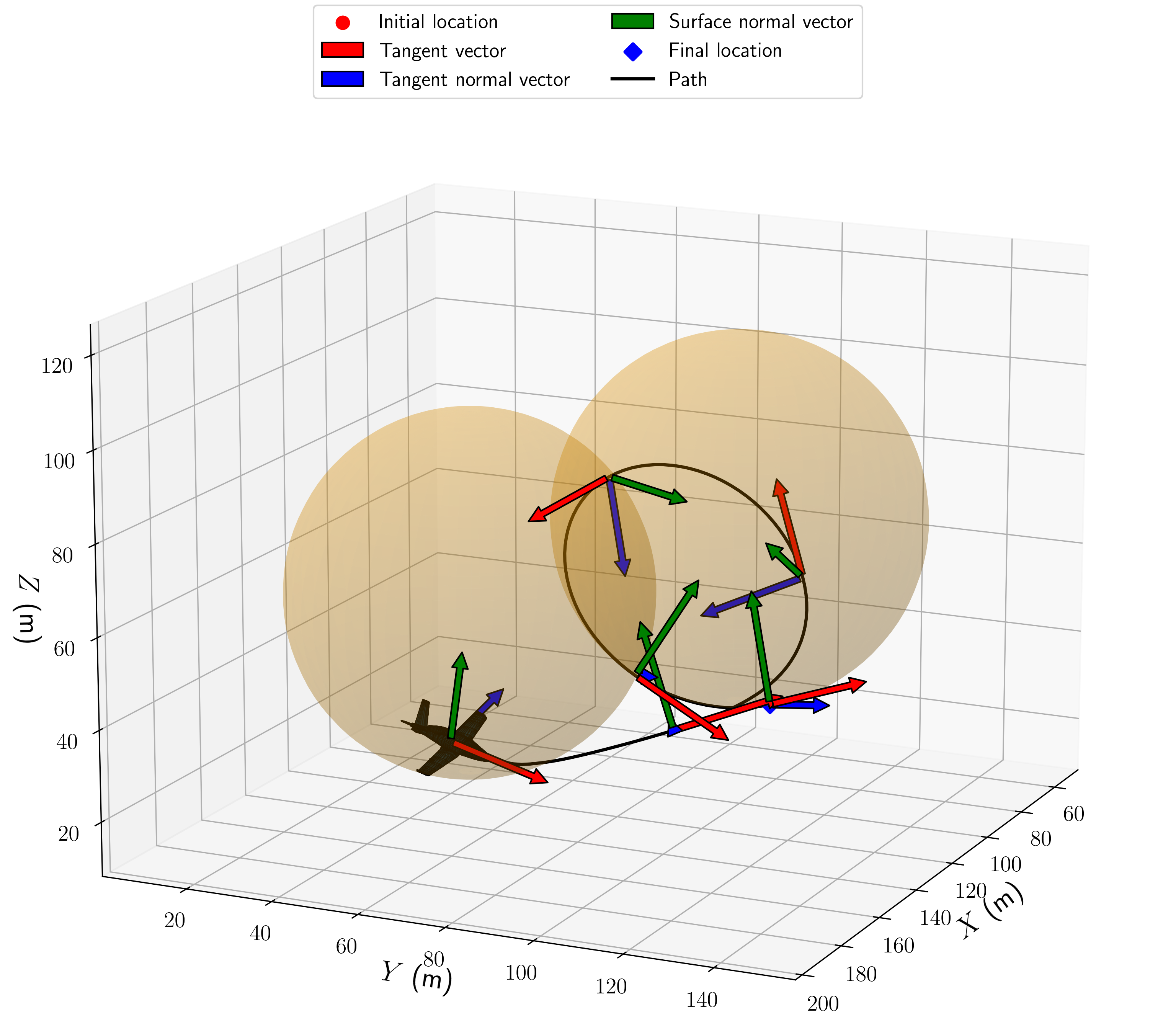}
    \label{subfig: Additional 2 ini_roll 15 fin_roll -15 Ryaw 40}} \hfill
    \subfloat[$R_{yaw} = 50$ m (path through intermediary cross-tangent plane)]{\includegraphics[width = 0.3\linewidth]{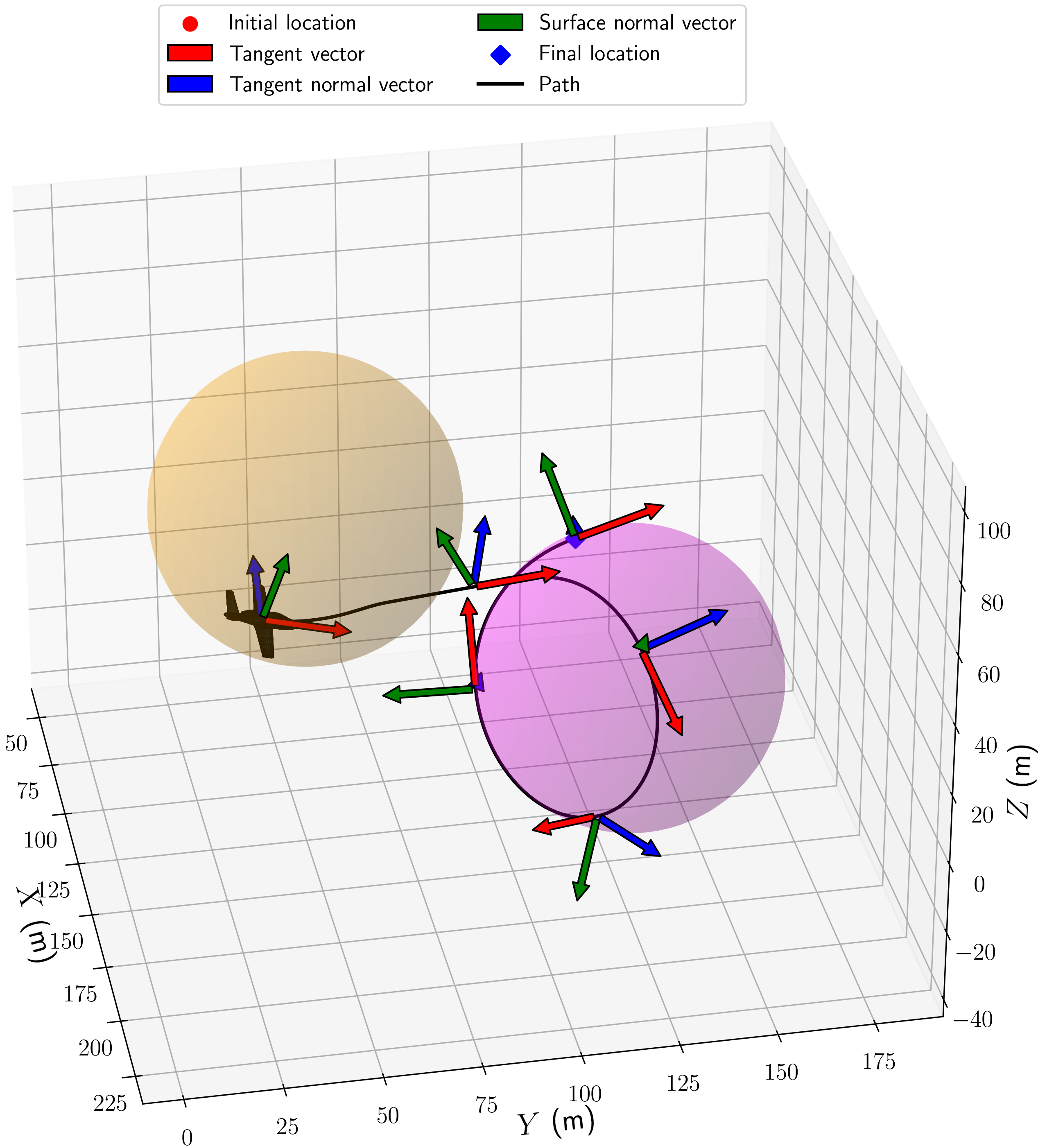}
    \label{subfig: Additional 2 ini_roll 15 fin_roll -15 Ryaw 50}}
    \caption{Depiction of varying paths with $R_{yaw}$ for ``Additional 2" with initial roll angle of $15^\circ$ and final roll angle of $-15^\circ$ and instantaneous configuration of the vehicle. Animations of a vehicle moving along these paths are available on our GitHub page.}
    \label{fig: varying_paths_Ryaw_additional_2}
\end{figure*}

\begin{figure}[htb!]
    \centering
    \includegraphics[width = \linewidth]{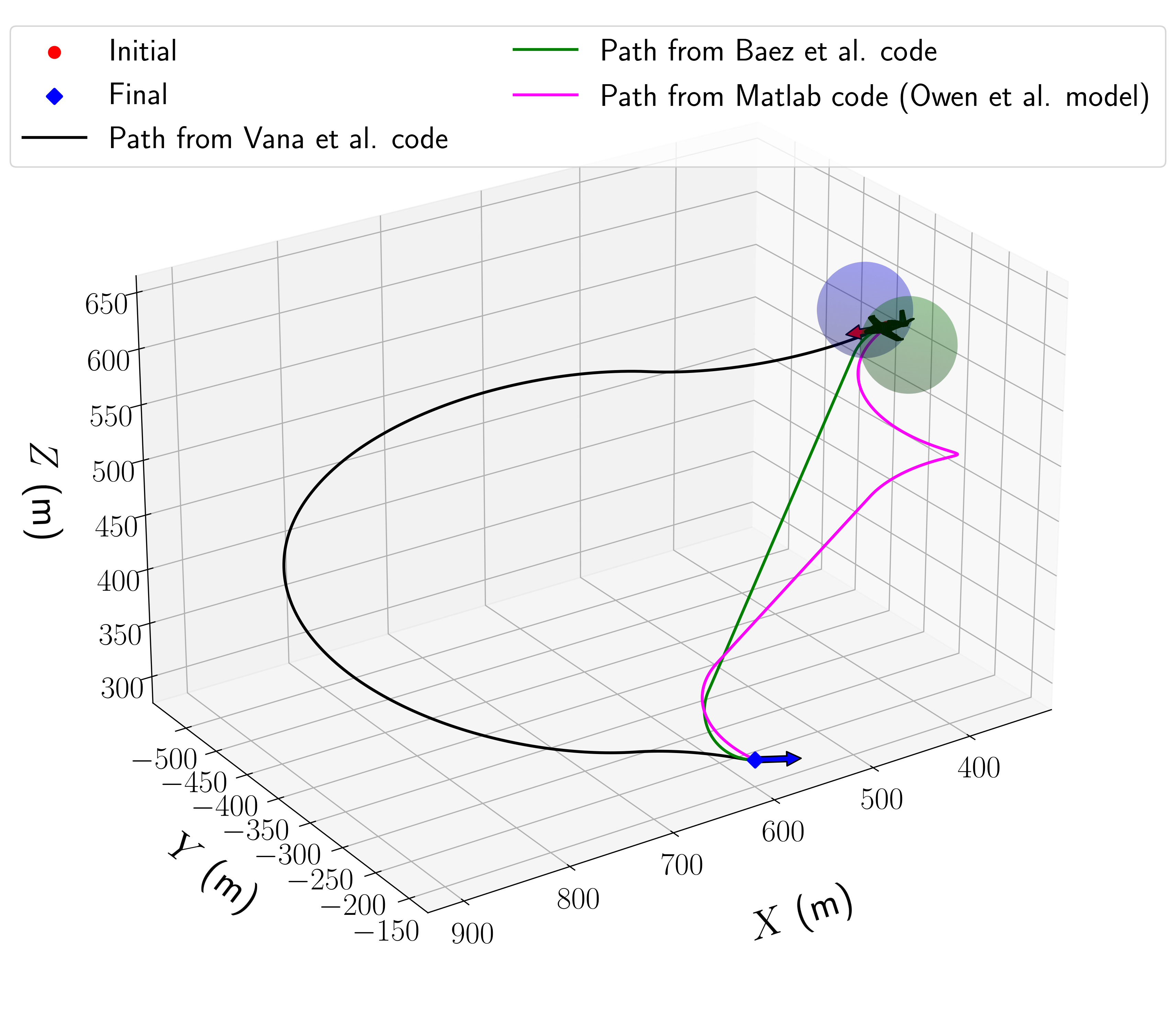}
    \caption{Solutions from \cite{minimal_3D_Dubins_path_bounded_curvature_pitch}, \cite{analytic_solution_3D}, and \cite{mathworks_uavdubinsconnection} for ``Short 4"}
    \label{fig: Short4_Vana}
\end{figure}

\begin{figure}[htb!]
    \centering
    \includegraphics[width = 0.8\linewidth]{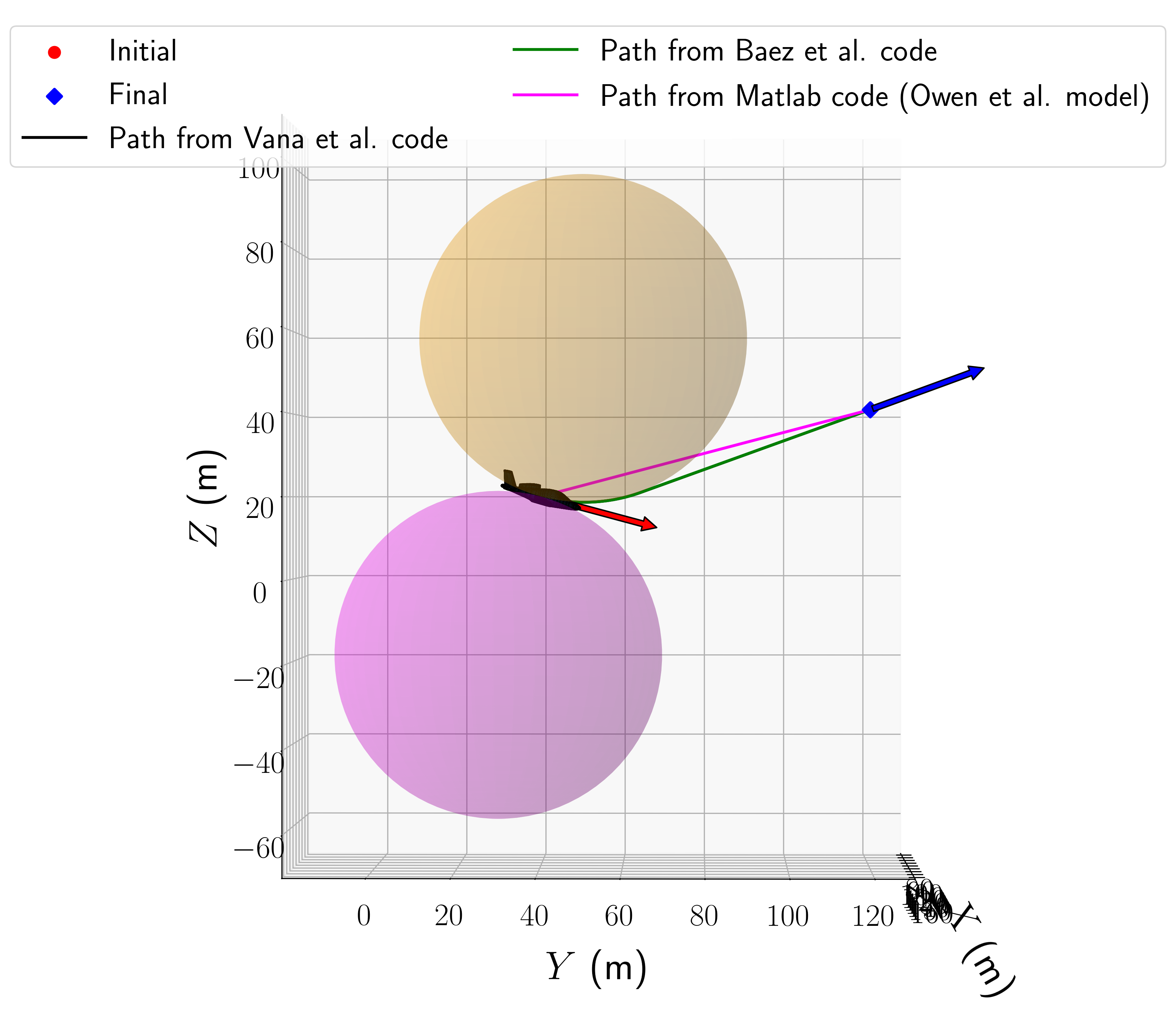}
    \caption{Solutions from \cite{minimal_3D_Dubins_path_bounded_curvature_pitch}, \cite{analytic_solution_3D}, and \cite{mathworks_uavdubinsconnection} for ``Additional 2"}
    \label{fig: Additional_2_comparison}
\end{figure}

\subsection{Effect of considering complete configuration}

From Table~\ref{tab: results}, we can observe that the roll angle at the initial and final locations impacts the path length and the path type as well. For instance, consider ``Long 1" with $R_{yaw} = 40$ m. We illustrate the change in the path type with changing roll angles across the three subfigures in Fig.~\ref{fig: varying_paths_config}. We observe that for initial and final roll angles of $-15^\circ$ and $0^\circ,$ the vehicle travels on the outer sphere, followed by traveling on the cross-tangent plane and an inner sphere. However, changing the initial and final roll angles changes the minimum cost path from our algorithm to travel along a cylindrical envelope connecting the left spheres for initial and final roll angles of $0^\circ$ and $15^\circ.$ For initial and final roll angles of $15^\circ$ and $-15^\circ,$ the best feasible path is a cylindrical envelope connecting inner spheres.

In contrast, the paths obtained from \cite{minimal_3D_Dubins_path_bounded_curvature_pitch}, \cite{analytic_solution_3D}, and \cite{mathworks_uavdubinsconnection} are shown in Fig.~\ref{fig: Long_1_Vana}. All three algorithms will produce the same path for any values of $R_{yaw}$ and for any of the three combinations of initial and final roll angles we considered in this subsection. This is due to the fact that the generated path does not account for the complete orientation of the vehicle at the initial and final locations; this highlights the novelty of our approach considering the full orientation of the vehicle.

We also note here that all the paths satisfy the pitch and yaw rate constraints close to the initial configuration when $R_{yaw} = 40$ m. However, the paths from \cite{minimal_3D_Dubins_path_bounded_curvature_pitch} and \cite{mathworks_uavdubinsconnection} violate the yaw rate constraints at the initial configuration for $R_{yaw} = 50$ m, while the path from \cite{analytic_solution_3D} violates for $R_{yaw} = 53$ m, since they enter the left sphere (similar to Fig.~\ref{subfig: yaw_rate_violation}). These results reaffirm the importance of the model and control inputs presented in the current paper.

\begin{figure*}[htb!]
    \centering
    \subfloat[Initial roll $= -15^\circ,$ final roll $= 0^\circ$ (path through cross-tangent plane)]{\includegraphics[width = 0.32\linewidth]{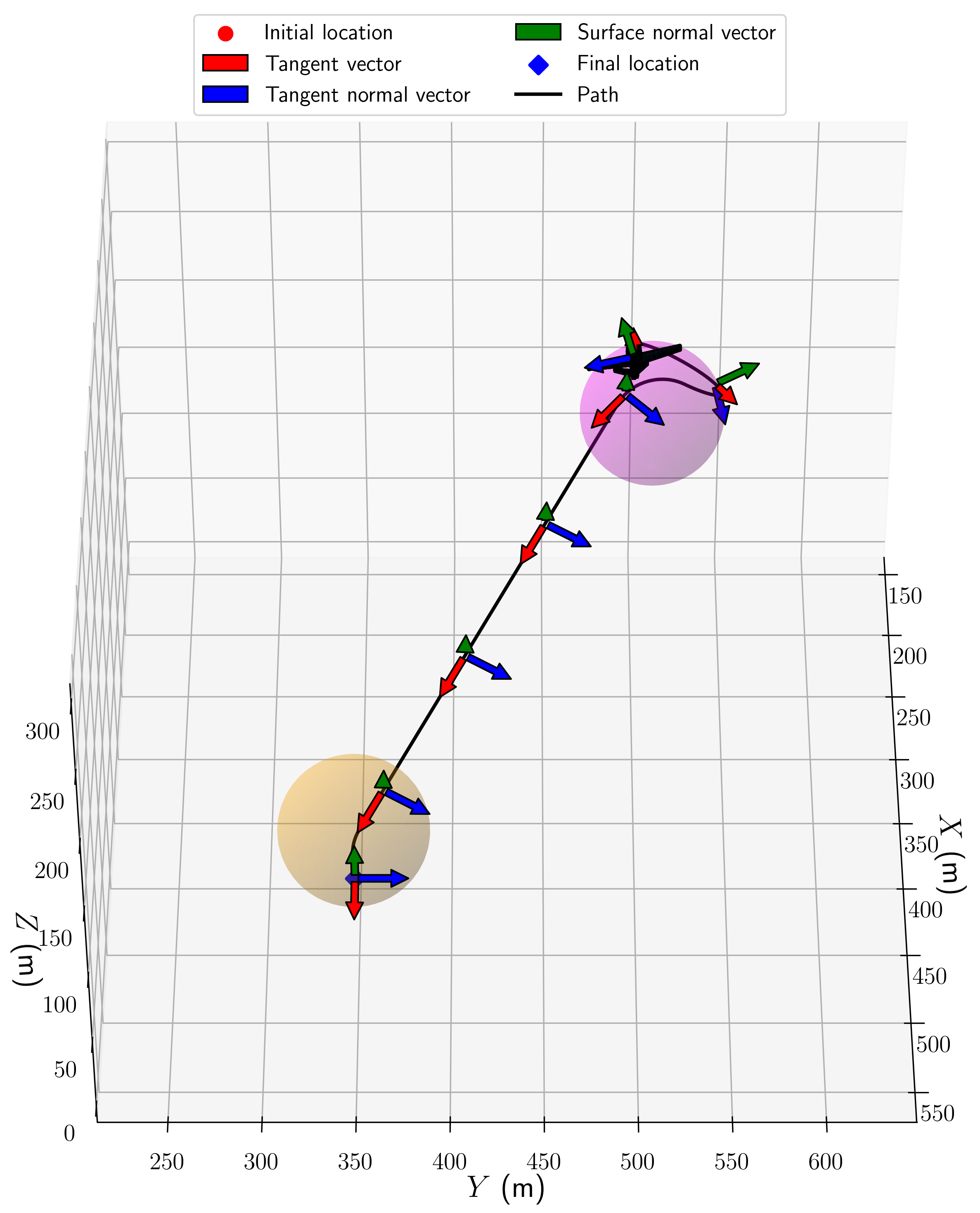}
    \label{subfig: Long 1 ini_roll -15 fin_roll 0 Ryaw 40}} \hfill
    \subfloat[Initial roll $= 0^\circ,$ final roll $= 15^\circ$ (path through cylindrical envelope)]{\includegraphics[width = 0.32\linewidth]{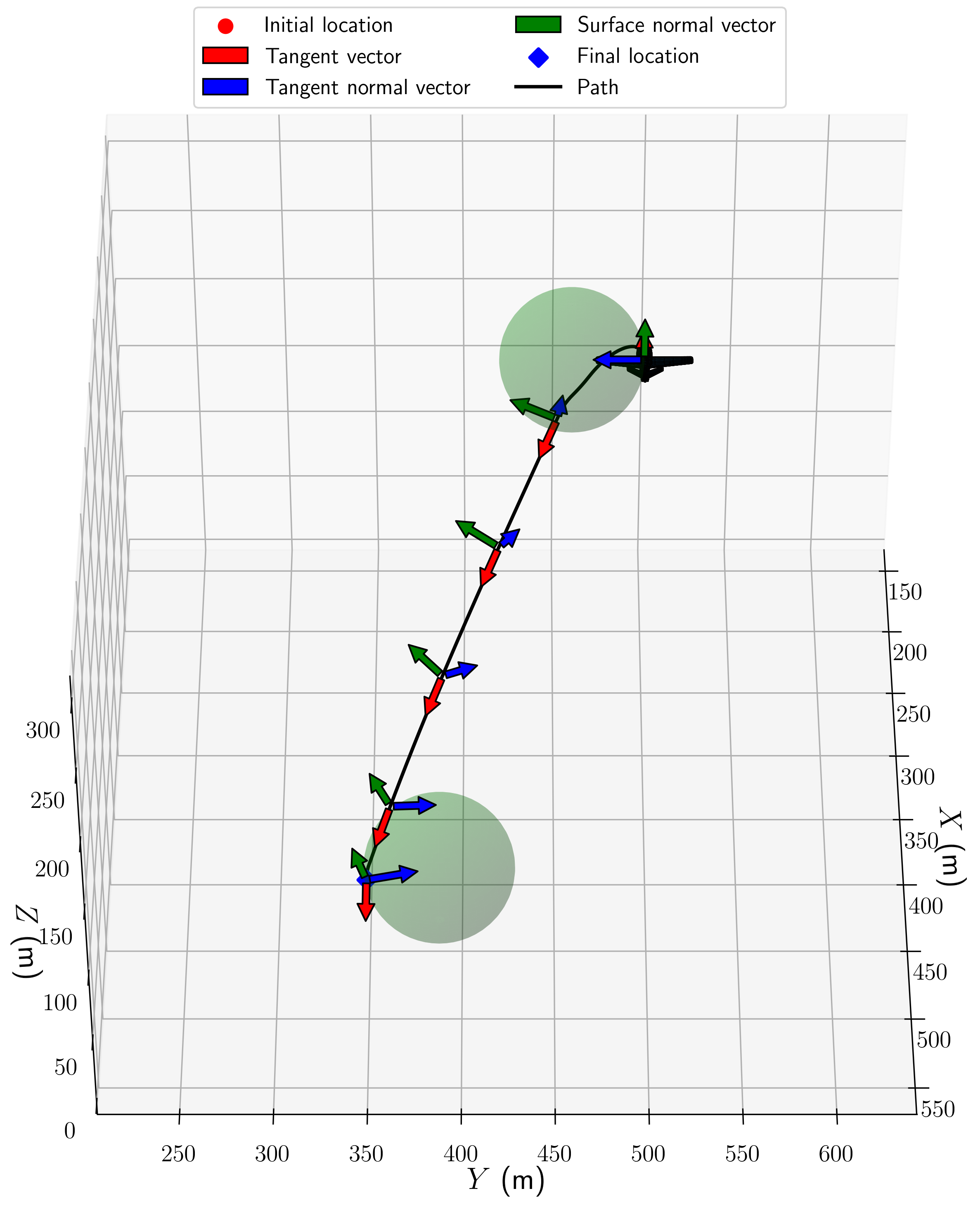}
    \label{subfig: Long 1 ini_roll 0 fin_roll 15 Ryaw 40}} \hfill
    \subfloat[Initial roll $= 15^\circ,$ final roll $-15^\circ$ (path through cylindrical envelope)]{\includegraphics[width = 0.32\linewidth]{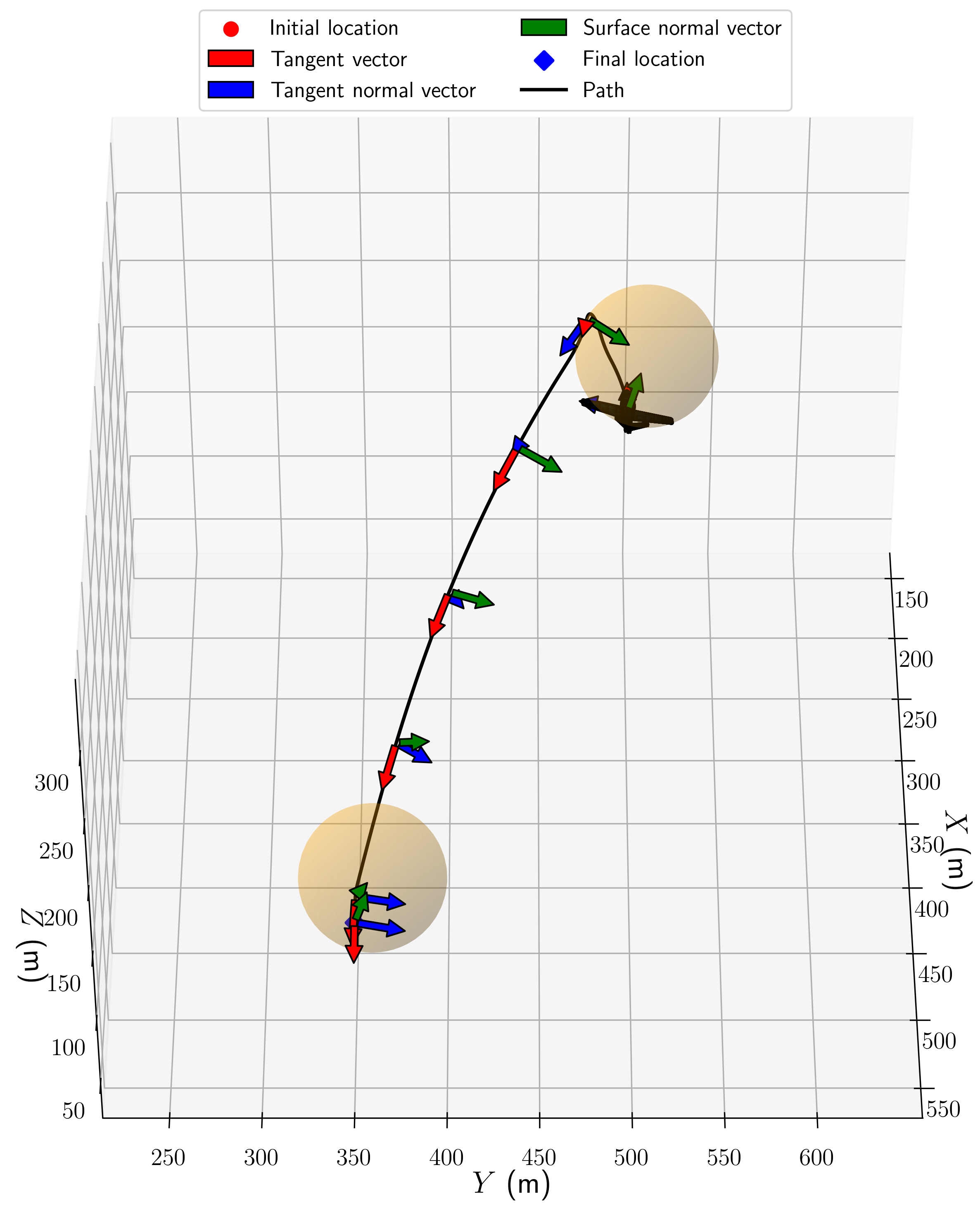}
    \label{subfig: Long 1 ini_roll 15 fin_roll -15 Ryaw 40}}
    \caption{Depiction of varying paths with initial and final roll angles for ``Long 1" with $R_{yaw} = 40$ m and instantaneous configuration of the vehicle. Animations of a vehicle moving along these paths are available on our GitHub page.}
    \label{fig: varying_paths_config}
\end{figure*}

\begin{figure}[htb!]
    \centering
    \includegraphics[width=0.8\linewidth]{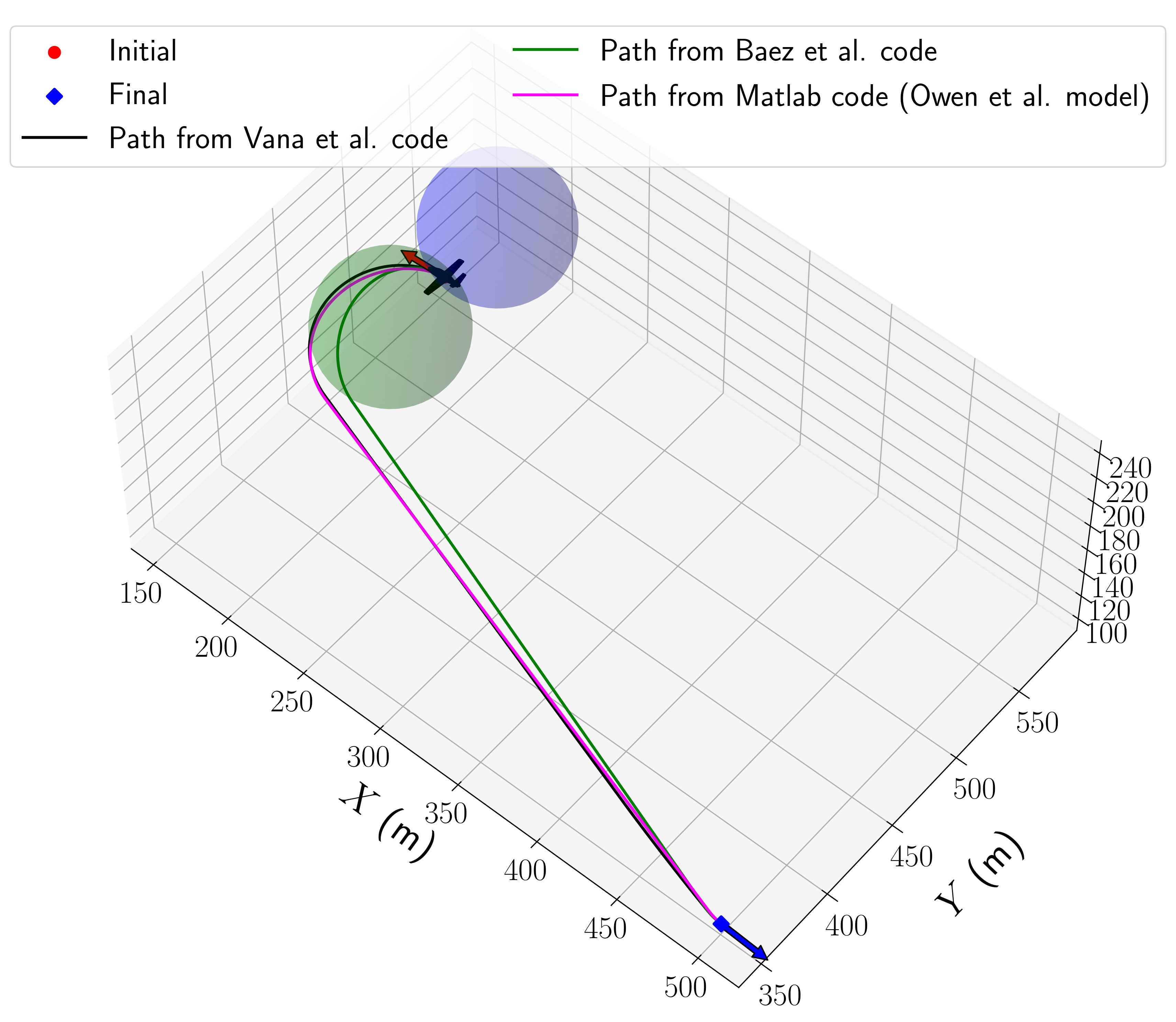}
    \caption{Feasible path from \cite{minimal_3D_Dubins_path_bounded_curvature_pitch}, \cite{analytic_solution_3D}, and \cite{mathworks_uavdubinsconnection} for ``Long 1"}
    \label{fig: Long_1_Vana}
\end{figure}

\subsection{Additional examples} \label{subsect: additional_examples_inside_pitch_yaw_spheres}

In addition to the previous instances, we consider two additional instances. First, we consider an instance where the final location is inside the ``inner" sphere of the vehicle, i.e., one of the pitch spheres. In this instance, the vehicle needs to depart from the origin with a heading, pitch, and yaw angle of $30^\circ,$ $10^\circ,$ and $15^\circ,$ respectively. The desired final location is $(5, 10, 15)$ with a desired heading, pitch, and yaw angle of $190^\circ,$ $10^\circ,$ and $-15^\circ,$ respectively. Furthermore, we chose $R_{pitch}$ and $R_{yaw}$ to be $40$ m and $50$ m, respectively. The minimum cost path obtained using Algorithm~\ref{alg: algorithm} and from the three benchmarking algorithms are shown in Figs.~\ref{subfig: Maneuver_inside_top_sphere} and \ref{subfig: Vana_inside_top_sphere}, respectively. From our algorithm, the vehicle leverages the pitch and yaw rate bounds to make a sharper turn, leading to a path with a length of $253.36$ m, which traverses through an intermediary sphere. On the other hand, due to the pitch angle bounds in \cite{minimal_3D_Dubins_path_bounded_curvature_pitch}, the vehicle takes a longer path of length $290.57$ m. The algorithm \cite{mathworks_uavdubinsconnection} provides a comparable path length of $289.23$ m, whereas \cite{analytic_solution_3D} provides a much longer path with a length of $419.03$ m.

A similar result was obtained in the second instance, where the final location was chosen inside the right sphere (one of the yaw spheres). The initial configuration and vehicle parameters were chosen to be the same as the first instance, the final location is chosen to be $(0, -30, 5)$, and the desired final heading, pitch, and roll angles are $190^\circ,$ $10^\circ,$ and $-15^\circ,$ respectively. The path length from our algorithm was $257.27$ m, whereas the path length with the algorithms from \cite{minimal_3D_Dubins_path_bounded_curvature_pitch}, \cite{analytic_solution_3D}, and \cite{mathworks_uavdubinsconnection} were $293.03$ m, $391.13$ m, and $290.63$ m, respectively.

\begin{figure}[htb!]
    \centering
    \subfloat[Best path for final location inside pitch sphere from Algorithm~\ref{alg: algorithm} (best feasible path is left sphere -- right sphere -- left sphere). Animation of a vehicle moving along this path is available on our GitHub page.]{\includegraphics[width = 0.6\linewidth]{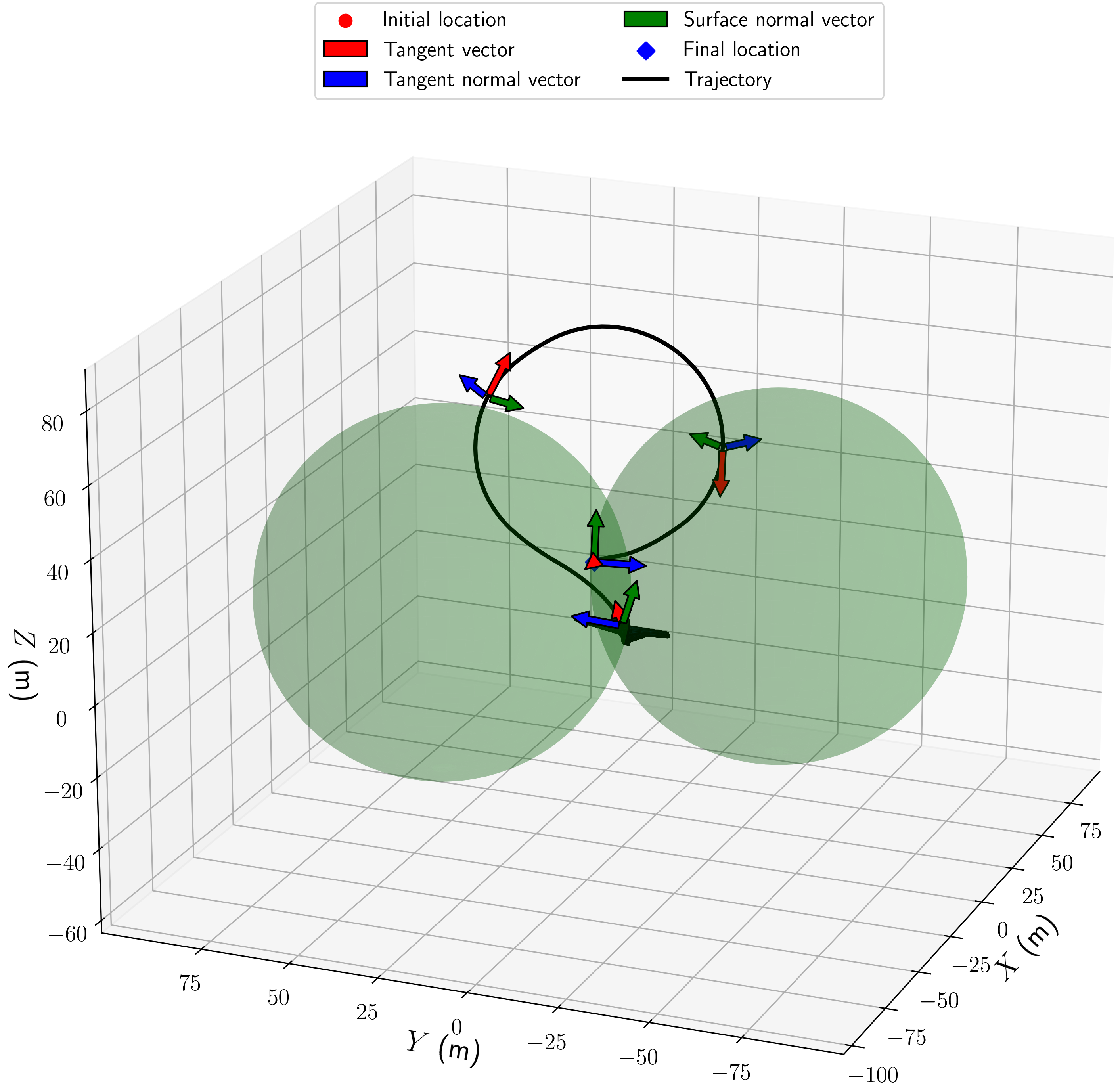}
    \label{subfig: Maneuver_inside_top_sphere}} \hfill
    \subfloat[Best path for same configuration using algorithms from \cite{minimal_3D_Dubins_path_bounded_curvature_pitch}, \cite{analytic_solution_3D}, and \cite{mathworks_uavdubinsconnection}]{\includegraphics[width = 0.7\linewidth]{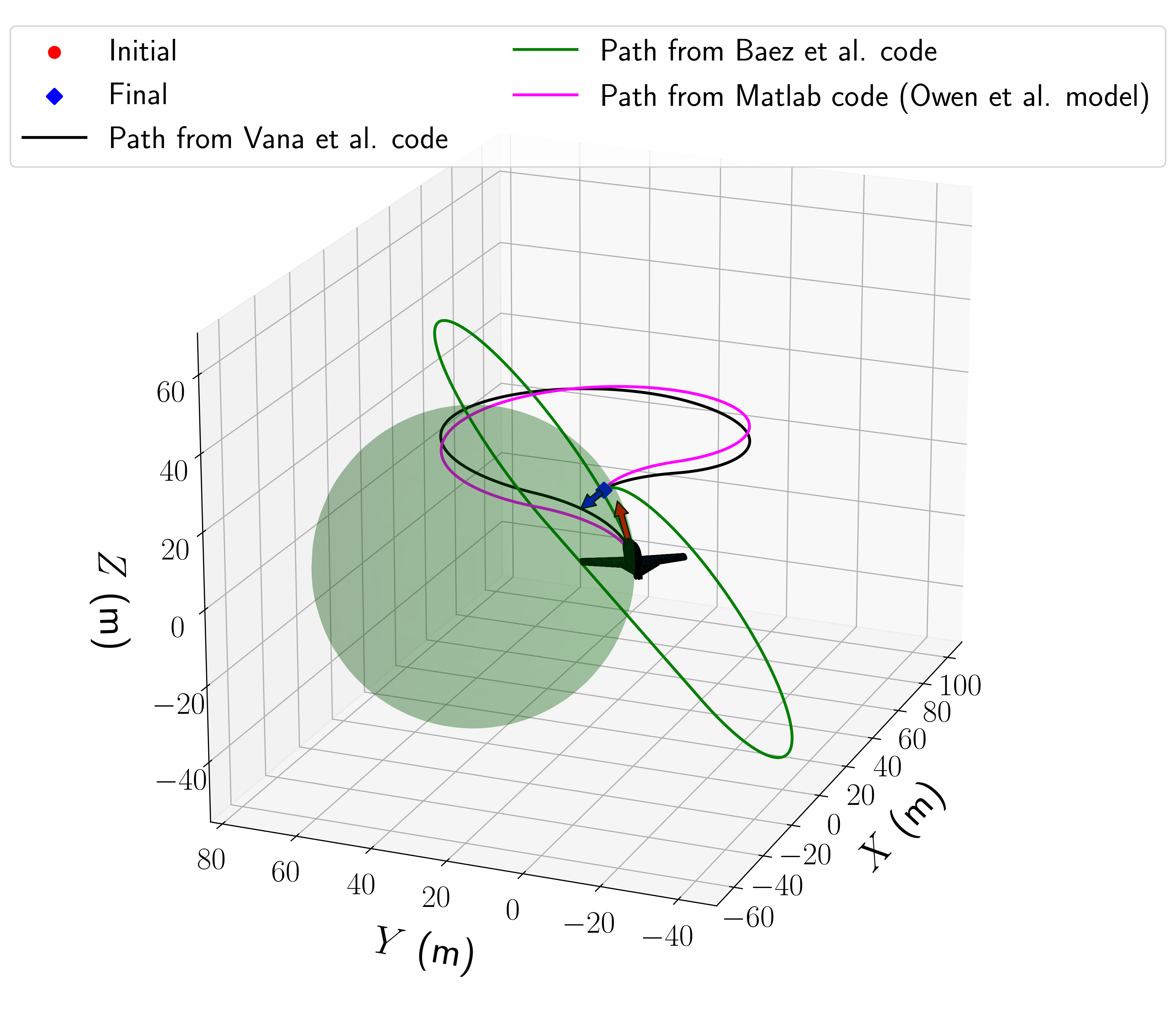}
    \label{subfig: Vana_inside_top_sphere}}
    \caption{Depiction of path for final location lying inside pitch sphere (specifications given in Section~\ref{subsect: additional_examples_inside_pitch_yaw_spheres})}
    \label{fig: maneuver_pitch_sphere}
\end{figure}

\begin{figure}[htb!]
    \centering
    \subfloat[Best path for final location inside yaw sphere from Algorithm~\ref{alg: algorithm} (best feasible path is right sphere -- left sphere -- right sphere). Animation of a vehicle moving along this path is available on our GitHub page.]{\includegraphics[width = 0.6\linewidth]{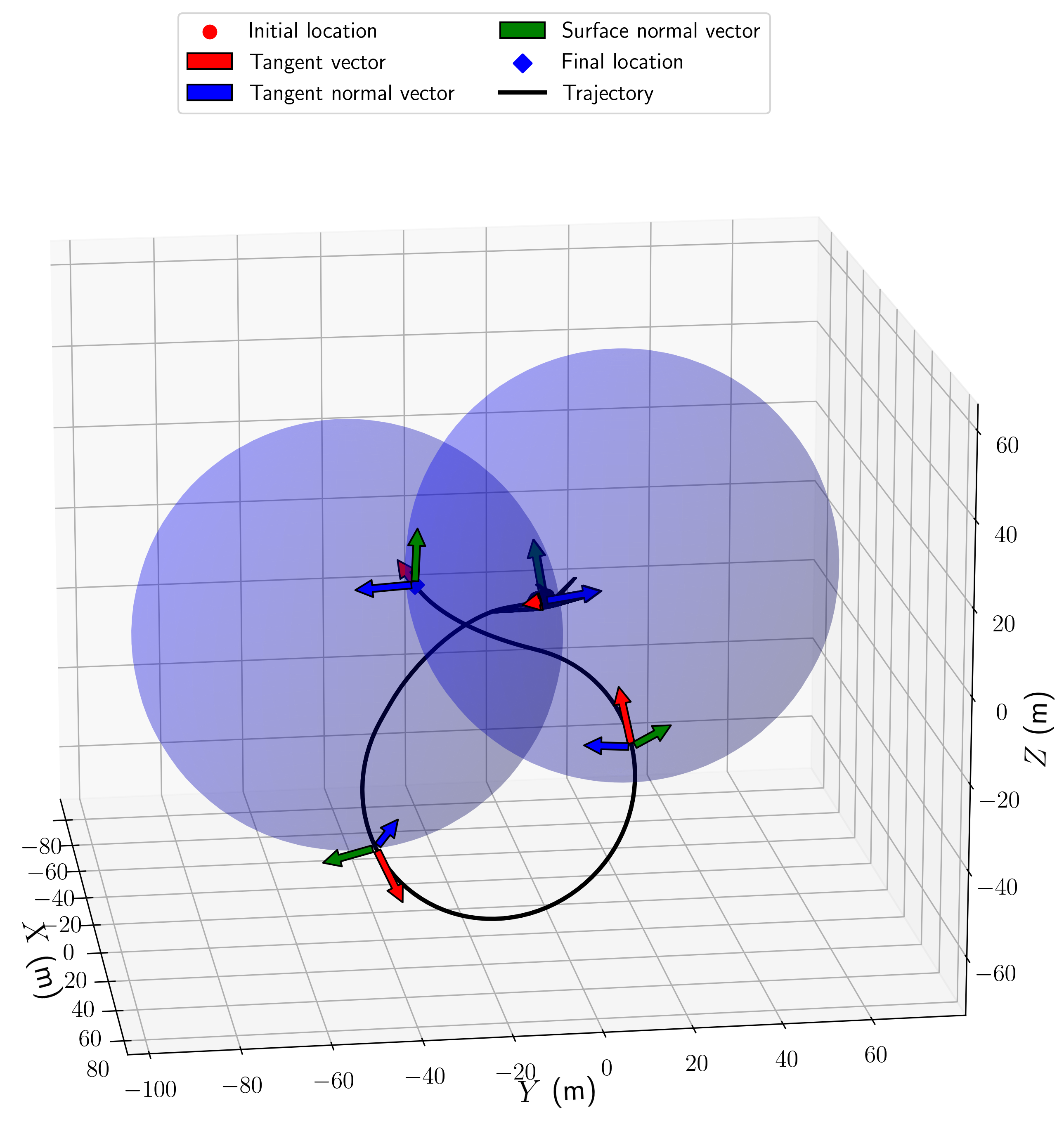}
    \label{subfig: Maneuver_inside_right_sphere}} \hfill
    \subfloat[Best path for same configuration using algorithms from \cite{minimal_3D_Dubins_path_bounded_curvature_pitch}, \cite{analytic_solution_3D}, and \cite{mathworks_uavdubinsconnection}]{\includegraphics[width = 0.7\linewidth]{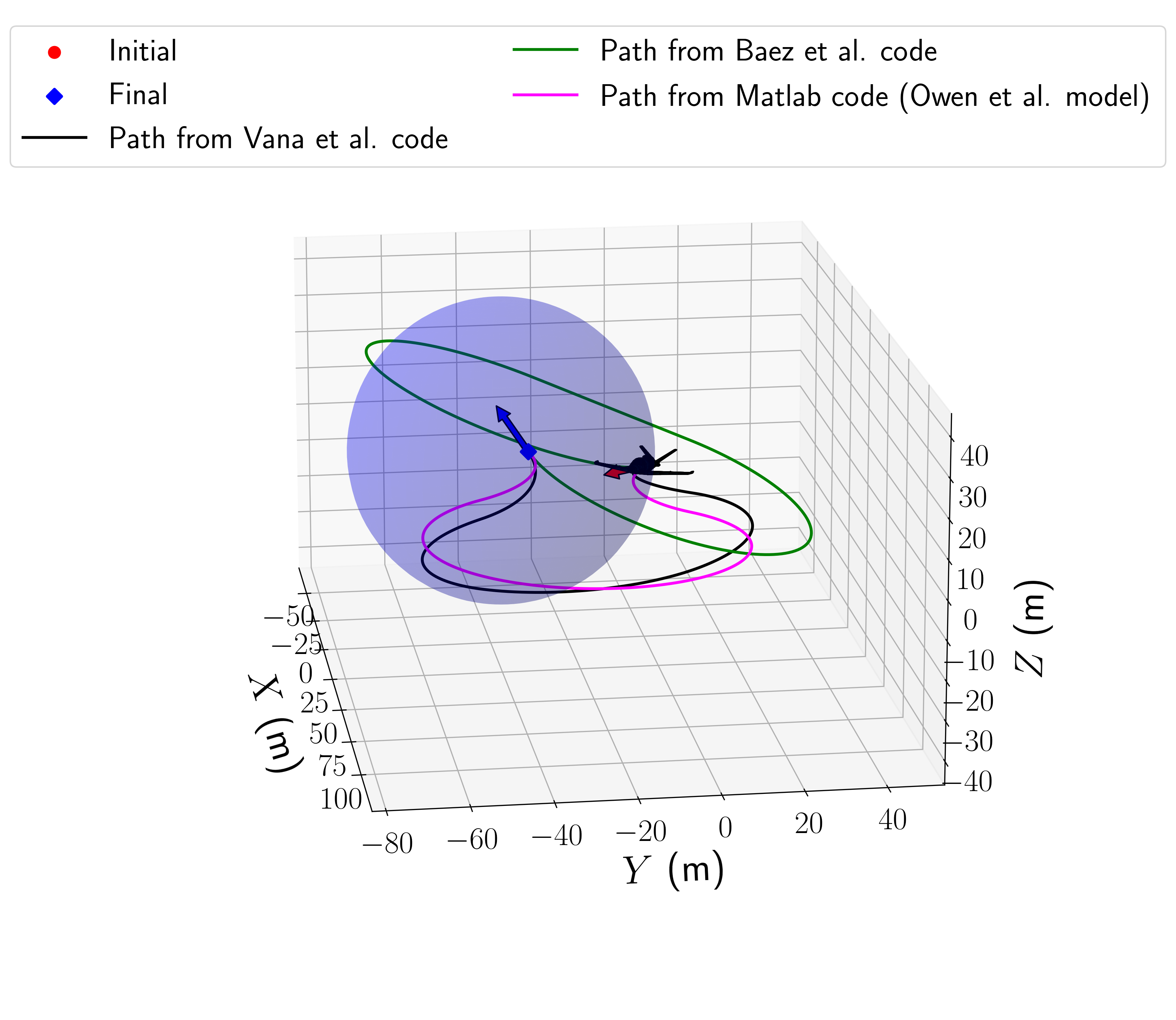}
    \label{subfig: Vana_inside_right_sphere}}
    \caption{Depiction of path for final location lying inside yaw sphere (specifications given in Section~\ref{subsect: additional_examples_inside_pitch_yaw_spheres})}
    \label{fig: maneuver_yaw_sphere}
\end{figure}

\begin{remark}
    We identified two key issues with state-of-the-art methods: (i) the generated path does not change with changing roll angle, and (ii) the use of a single turning radius constraint can lead to violations of one of the curvature constraints in our model. While we have illustrated these issues with a few examples, they are expected to persist across a broader range of scenarios.
\end{remark}

\subsection{Impact of discretization of parameters}

Noting that the number of discretizations of the path parameters can be freely chosen, we performed a sensitivity study by varying them across 5, 10, 15, 20, and 25 values. For each setting, we obtained the best path using our algorithm for all instances and variations in $R_{yaw}$ and initial and final roll angles, as outlined in Table~I. The results are summarized in Fig.~\ref{fig: discretization}, which shows both the change in shortest path lengths and the computation times across the 12 considered instances. For each instance, the nine variations ($R_{\mathrm{yaw}}$, initial, and final roll angles) are condensed into a box plot.

From these figures, we observe that a smaller number of discretizations yields faster computation time, but typically at the expense of solution quality. This trend is consistent across all instances, and is especially apparent in ``Additional 2.'' From Fig.~\ref{fig: discretization}, 15 discretizations represent a good trade-off between computation time and path length. Alternatively, using 10 discretizations provides solutions in approximately 5 seconds, compared to 10 seconds for 15 discretizations, albeit with some compromise in path quality.

\begin{figure}[htb!]
    \centering
    \subfloat[Percentage change in shortest path's length]{\includegraphics[width = \linewidth]{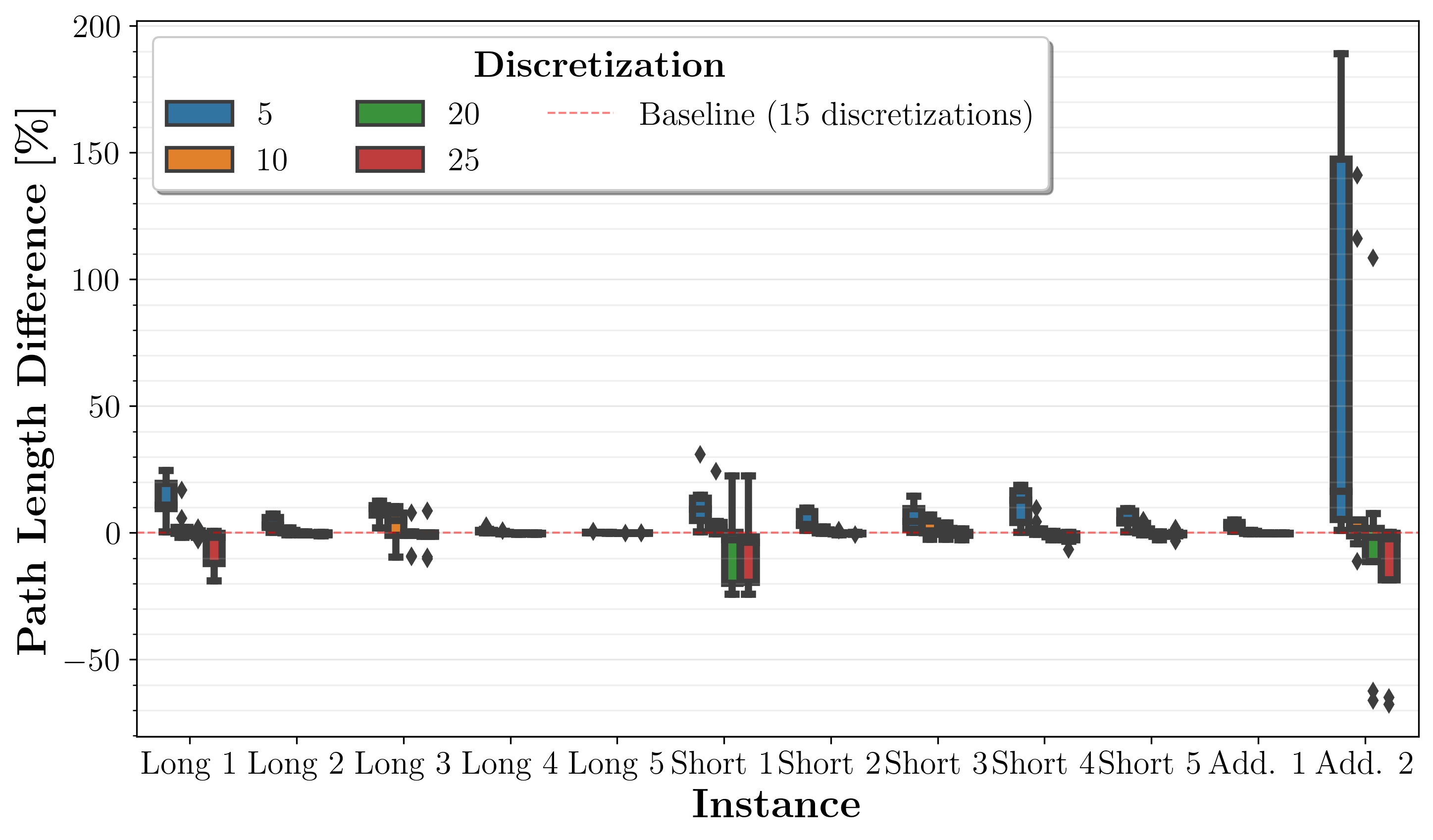}\label{subfig: disc_path_length}}
    \hfil
    \subfloat[Varying computation time of Algorithm~\ref{alg: algorithm} (discretizations shown with a colored marker as well)]{\includegraphics[width = \linewidth]{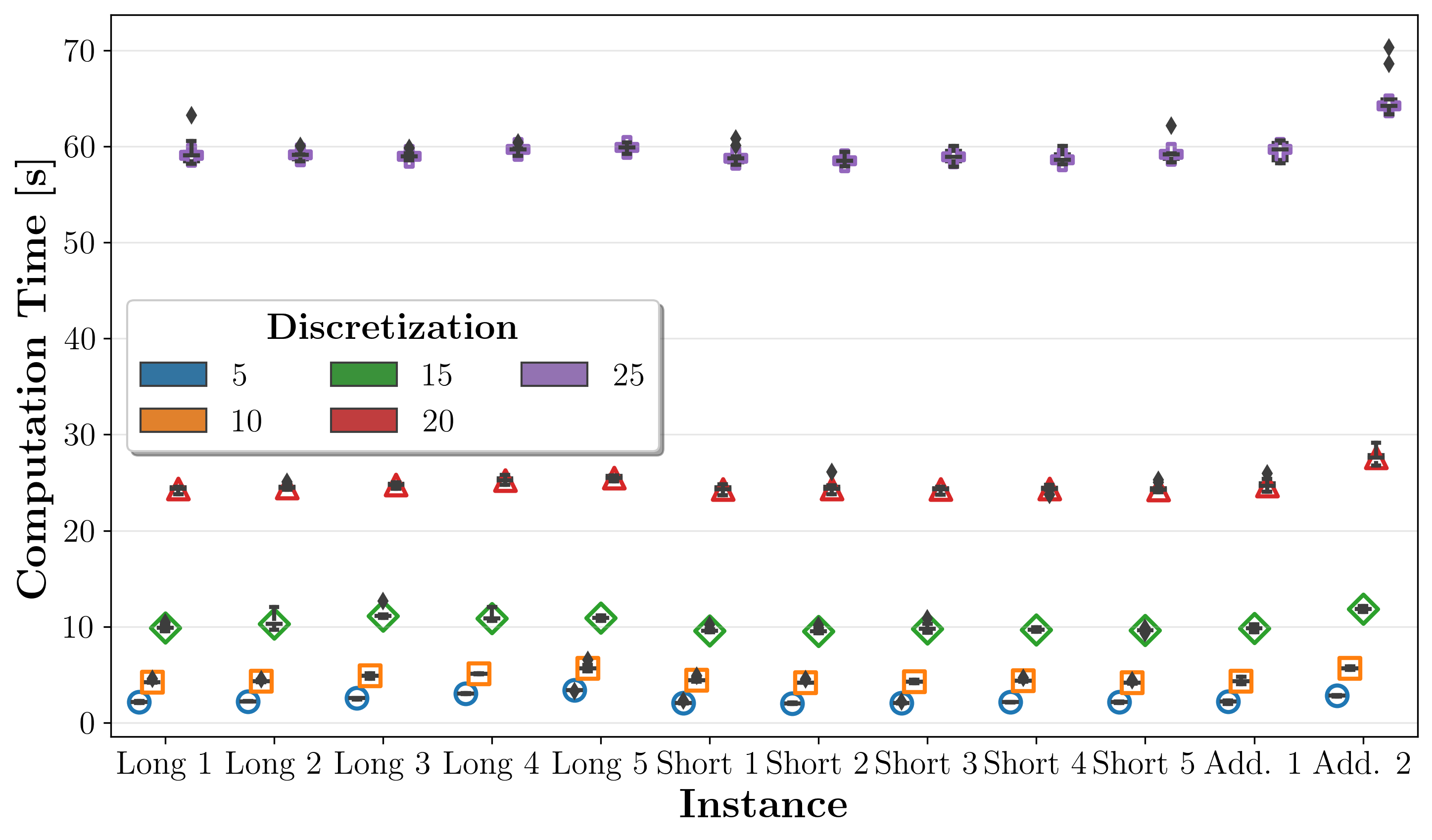}\label{subfig: disc_comp_time}}
    \caption{Impact of varying the discretization of parameters on path length and computation time for each instance (shown with a box plot)}
    \label{fig: discretization}
\end{figure}

\section{Conclusion}

In this paper, we propose a novel model for 3D motion planning and a methodology for generating high-quality feasible trajectories. The proposed model and the approach address two key limitations of the existing methods: the incomplete representation of vehicle configuration, and inadequate modeling of the motion constraints. We highlight these issues using illustrative examples. To address these, we proposed a model using a rotation-minimizing frame to uniquely represent the vehicle's configuration, and the model includes two control inputs corresponding to the pitch rate and yaw rate of the vehicle. We proved that the pitch rate and yaw rate bounds yield a total of four distinct spheres tangential to the vehicle's configuration, \emph{viz.} inner, outer, left and right, which represent temporarily inaccessible regions. 

We proposed three classes of curvature-constrained paths beginning and ending on the surface of the spheres, tangential to the initial and final configurations. The three classes differ in the transition mechanism from the initial to the final sphere. These transitions involve a path on the surface of a cylindrical envelope, a cross-tangent plane, or another spherical surface. Finally, we presented extensive computational experiments to evaluate the impact of the motion constraints and the vehicle's configuration on the path length and the path classification. Additionally, a comparison of the proposed methodology with the algorithms in \cite{minimal_3D_Dubins_path_bounded_curvature_pitch}, \cite{analytic_solution_3D}, and \cite{mathworks_uavdubinsconnection} underscored the advantages of modeling with complete orientation. The proposed model and path generation methodology offer a novel perspective on the 3D motion planning problem.




\bibliographystyle{IEEEtran}
\bibliography{IEEEabrv, cite}

\begin{thebibliography}{10}
\providecommand{\url}[1]{#1}
\csname url@samestyle\endcsname
\providecommand{\newblock}{\relax}
\providecommand{\bibinfo}[2]{#2}
\providecommand{\BIBentrySTDinterwordspacing}{\spaceskip=0pt\relax}
\providecommand{\BIBentryALTinterwordstretchfactor}{4}
\providecommand{\BIBentryALTinterwordspacing}{\spaceskip=\fontdimen2\font plus
\BIBentryALTinterwordstretchfactor\fontdimen3\font minus
  \fontdimen4\font\relax}
\providecommand{\BIBforeignlanguage}[2]{{%
\expandafter\ifx\csname l@#1\endcsname\relax
\typeout{** WARNING: IEEEtran.bst: No hyphenation pattern has been}%
\typeout{** loaded for the language `#1'. Using the pattern for}%
\typeout{** the default language instead.}%
\else
\language=\csname l@#1\endcsname
\fi
#2}}
\providecommand{\BIBdecl}{\relax}
\BIBdecl

\bibitem{fixed_wing_UAV_1}
\BIBentryALTinterwordspacing
Fixed Wing Drone: The Complete Guide for Professionals (Accessed Apr 2025).
  [Online]. Available:
  \url{https://quantum-systems.com/blog/2025/02/05/fixed-wing-drone-guide/#:~:text=Superior%20Range%20%26%20Coverage%3A%20The%20design,to%20survey%20extensive%20landscapes%20quickly.}
\BIBentrySTDinterwordspacing

\bibitem{fixed_wing_UAV_2}
\BIBentryALTinterwordspacing
What Is A Fixed Wing Drone? | Advantages And Uses Of Fixed Wing Drones
  (Accessed Apr 2025). [Online]. Available:
  \url{https://uavsystemsinternational.com/blogs/drone-guides/what-is-a-fixed-wing-drone-advantages-and-uses-of-fixed-wing-drones?srsltid=AfmBOorCTIkSZHuUP3N4YbgcjBHHJoqQUL_wTsxMGUmw4nzJpJ1lvfb_}
\BIBentrySTDinterwordspacing

\bibitem{Dubins}
L.~E. Dubins, ``On curves of minimal length with a constraint on average
  curvature, and with prescribed initial and terminal positions and tangents,''
  \emph{American Journal of Mathematics}, vol.~79, 1957.

\bibitem{PMP}
L.~S. Pontryagin, V.~G. Boltyanskii, R.~V. Gamkrelidze, and E.~F. Mishchenko,
  \emph{The mathematical theory of optimal processes}.\hskip 1em plus 0.5em
  minus 0.4em\relax Interscience Publishers, 1962.

\bibitem{Shortest_path_synthesis_Boissonat}
X.-N. Bui, J.-D. Boissonnat, P.~Soueres, and J.-P. Laumond, ``Shortest path
  synthesis for dubins non-holonomic robot,'' in \emph{Proceedings of the 1994
  IEEE International Conference on Robotics and Automation}, 1994, pp. 2--7.

\bibitem{sinistral/dextral}
E.~Bakolas and P.~Tsiotras, ``The asymmetric sinistral/dextral markov-dubins
  problem,'' in \emph{Proceedings of the 48h IEEE Conference on Decision and
  Control (CDC)}, 2009, pp. 5649--5654.

\bibitem{3D_underwater}
Y.~Wang and Y.~R. Zheng, ``3-dimensional path planning for autonomous
  underwater vehicle,'' in \emph{OCEANS 2018 MTS/IEEE Charleston}, 2018, pp.
  1--6.

\bibitem{motion_planning_two_3D_Dubins_vehicles}
H.~Marino, M.~Bonizzato, R.~Bartalucci, P.~Salaris, and L.~Pallottino, ``Motion
  planning for two {3D-Dubins} vehicles with distance constraint,'' in
  \emph{2012 IEEE/RSJ International Conference on Intelligent Robots and
  Systems}, 2012, pp. 4702--4707.

\bibitem{path_generation_tracking_3D}
G.~Ambrosino, M.~Ariola, U.~Ciniglio, F.~Corraro, E.~De~Lellis, and A.~Pironti,
  ``Path generation and tracking in 3-d for uavs,'' \emph{IEEE Transactions on
  Control Systems Technology}, vol.~17, no.~4, pp. 980--988, 2009.

\bibitem{rick_lind}
\BIBentryALTinterwordspacing
R.~Hurley, R.~Lind, and J.~Kehoe, ``A mixed local-global solution to motion
  planning within 3-d environments,'' in \emph{AIAA Guidance, Navigation, and
  Control Conference}, 2009. [Online]. Available:
  \url{https://arc.aiaa.org/doi/abs/10.2514/6.2009-6297}
\BIBentrySTDinterwordspacing

\bibitem{path_planning_3D_Dubins_saripalli}
Y.~Lin and S.~Saripalli, ``Path planning using {3D Dubins} curve for unmanned
  aerial vehicles,'' in \emph{2014 International Conference on Unmanned
  Aircraft Systems (ICUAS)}, 2014, pp. 296--304.

\bibitem{sussman_3D}
H.~Sussmann, ``Shortest 3-dimensional paths with a prescribed curvature
  bound,'' in \emph{IEEE Conference on Decision and Control}, 1995, pp.
  3306--3312.

\bibitem{analytic_solution_3D}
V.~M. Baez, N.~Navkar, and A.~T. Becker, ``An analytic solution to the 3{D} csc
  {D}ubins path problem,'' in \emph{2024 IEEE International Conference on
  Robotics and Automation (ICRA)}, 2024, pp. 7157--7163.

\bibitem{optimal_geometrical_path_in_3D}
S.~Hota and D.~Ghose, ``Optimal geometrical path in 3{D} with curvature
  constraint,'' in \emph{2010 IEEE/RSJ International Conference on Intelligent
  Robots and Systems}, 2010, pp. 113--118.

\bibitem{optimal_path_planning_aerial_vehicle_3D}
S.~Hota and D.~Ghose, ``Optimal path planning for an aerial vehicle in 3{D}
  space,'' in \emph{49th IEEE Conference on Decision and Control (CDC)}, 2010,
  pp. 4902--4907.

\bibitem{reparametrization_3D_Dubins}
L.~Xu, Y.~Baryshnikov, and C.~Sung, ``Reparametrization of 3{D} csc {D}ubins
  paths enabling 2d search,'' in \emph{Algorithmic Foundations of Robotics
  XVI}, 2024.

\bibitem{time_optimal_paths_Dubins_airplane}
H.~Chitsaz and S.~M. LaValle, ``Time-optimal paths for a {D}ubins airplane,''
  in \emph{2007 46th IEEE Conference on Decision and Control}, 2007, pp.
  2379--2384.

\bibitem{Dubins_airplane_fixed_wing_UAVs}
\BIBentryALTinterwordspacing
M.~Owen, R.~W. Beard, and T.~W. McLain, \emph{Implementing Dubins Airplane
  Paths on Fixed-Wing UAVs*}.\hskip 1em plus 0.5em minus 0.4em\relax Dordrecht:
  Springer Netherlands, 2015, pp. 1677--1701. [Online]. Available:
  \url{https://doi.org/10.1007/978-90-481-9707-1_120}
\BIBentrySTDinterwordspacing

\bibitem{small_unmanned_aircraft}
R.~W. Beard and T.~W. McLain, \emph{Small Unmanned Aircraft: Theory and
  Practice}.\hskip 1em plus 0.5em minus 0.4em\relax Princeton University Press,
  2012.

\bibitem{goncalves}
V.~M. Goncalves, L.~C.~A. Pimenta, C.~A. Maia, B.~C.~O. Dutra, and G.~A.~S.
  Pereira, ``Vector fields for robot navigation along time-varying curves in
  $n$ -dimensions,'' \emph{IEEE Transactions on Robotics}, vol.~26, no.~4, pp.
  647--659, 2010.

\bibitem{minimal_3D_Dubins_path_bounded_curvature_pitch}
P.~Váňa, A.~Alves~Neto, J.~Faigl, and D.~G. Macharet, ``Minimal {3D Dubins}
  path with bounded curvature and pitch angle,'' in \emph{2020 IEEE
  International Conference on Robotics and Automation (ICRA)}, 2020, pp.
  8497--8503.

\bibitem{finding_3D_dubins_paths_pitch_angle_nonlinear_optimization}
J.~Herynek, P.~Váňa, and J.~Faigl, ``Finding 3{D} {D}ubins paths with pitch
  angle constraint using non-linear optimization,'' in \emph{2021 European
  Conference on Mobile Robots (ECMR)}, 2021, pp. 1--6.

\bibitem{towards_finding_shortest_paths_3D_rigid_bodies}
W.~Wang and P.~Li, ``Towards finding the shortest-paths for {3D} rigid
  bodies,'' in \emph{Robotics: Science and Systems 2021}, 2021.

\bibitem{kumar2025newapproachmotionplanning}
D.~P. Kumar, S.~Darbha, S.~G. Manyam, and D.~W. Casbeer, ``A new approach to
  motion planning in 3d for a dubins vehicle: Special case on a sphere,''
  \emph{IEEE Transactions on Robotics}, pp. 1--18, 2026.

\bibitem{Bishop}
R.~L. Bishop, ``There is more than one way to frame a curve,'' \emph{The
  American Mathematical Monthly}, vol.~82, no.~3, pp. 246--251, 1975.

\bibitem{3D_Dubins_sphere}
S.~Darbha, A.~Pavan, R.~Kumbakonam, S.~Rathinam, D.~W. Casbeer, and S.~G.
  Manyam, ``Optimal geodesic curvature constrained dubins’ paths on a
  sphere,'' \emph{Journal of Optimization Theory and Applications}, vol. 197,
  pp. 966--992, 2023.

\bibitem{kumar2025generationpathsmotionplanning}
\BIBentryALTinterwordspacing
D.~P. Kumar, S.~Darbha, S.~G. Manyam, and D.~Casbeer, ``Generation of paths for
  motion planning for a dubins vehicle on sphere,'' 2025. [Online]. Available:
  \url{https://arxiv.org/abs/2504.11832}
\BIBentrySTDinterwordspacing

\bibitem{struik}
D.~J. Struik, \emph{Lectures on Classical Differential Geometry}, 2nd~ed.\hskip
  1em plus 0.5em minus 0.4em\relax Dover, 1988, ch.~4.

\bibitem{dubins_classification}
A.~M. Shkel and V.~Lumelsky, ``Classification of the dubins set,''
  \emph{Robotics and Autonomous Systems}, vol.~34, pp. 179--202, 2001.

\bibitem{mathworks_uavdubinsconnection}
\BIBentryALTinterwordspacing
{MathWorks}, ``{uavDubinsConnection}: {Dubins} path connection for {UAV},''
  2019, {MATLAB} Documentation (accessed March 1, 2026). [Online]. Available:
  \url{https://www.mathworks.com/help/uav/ref/uavdubinsconnection.connect.html}
\BIBentrySTDinterwordspacing

\bibitem{do_carmo}
M.~P.~D. Carmo, \emph{Differential Geometry of Curves \& Surfaces},
  2nd~ed.\hskip 1em plus 0.5em minus 0.4em\relax Dover, 2016, ch.~3.

\end{thebibliography}

\appendix

\subsection{Construction of segments} \label{appsubsect: construction_segments}

Consider an interval $s \in [s_0, s_1]$ in which $\kappa_g$ and $\kappa_n$ are constants.  In this case, noting that $\mathbf{R} (s):= \begin{bmatrix}
    \mathbf{T} (s) & \mathbf{Y} (s) & \mathbf{U} (s)
\end{bmatrix}$ is a rotation matrix, \eqref{eq:kinematics2}, \eqref{eq:kinematics3}, and \eqref{eq:kinematics4} can be rewritten as
\begin{align} \label{eq: Darboux_frame_rotation}
    \mathbf{R}' (s) &= \mathbf{R} (s) \underbrace{\begin{pmatrix}
        0 & -\kappa_g & -\kappa_n \\
        \kappa_g & 0 & 0 \\
        \kappa_n & 0 & 0
    \end{pmatrix}}_{\Omega}.
\end{align}
Suppose at least one of $\kappa_g$ and $\kappa_n$ is non-zero. The solution for the above differential equation is given by $\mathbf{R} (s) = \mathbf{R} (s_0) e^{\Omega (s - s_0)},$ where the expression for the exponential of the skew-symmetric matrix can be obtained using the Euler-Rodriguez formula. Furthermore, using the obtained expression for $\mathbf{T} (s),$ the solution for $\mathbf{X} (s)$ can be obtained by integrating $\mathbf{X}' (s)$ from \eqref{eq:motion}. Hence, the solution for the evolution of the four vectors can be written as
\begin{align} \label{eq: evolution_rotation_position}
    \begin{bmatrix}
        \mathbf{R} (\phi) & \mathbf{X} (\phi) \\
        \mathbf{0}_{3 \times 1} & 1
    \end{bmatrix} = \begin{bmatrix}
        \mathbf{R} (0) & \mathbf{X} (0) \\
        \mathbf{0}_{3 \times 1} & 1
    \end{bmatrix} \mathbf{H} (\phi).
\end{align}
Here, $\phi := (s - s_0) \sqrt{\kappa_n^2 + \kappa_g^2}$ and denotes the arc angle of the considered segment, and $\mathbf{H} (\phi)$ is given by
\begin{align} \label{eq: homogeneous_transformation}
    \mathbf{H} = \begin{pmatrix}
        c {\phi} & \frac{-\kappa_g s {\phi}}{K} & \frac{-\kappa_n s {\phi}}{K} & \frac{s {\phi}}{K} \\
        \frac{\kappa_g s {\phi}}{K} & \frac{\kappa_n^2 + c {\phi} \kappa_g^2}{K^2} & \frac{-\kappa_g \kappa_n (1 - c {\phi})}{K^2} & \frac{\kappa_g \left(1 - c {\phi} \right)}{K^2} \\
        \frac{\kappa_n s {\phi}}{K} & \frac{-\kappa_g \kappa_n (1 - c {\phi})}{K^2} & \frac{\kappa_g^2 + \kappa_n^2 c {\phi}}{K^2} & \frac{\kappa_n \left(1 - c {\phi} \right)}{K^2} \\
        0 & 0 & 0 & 1
    \end{pmatrix},
\end{align}
where $K := \sqrt{\kappa_n^2 + \kappa_g^2}$ and $c \phi := \cos{\phi}, s \phi := \sin{\phi}.$ We can observe that the solution is periodic with a period of $2 \pi.$ 

In the case of $\kappa_g = \kappa_n = 0,$ $\mathbf{T}, \mathbf{Y},$ and $\mathbf{U}$ remain constant (from \eqref{eq:motion}); furthermore, $\mathbf{X} (s) = \mathbf{X} (0) + s \mathbf{T} (0)$ is obtained.

\subsection{Proof for Lemma~\ref{lemma: kappan}} \label{appsect: Proof_Lemma_1}

Without loss of generality, consider the initial rotation matrix $\mathbf{R} (0)$ to be the identity matrix and the initial location $\mathbf{X} (0)$ to coincide with the origin. Hence, using the closed-form expressions in Appendix~\ref{appsubsect: construction_segments}, the position is given by 
\begin{align} \label{eq: position_function}
\mathbf{X} (\phi) = \begin{pmatrix}
    \frac{s \phi}{K} & \frac{\kappa_g \left(1 - c \phi \right)}{K^2} & \frac{\kappa_n \left(1 - c \phi \right)}{K^2}
\end{pmatrix}^T,
\end{align}
where $K = \sqrt{\kappa_n^2 + \kappa_g^2}.$ Consider $\kappa_n = \frac{1}{R_{pitch}}.$ We claim that the obtained position $\mathbf{X} (\phi)$ lies on a sphere centered at $(0, 0, R_{pitch})^T,$ which is along $\mathbf{U},$ with radius $R_{pitch}.$ To this end, we can show that
$\|\mathbf{X} (\phi) - (0, 0, R_{pitch})^T \|_2^2 = R_{pitch}^2,$ 
for all $\kappa_g \in \left[-\frac{1}{R_{yaw}}, \frac{1}{R_{yaw}} \right].$ Hence, all segments corresponding to $\kappa_n = \frac{1}{R_{pitch}}$ lie on a sphere centered at $(0, 0, R_{pitch})^T,$ with radius $R_{pitch}.$ A similar argument can be made for $\kappa_n = - \frac{1}{R_{pitch}},$ where all segments lie on a sphere centered at $(0,0,-R_{pitch})^T$ with radius of $R_{pitch}.$ Therefore, we can observe that $\kappa_n$ controls the pitch motion of the vehicle; furthermore, $\kappa_n = \pm \frac{1}{R_{pitch}}$ yield maximum ascent or descent motion for the vehicle.

The radius of a segment corresponding to when $\kappa_g$ is constant in $\left[-\frac{1}{R_{yaw}}, \frac{1}{R_{yaw}} \right]$ can be obtained using the expression for $\mathbf{X} (\phi)$ as $\frac{1}{2}\|\mathbf{X} (\pi) - \mathbf{X} (0) \|_2.$ Using the expression for $\mathbf{X} (\phi)$ given in \eqref{eq: position_function}, it follows that $\frac{1}{2}\|\mathbf{X} (\pi) - \mathbf{X} (0) \|_2 = \frac{1}{\sqrt{\kappa_g^2 + \frac{1}{R_{pitch}^2}}}.$

\subsection{Sabban frame equations for sphere with radius $\overline{R}$ and path obtained in 3D} \label{appsubsect: Sabban_frame_generalization}

The evolution equations for the Sabban frame on a unit sphere, described in Section~\ref{subsect: initial_final_spheres_path_generation}, can be generalized to a sphere with radius $\overline{R}.$ To this end, the arc length $s$ on the sphere with radius $\overline{R}$ is defined to be $s := \overline{R} \hat{s},$ where $\hat{s}$ is the arc length on the unit sphere. Furthermore, $\mathbf{X}_{sp} := \overline{R} \hat{\mathbf{X}}_{sp}$ depicts the location on the new sphere; here, $\hat{\mathbf{X}}_{sp}$ denotes the location on the unit sphere (refer to Fig.~\ref{fig: sphere_motion_planning}). Additionally, the bound for the control input from the unit sphere ($=\hat{U}_{max}$) is scaled to obtain $U_{max} := \left(\overline{R}\right)^{-1} \hat{U}_{max}.$ Following a similar process as Lemma~3.2 in \cite{3D_Dubins_sphere}, we can show that the minimum turning radius $r$ on the sphere of radius $\overline{R}$ is given by $r = \overline{R} \hat{r}.$ A depiction of the scaling for the initial configuration is shown in Fig.~\ref{fig: sphere_motion_planning}.

The evolution equations for the sphere with radius $\overline{R}$ can therefore be obtained as
\begin{align} \label{eq: Sabban_frame_generalized}
\begin{split}
    \frac{d\mathbf{X}_{sp}}{ds} (s) &= \mathbf{T}_{sp} (s), \quad
    \frac{d\mathbf{T}_{sp}}{ds} (s) = - \frac{1}{\overline{R}^2} \mathbf{X}_{sp} (s) + u_g \mathbf{N}_{sp} (s), \\
    \frac{d\mathbf{N}_{sp}}{ds} (s) &= - u_g \mathbf{T}_{sp} (s),
\end{split}
\end{align}
where $u_g \in [-U_{max}, U_{max}]$ is the geodesic curvature on the sphere of radius $\overline{R}$ and relates to the minimum turning radius $r$ on the sphere by $r = \frac{\overline{R}}{\sqrt{1 + U_{max}^2 \overline{R}^2}}.$ 
The evolution of these three vectors for $u_g \equiv U_{max}, 0,$ or $-U_{max},$ which correspond to left turn, great circular arc, and right turns on the sphere, can be obtained using the Euler-Rodriguez formula.

\subsection{Proof for Lemma~\ref{lemma: cylinder}} \label{appsubsect: proof_lemma_cylinder}

Consider cylinders connecting a pair of inner or outer spheres, which are shown in Fig.~\ref{subfig: cylindrical_envelope}. The tangent plane for these cylinders is the $\mathbf{T}-\mathbf{Y}$ plane. Consider the bound chosen for the geodesic curvature on the cylinder ($\kappa_{g, cyc}$) from \eqref{eq: geodesic_curvature_cylinder}, which is $\frac{1}{R_{yaw}}.$ We note that $\kappa_{g, cyc}$ denotes the magnitude of the projection of the curvature vector on the tangent plane \cite{struik}. Since the tangent plane of this cylinder is the $\mathbf{T} - \mathbf{Y}$ plane, $\kappa_{g, cyc}$ also represents the geodesic curvature for the rotation minimizing frame model. This is because geodesic curvature for the minimizing frame model ($\kappa_g$) controls the yaw motion (in the $\mathbf{T} - \mathbf{Y}$ plane) for the vehicle. Hence, the geodesic curvature bound for the rotation minimizing frame model is automatically satisfied for the chosen bound for $\kappa_{g, cyc}.$.

On the other hand, the normal curvature of a cylinder depends on the direction of traversal along the cylinder. If the curve is along a direction parallel to the axis of the cylinder, $\kappa_{n, cyc} = 0;$ however, if the curve is perpendicular to the axis, $\kappa_{n, cyc} = \frac{1}{\overline{R}},$ since the curve is along the circular cross-section \cite{do_carmo}. Since these curvatures of $0$ and $\frac{1}{\overline{R}}$ are the principal curvatures, the normal curvature for any intermediary direction of motion lies between the two through Euler's theorem \cite{struik}. Noting that the normal curvature is along the surface normal for the cylinder, which is along $\mathbf{U}$ for the rotation minimizing frame model, the normal curvature bounds for the rotation minimizing frame model are automatically satisfied. A similar argument applies for the selected bound for $\kappa_{g, cyc}$ for the left and right cylinders, with the only difference arising in considering $\mathbf{T}-\mathbf{U}$ as the tangent plane and the surface normal for the cylinders being $\mathbf{Y}$.

\subsection{Proof for Lemma~\ref{lemma: mapping_preservation}} \label{appsubsect: Lemma_mapping_proof}

Consider a cylinder that is parameterized in terms of $u$ and $v$ through \eqref{eq: Pcyc}. It suffices to show that the first fundamental form for the cylinder is the same as that for the plane, parameterized using $u$ and $v$ as given in \eqref{eq: Pplane}. The first fundamental form coefficients are given by (\cite{struik})
$E_{cyl} = \frac{\partial \mathbf{X}_{cyl}^\mathcal{U} (s)}{\partial u} \cdot \frac{\partial \mathbf{X}_{cyl}^\mathcal{U} (s)}{\partial u} = 1,$ $F_{cyl} = \frac{\partial \mathbf{X}_{cyl}^\mathcal{U} (s)}{\partial u} \cdot \frac{\partial \mathbf{X}_{cyl}^\mathcal{U} (s)}{\partial v} = 0,$ and $G_{cyl} = \frac{\partial \mathbf{X}_{cyl}^\mathcal{U} (s)}{\partial v} \cdot \frac{\partial \mathbf{X}_{cyl}^\mathcal{U} (s)}{\partial v} = 1$.
Similarly, $E_{plane}, F_{plane},$ and $G_{plane}$ can be obtained to be $1,$ $0,$ and $1.$ Since the first fundamental form coefficients of the cylinder and the plane are equal, the length of the curve is preserved. This is because the distance between two closely spaced points on the curve on the cylinder that are initially separated by $du$ and $dv$ on the plane is given by $ds_{cyl}^2 = E_{cyl} du^2 + 2 F_{cyl} du dv + G_{cyl} dv^2 = ds_{plane}^2.$

\begin{IEEEbiography}[{\includegraphics[width=1in,height=1.25in,clip,keepaspectratio]{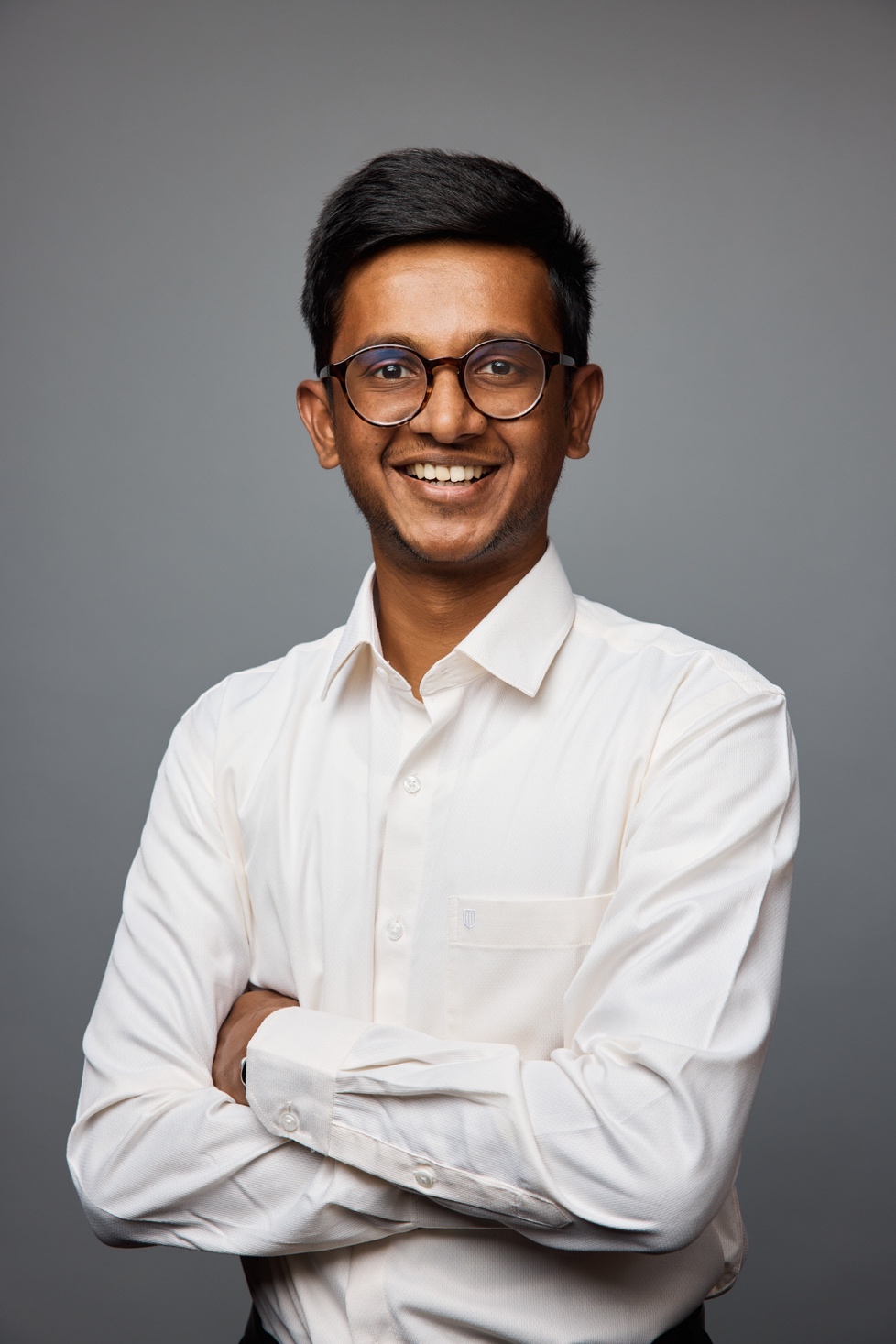}}]{Deepak Prakash Kumar} (Member, IEEE)
received his Ph.D. degree in Mechanical Engineering from Texas A\&M University in 2025. He received his B.Tech in Engineering Design and M.Tech in Automotive Engineering from the Engineering Design department at IIT Madras in 2020. He is currently a Postdoctoral Scholar with the Center for Resilient Autonomous Systems, Department of Electrical Engineering and Computer Science, University of California, Irvine. His research interests include safe physical AI algorithms for multi-agent collaboration, physical AI for human--robot teaming, motion planning and control for autonomous vehicles, and vehicle routing algorithms.
\end{IEEEbiography}

\begin{IEEEbiography}[{\includegraphics[width=1 in,height = 1.25 in,clip,keepaspectratio]{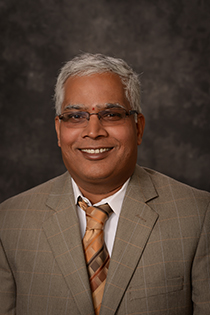}}]{Swaroop~~Darbha} (Fellow, IEEE), received the Ph.D.
degree from the University of California at Berkeley,
Berkeley, CA, USA, in 1994. He is currently the Gulf Oil/Thomas A Dietz
Professor of mechanical engineering with Texas
A\&M University, College Station, TX, USA. His
research interests include dynamics, control, and
diagnostics of connected and autonomous ground
vehicles, routing of unmanned aerial vehicles, and
decision-making under uncertainty. He is a fellow
of ASME and IEEE for his contributions to Intelligent Transportation Systems and Unmanned Vehicles.
\end{IEEEbiography}

\begin{IEEEbiography}[{\includegraphics[width=1 in,height = 1.25 in,clip,keepaspectratio]{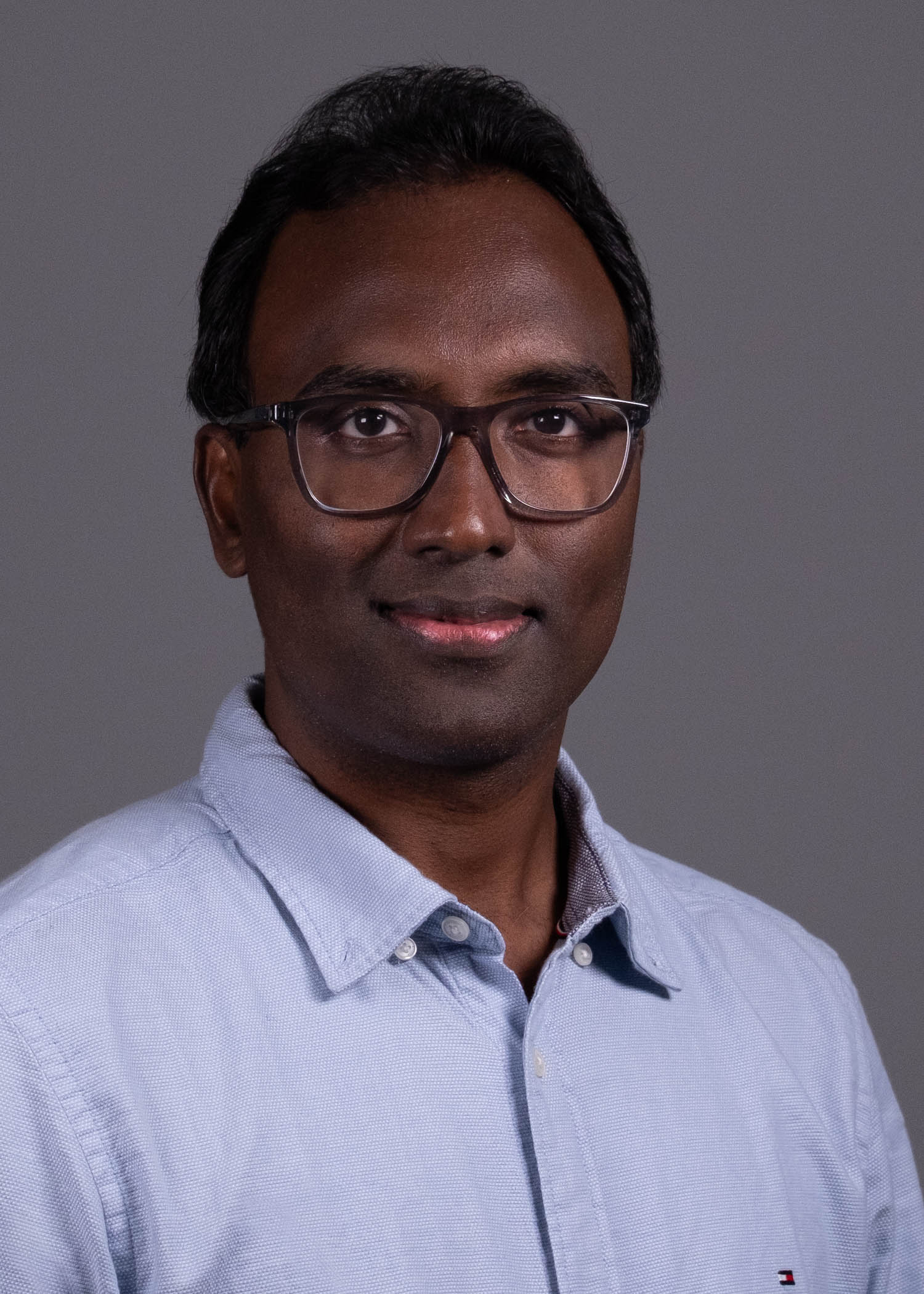}}]{Satyanarayana Gupta Manyam} (Senior Member, IEEE), Satyanarayana Gupta Manyam received the Ph.D. degree in Mechanical Engineering from Texas A \& M University, College Station, Texas, in 2015. He is currently a Research Scientist at DCS Corporation, where he works as a contractor for the Control Science Center of the U.S. Air Force Research Laboratory (AFRL) at Wright-Patterson Air Force Base, OH, USA. His primary research interest include cooperative path planning for multi-vehicle systems, trajectory and motion planning for autonomous vehicles. His interests also include  combinatorial optimization and bounding algorithms for discrete optimization problems.
\end{IEEEbiography}

\begin{IEEEbiography}[{\includegraphics[width=1 in,height = 1.25 in,clip,keepaspectratio]{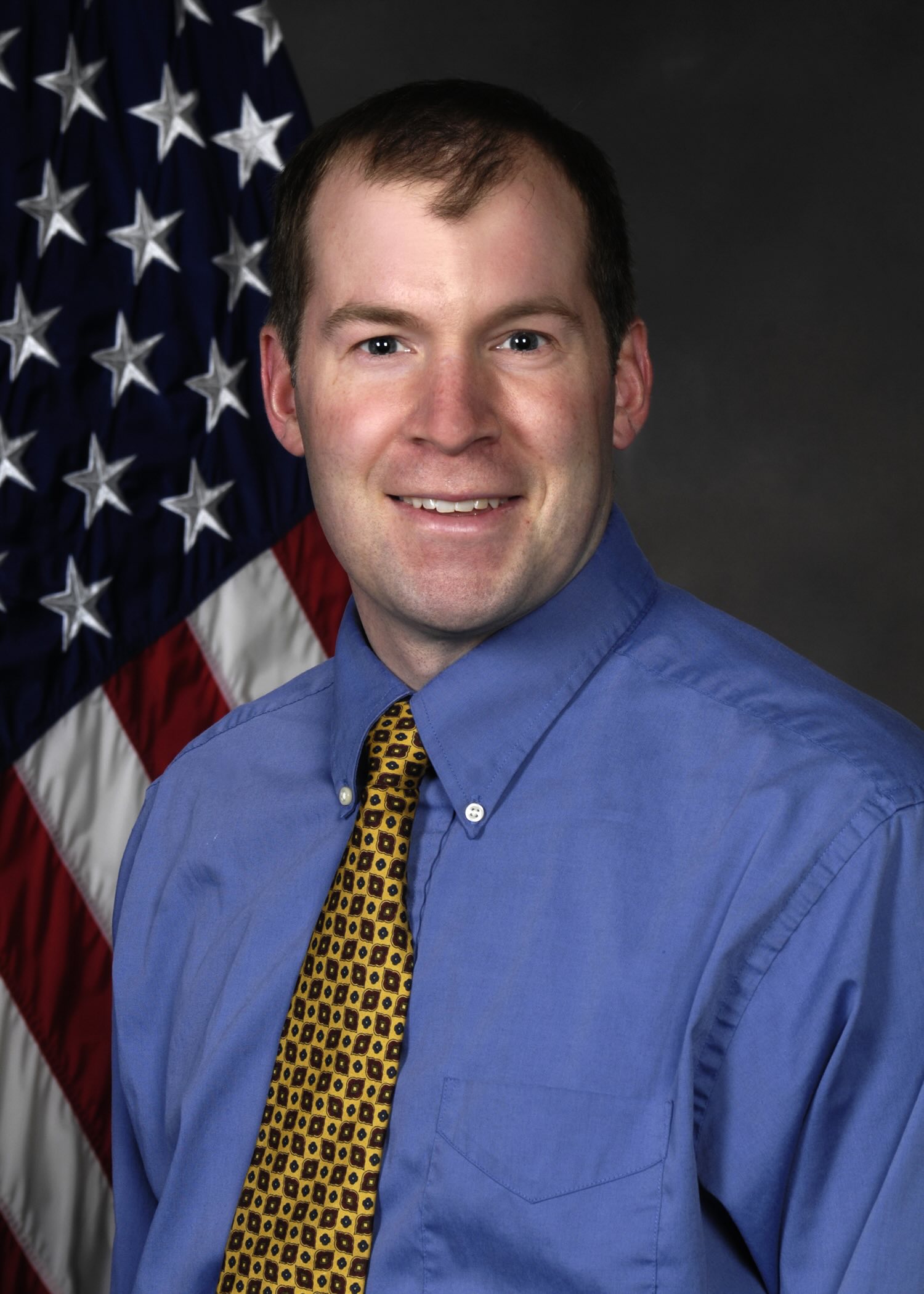}}]{David W. Casbeer} (Senior Member, IEEE) is the technical area lead for UAV Cooperative and Intelligent Control at the Air Force Research Laboratory’s Control Science Center, where he leads research to enable autonomous UAVs in future Air Force missions. He received the BS (2003) and PhD (2009) degrees from Brigham Young University, where he focused on decentralized estimation techniques. He is a Senior Member of the IEEE and an Associate Fellow in the AIAA. He currently serves as an associate editor for the AIAA Journal of Aerospace Information Systems.
    
\end{IEEEbiography}

\end{document}